\documentclass{article}
\usepackage{microtype}
\usepackage{graphicx}
\usepackage{subfigure}
\usepackage{booktabs} % for professional tables

\usepackage{hyperref}

\usepackage[accepted]{icml2025}
\usepackage{amsmath}
\usepackage{amssymb}
\usepackage{mathtools}
\usepackage{amsthm}
\usepackage[linesnumbered,ruled,vlined]{algorithm2e}
\usepackage{thmtools} 
\usepackage{minitoc}
\usepackage{thm-restate}
\usepackage[capitalize,noabbrev]{cleveref}
\usepackage{scalefnt}
\usepackage{tabularx}

\newtheorem{theorem}{Theorem}[section]

\newtheorem{definition}[theorem]{Definition}
\newtheorem{assump}[theorem]{Assumption}

\usepackage[textsize=tiny]{todonotes}

\newcommand{\DP}{\texttt{DP}\xspace}
\newcommand{\CDP}{\texttt{CDP}\xspace}
\newcommand{\RDP}{\texttt{RDP}\xspace}
\newcommand{\DPSGD}{\texttt{DPSGD}\xspace}

\newcommand{\GMM}{\texttt{GMM}\xspace}
\newcommand{\EM}{\texttt{EM}\xspace}

\newcommand{\DPFL}{\texttt{DPFL}\xspace}

\newcommand{\userind}{i}
\newcommand{\FL}{\texttt{FL}\xspace}
\newcommand{\ML}{\texttt{ML}\xspace}

\newcommand{\algname}[1]{{\sf\scalefont{0.90}{#1}}}

\usepackage{color}
\usepackage{enumerate}
\usepackage{graphicx}
\usepackage{amsfonts, amsmath, bm, amssymb}
\usepackage{dsfont}

\usepackage{xspace}
\makeatletter
\DeclareRobustCommand\onedot{\futurelet\@let@token\@onedot}
\def\@onedot{\ifx\@let@token.\else.\null\fi\xspace}

\makeatother

\newcommand{\thetav     }{\boldsymbol \theta     }

\DeclareMathOperator*{\argmax}{arg\,max}
\DeclareMathOperator*{\argmin}{arg\,min}

\icmltitlerunning{Differentially Private Clustered Federated Learning}

\begin{document}

\twocolumn[
\icmltitle{Differentially Private Clustered Federated Learning}




\begin{icmlauthorlist}
\icmlauthor{Saber Malekmohammadi}{yyy,zzz}
\icmlauthor{Afaf Taik}{zzz}
\icmlauthor{Golnoosh Farnadi}{zzz,xxx,vvv}
\end{icmlauthorlist}

\icmlaffiliation{yyy}{School of Computer Science, University of Waterloo, Waterloo, Canada}
\icmlaffiliation{zzz}{Mila - Quebec AI institute, Montreal, Canada}
\icmlaffiliation{xxx}{School of Computer Science, McGill University, Montreal, Canada}
\icmlaffiliation{vvv}{Universit\'e de Montr\'eal, Montreal, Canada}

\icmlcorrespondingauthor{Saber Malekmohammadi}{saber.malekmohammadi@uwaterloo.ca}

\icmlkeywords{Machine Learning, ICML}

\vskip 0.3in
]



\printAffiliationsAndNotice{} % otherwise use the standard text.

\begin{abstract}
Federated learning (\FL), which is a decentralized machine learning (\ML) approach, often incorporates differential privacy (\DP) to provide rigorous data privacy guarantees. Previous works attempted to address high \emph{structured} data heterogeneity in vanilla \FL settings through clustering clients (a.k.a clustered \FL), but these methods remain sensitive and prone to errors, further exacerbated by the \DP noise. This vulnerability makes the previous methods inappropriate for differentially private \FL (\DPFL) settings with structured data heterogeneity. To address this gap, we propose an algorithm for differentially private clustered \FL, which is robust to the \DP noise in the system and identifies the underlying clients' clusters correctly. To this end, we propose to cluster clients based on both their model updates and training loss values. Furthermore, for clustering clients' model updates at the end of the first round, our proposed approach addresses the server's uncertainties by employing large batch sizes as well as Gaussian Mixture Models (\GMM) to reduce the impact of \DP and stochastic noise and avoid potential clustering errors. This idea is efficient especially in privacy-sensitive scenarios with more \DP noise. We provide theoretical analysis to justify our approach and evaluate it across diverse data distributions and privacy budgets. Our experimental results show its effectiveness in addressing large structured data heterogeneity in \DPFL. 
\end{abstract}

\section{Introduction}
Federated learning (\FL) \citep{mcmahan2017communication} is a collaborative \ML paradigm, which allows multiple clients to train a shared global model without sharing their data. However, in order for \FL algorithms to ensure rigorous privacy guarantees against data privacy attacks \citep{Hitaj2017DeepMU, Rigaki2020ASO, Wang2018BeyondIC, Zhu2019DeepLF, Geiping2020InvertingG}, they are reinforced with \DP\citep{Dwork2006, Dwork2006OurDO,Dwork2011AFF, Dwork2014TheAF}. This is done in the presence of a trusted server \citep{McMahan2018LearningDP, Geyer2017DPFedAvg} as well as its absence \citep{Zhao2020LocalDP, Duchi2013LocalPA, Duchi2016MinimaxOP}. In the latter case and for sample-level \DP, each client runs \DPSGD \citep{Abadi2016} locally and shares its noisy model updates with the server at the end of each round.

A key challenge in \FL settings is ensuring an acceptable performance across clients under heterogeneous data distributions.  
% Fairness in centralized \ML is usually defined as a notion of parity across different groups given by a protected attribute (e.g. accuracy parity \citep{gendershades, inherenttradeoff_NEURIPS2019}, demographic parity \citep{fairnessthroguhawareness, fairrepresentations} and equalized odds \citep{EoOp}). 
Several existing works focus on accuracy parity across clients with a \emph{single} common model by agnostic \FL \citep{mohri2019agnostic} and client reweighting \citep{li2019fair, li2020tilted, zhang2023proportional}. However, a single global model often fails to adapt to the data heterogeneity across clients \citep{chu2023focus}, especially when high data heterogeneity exists. Furthermore, when using a single model and augmenting \FL with \DP, different subgroups of clients are unevenly affected - even with loose privacy guarantees \citep{Farrand2020NeitherPN, Fioretto_2022,bagdasaryan2019differential}. In fact, subgroups with minority clients experience a larger drop in model utility, due to the inequitable gradient clipping in \DPSGD \citep{Abadi2016, bagdasaryan2019differential, Xu2021RemovingDI, Esipova2022DisparateII}. Accordingly, some works proposed to use model personalization by multi-task learning \citep{Smith2017FederatedML, ditto, Marfoq2021FederatedML, Wu2023PersonalizedFL}, transfer learning \citep{Li2019FedMDHF, Liu2020ASF} and clustered \FL \citep{ghosh2020, mansour2020approaches, fedsoft2022, Sattler2019ClusteredFL, werner2023provably, briggs2020federated}. The latter has been proposed for vanilla \FL and is suitable when ``structured data heterogeneity" exists across clusters of clients (as in this work): subsets of clients can be naturally grouped together based on their data distributions and one model is learned for each group (cluster). However, as discussed in \citep{werner2023provably}, the existing non-private clustered \FL approaches are vulnerable to errors in clustering due to their sensitivity to:  1. model initialization 2. randomness in clients' model updates due to stochastic noise. The \DP noise existing in \DPFL systems' training mechanism exacerbates this vulnerability by injecting more randomness.

To address the aforementioned gap, we propose a differentially private clustered \FL algorithm which uses both clients' model updates and loss values for clustering clients, making it more robust to \DP/stochastic noise (\Cref{alg:R-DPCFL}): 1) Justified by our theoretical analysis (\Cref{lemma:updatesnoise} and \ref{lemma:localdp}) and in order to cluster clients correctly, our proposed algorithm uses a full batch size in the first \FL round and a small batch size in the subsequent rounds, to reduce the noise in clients' model updates at the end of the first round. 2) Then, the server soft clusters clients based on these less noisy model updates using a Gaussian Mixture Model (\GMM). Depending on the ``confidence" of the learned \GMM, the server keeps using it to soft cluster clients during the next few rounds (\Cref{sec:applicability}). 3) Finally, the server switches the clustering strategy to local clustering of clients based on their loss values in the remaining rounds. These altogether make our \DP clustered \FL algorithm effective and robust. The highlights of our contributions are as follows:

\begin{itemize}

\item  We propose a \DP clustered \FL algorithm (\algname{R-DPCFL}), which combines information from both clients' model updates and their loss values. The algorithm is robust and achieves high-quality clustering of clients, even in the presence of \DP noise in the system (\Cref{alg:R-DPCFL}).

\item We theoretically prove that increasing clients' batch sizes in the first round (and decreasing them in the subsequent rounds) improves the server's ability to cluster clients based on their model updates at the end of the first round with high accuracy (\Cref{lemma:localdp}). 
\item  We show that utilizing sufficiently large client batch sizes in the first round (and sufficiently small batch sizes in the next rounds) enables super-linear convergence rate for learning a \GMM on clients' model updates at the end of the first round. This leads to soft clustering of clients using a \GMM with a low computational overhead (\Cref{theorem:convrate}).  

\item We extensively evaluate across diverse datasets and scenarios, and demonstrate the effectiveness of our robust \DP clustered \FL algorithm in detecting the underlying cluster structure of clients, which leads to an overall utility improvement for the system (\Cref{sec:evaluation}).

\end{itemize}

\iffalse
\begin{itemize}
    \item We show the considerable effect of batch size on the noise in clients' model updates

    \item We prove that increasing batch size of clients in the first round constantly makes it easier for the server to cluster clients based on their model update

    \item We show that large enough batch size for clients, especially in the first communication round, leads to super-linear convergence rate of GMM clustering on their model updates

    \item We propose to combine the two strategies of clustering clients based on their model updates and their loss values leading to a robust algorithm with high quality clustering of clients.

\end{itemize}
\fi

\section{Related work}

% Also, as truly stated in \citep{chu2023focus}, this definition of fairness fails to capture the heterogeneity in data quantity/quality of clients: agent with high quantity/quality data get suppressed for the clients with poor data. We borrow the fairness definition of \citep{chu2023focus} to quantitatively measure the contribution  of clients' local data based on their excess risk. This definition works well even when some clients hold minority data, which is the case in this paper.   

%`\par{\textbf{Personalization via clustered \FL}:} 
%\par{\textbf{Personalized and clustered \FL:}} 

Model personalization is a technique for improving utility under moderate data heterogeneity \citep{ditto, liu2022csfl}, which usually leverages extra computations, e.g., extra local iterations \citep{ditto}. On the other hand, clustered \FL has been proposed for personalized \FL with ``\emph{structured}" data heterogeneity, where clients can be naturally partitioned into clusters: clients in the same cluster have similar data distributions, while there is significant heterogeneity among the various clusters. Existing clustered \FL algorithms group clients based on their loss values 
\citep{ghosh2020,mansour2020approaches, fedsoft2022, chu2023focus, liu2022privacypersonalizationcrosssilofederated} or their model updates (based on e.g., their euclidean distance \citep{werner2023provably, briggs2020federated} or cosine similarity \citep{Sattler2019ClusteredFL}). As shown by \cite{werner2023provably}, the algorithms are prone to clustering errors in the early rounds of \FL training --due to gradient stochasticity, model initialization or the form of loss functions far from their optima-- which can even propagate in the subsequent rounds. This vulnerability is exacerbated in \DPFL systems, due to the existing extra \DP noise. 
Without addressing this vulnerability, \cite{LUO2024384} proposed a \DP clustered \FL algorithm with a limited applicability, which clusters clients based on the labels that they do not have in their local data. In contrast, our \DP clustered \FL algorithm is applicable to any setting characterized by a number of clients, where each client holds many data samples and needs sample-level privacy protection. Cross-silo \FL systems can be considered as an instance. The closest study to this setting was recently done by \cite{liu2022csfl}, which considers silo-specific sample-level \DP and studies the interplay between privacy and data heterogeneity. More specifically, they show that when clients have large dataset sizes and under ``moderate" data heterogeneity across clients: 1. participation in \FL by clients is encouraged over local training, as the model averaging on the server side yields to mitigation of \DP noise 2. under the same total privacy budget, model personalization - through mean regularized multi-task learning (\algname{MR-MTL}) - leads to a better performance compared to learning a single global model or local training by clients (see \cref{app:mrmtl} about \algname{MR-MTL} formulation). Complementing the work, we show that \algname{MR-MTL}, local training and even loss-based  client clustering are not efficient for \DPFL with ``structured" data heterogeneity across clusters of clients.
% \footnote{Also, see \cref{app:related_work} for further related works.}

\begin{figure*}[h!]
    \centering
    \includegraphics[width=0.35\linewidth]{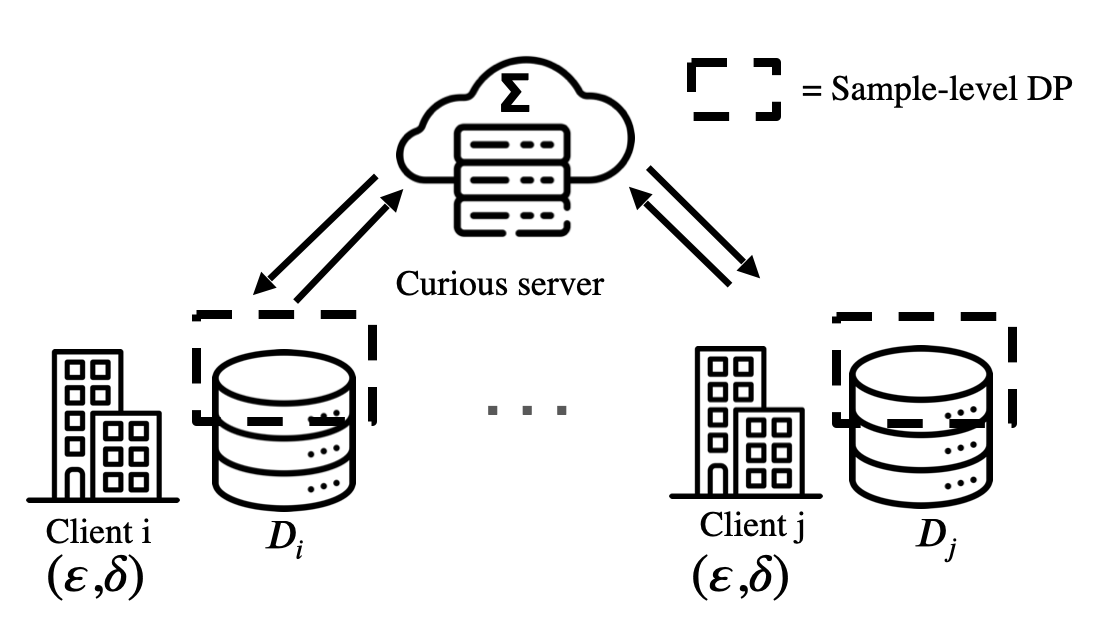} \hspace{4em}\includegraphics[width=0.95\columnwidth,height=5cm]{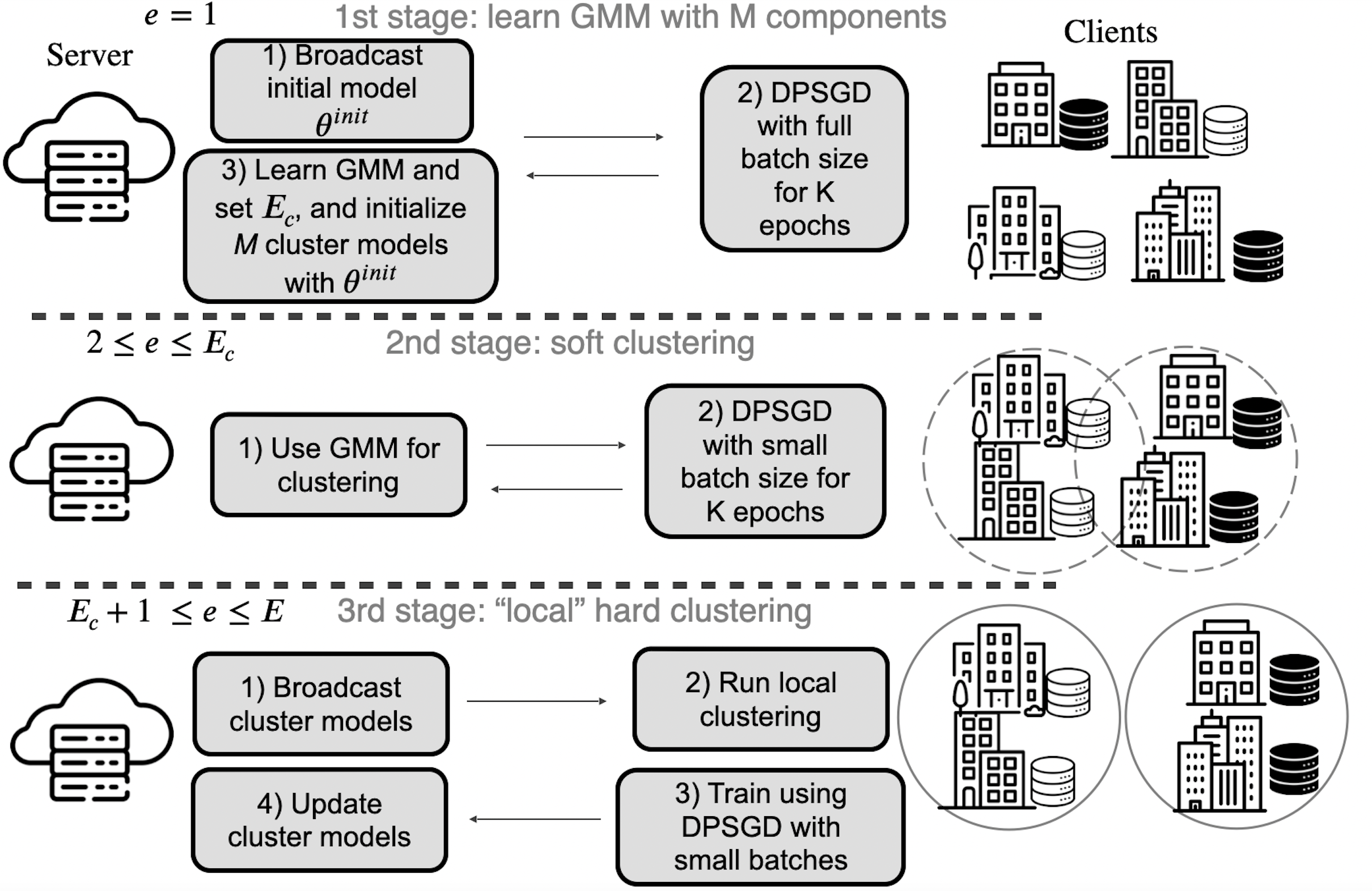}
    
    \caption{\textbf{Left:} Considered threat model in this work, where client $i$ has local train data $\mathcal{D}_i$ and ``sample-level" \DP privacy parameters $(\epsilon, \delta)$, and does not trust any external party. \textbf{Right:} Three main stages of the proposed \algname{R-DPCFL} algorithm.}
    \label{fig:security_model}
    \vspace{-0.5em}
\end{figure*}

\section{Definitions, Notations and  assumptions}\label{sec:notations}

There are multiple definitions of \DP. We adopt the following definition to be satisfied by every client:

\begin{definition}[($\epsilon,\delta$)-\DP \citep{Dwork2006OurDO}]
\label{def:epsilondeltadp}
A randomized mechanism $\mathcal{M}:\mathcal{A}\to \mathcal{R}$ with domain $\mathcal{A}$ and range $\mathcal{R}$ satisfies $(\epsilon,\delta)$-\DP if for any two adjacent inputs $\mathcal{D}$, $\mathcal{D}'\in \mathcal{A}$, which differ only by a single record (by replacement), and for any measurable subset of outputs $\mathcal{S} \subseteq \mathcal{R}$ it holds that
\begin{align*}
    \texttt{Pr}[\mathcal{M}(\mathcal{D})\in \mathcal{S}] \leq e^{\epsilon} \texttt{Pr}[\mathcal{M}(\mathcal{D}')\in \mathcal{S}]+\delta.
\end{align*}
\end{definition}

The gaussian mechanism randomizes the output of a query $f$ as $\mathcal{M}(\mathcal{D}) \triangleq f(\mathcal{D})+\mathcal{N}(0,\sigma^2)$. The randomized output of the mechanism satisfies ($\epsilon, \delta$)-\DP for a continuum of pairs ($\epsilon, \delta$): it is ($\epsilon, \delta$)-\DP for all $\epsilon$ and $\sigma>\frac{\sqrt{2\ln (1.25/\delta)}}{\epsilon} \Delta_2 f$, where $\Delta_2f \triangleq \max_{\mathcal{D},\mathcal{D}'}\parallel f(\mathcal{D})-f(\mathcal{D}')\parallel_2 $ is the $l_2$-sensitivity of the query $f$ with respect to its input. Also, the $\epsilon$ and $\delta$ privacy parameters resulting from running Gaussian mechanism depend on the quantity $z=\frac{\sigma}{\Delta_2f}$ (called ``noise scale"). We consider a \DPFL system (see \cref{fig:security_model}, left), where there are $n$ clients running \DPSGD with the same ``sample-level" privacy parameters ($\epsilon, \delta$): the set of information (including model updates and cluster selections) sent by client $i$ to the server satisfies $(\epsilon, \delta)$-\DP for all adjacent datasets $\mathcal{D}_i$ and $\mathcal{D}_i'$ of the client $i$ differing in one sample.

Let $x\in \mathcal{X}\subseteq\mathbb{R}^d$ and $y \in \mathcal{Y}=\left\{1, \ldots, C \right\}$ denote an input data point and its target label. Client $i$ holds dataset $\mathcal{D}_i$ with $N_i$ samples from distribution $P_i(x,y)=P_i(y|x)P_i(x)$. Let $h: \mathcal{X}\times \mathbf{\thetav} \to\mathbb{R}^C$ be the predictor function, which is parameterized by $\mathbf{\thetav}\in \mathbb{R}^p$. Also, let $\ell:\mathbb{R}^C\times\mathcal{Y}\to \mathbb{R}_+$ be the used loss function (cross-entropy loss). Client $i$ in the system has empirical train loss $f_i(\mathbf{\thetav})=\frac{1}{N_i}\sum_{(x,y)\in \mathcal{D}_i}[\ell(h(x,\mathbf{\thetav}), y)]$, with minimum value $f_i^*$. There are $E$ communication rounds indexed by $e$ and $K$ local epochs with learning rate $\eta_l$ during each round. There are $M$ clusters of clients indexed by $m$, and the server holds $M$ cluster models $\{\thetav_m^e\}_{m=1}^M$ for them at the beginning of round $e$. Clients $i$ and $j$ belonging to the same cluster have the same data distributions, while there is high data heterogeneity across clusters. $s(i)$ denotes the true cluster of client $i$ and $R^e(i)$ denotes the cluster assigned to it at the beginning of round $e$. Let us assume the batch size used by client $i$ in the first round $e=1$ is $b_i^1$, which may be different from the batch size $b_i^{>1}$ that it uses in the rest of the rounds $e>1$. At the $t$-th gradient update during the round $e$, client $i$ uses batch $\mathcal{B}_i^{e,t}$ with size $b_i^e$, and computes the following \DP noisy batch gradient:
\begin{align}
    \Tilde{g}_i^{e,t}(\mathbf{\thetav}) = \frac{1}{b_i^e}\bigg[ \Big(\sum_{j \in \mathcal{B}_i^{e,t}} \Bar{g}_{ij}(\mathbf{\thetav})\Big) + \mathcal{N}(0, \sigma_{i, \texttt{DP}}^2 \mathbb{I}_p)\bigg],
    \label{eq:noisy_sg}
\end{align}
where $\Bar{g}_{ij}(\mathbf{\thetav}) = \texttt{clip}(\nabla \ell(h(x_{ij},\mathbf{\thetav}), y_{ij}), c)$, and $c$ is a clipping threshold to clip sample gradients: for a given vector $\mathbf{v}$, $\texttt{clip}(\mathbf{v}, c) =  \min\{\|\mathbf{v}\|, c\} \cdot \frac{\mathbf{v}}{\|\mathbf{v}\|}$.  Also, $\mathcal{N}$ is the Gaussian noise distribution with variance  $\sigma_{i,\texttt{DP}}^2$, where $\sigma_{i,\texttt{DP}}=c\cdot z_i(\epsilon, \delta, b_i^1, b_i^{>1}, N_i, K, E)$, and $z_i$ is the noise scale needed for achieving $(\epsilon, \delta)-$\DP by client $i$, which can be determined with a privacy accountant, e.g., Renyi-\DP accountant \citep{mironov2019renyidifferentialprivacysampled} used in this work, which is capable of accounting composition of \emph{heterogeneous} \DP mechanisms \citep{mironovRDP}. The privacy parameter $\delta$ is fixed to $10^{-4}$ in this work and for every client $i$: $\delta < N_i^{-1}$. For an arbitrary random $\mathbf{v} = (v_1, \ldots, v_p)^\top \in \mathbb{R}^{p\times 1}$, we define $\texttt{Var}(\mathbf{v}):= \sum_{j=1}^p \mathbb E[(v_j - \mathbb E[v_j])^2]$, i.e., variance of $\mathbf{v}$ is the sum of the variances of its elements. \Cref{tab:notations} in the appendix summarizes the used notations. Finally, we have the following assumption:

\begin{assump}\label{assump:_boundedvariance}
The stochastic gradient $g_i^{e,t}(\thetav) = \frac{1}{b_i^e} \sum_{j \in \mathcal{B}_i^{e,t}} g_{ij}(\thetav)$ is an unbiased estimate of $\nabla f_i(\thetav)$ with a bounded variance: $\forall \thetav \in \mathbb R^p: \texttt{Var}(g_i^{e,t}(\thetav)) \leq \sigma_{i, g}^2(b_i^e)$. The tight bound $\sigma_{i, g}^2(b_i^e)$ is a constant depending only on the used batch size $b_i^e$: the larger $b_i^e$, the smaller $\sigma_{i, g}^2(b_i^e)$.
\end{assump}

\begin{algorithm*}[t]
\caption{\algname{R-DPCFL}}
\label{alg:R-DPCFL}
%\begin{algorithmic}[1]
\KwIn{Initial parameter $\mathbf{\thetav}^{\textit{init}}$,  dataset sizes $\{N_1, \ldots, N_n\}$, batch sizes $\{b_1^{>1}, \ldots, b_n^{>1}\}$, clip bound $c$, local epochs $K$, global round $E$, number of clusters $M$ (optional)}

\KwOut{cluster models $\{\thetav_m^E\}_{m=1}^M$}

%\KwOut{$\mathbf{\thetav}_E, \{\epsilon_1, \ldots, \epsilon_n\}$}

\For{each client $\userind \in \{1,\ldots,n\}$}
        {
          $b_i^1 \gets N_i$ \tcp*{full batch size}
          
          $z_i\gets$ \texttt{RDP}($\epsilon, \delta, b_i^1, b_i^{>1}, N_i, K, E$) 
        }
 
\For{$e\in \{1, \ldots, E\}$}
{
    
    \uIf {$e=1$}
    {
        \For{each client $\userind \in \{1,\ldots,n\}$ \textbf{in parallel}}
        {
          $\Delta \Tilde{\mathbf{\thetav}}_i^1 \gets$ \DPSGD($\mathbf{\thetav}^{init}, b_i^1, N_i, K, z_i, c$) 
        }
        
        on server:

        \uIf {$M$ is unknown} 
        {
        $M = \argmax_{M'} \texttt{MSS}\Big(\texttt{GMM}(\Delta \Tilde{\mathbf{\thetav}}_1^1, \ldots, \Delta \Tilde{\mathbf{\thetav}}_n^1; M')\Big)$ \tcp*{set $M$ (\cref{sec:applicability})}
        }

        $\{\pi_1, \ldots, \pi_{n}, \texttt{MPO}\} = \texttt{GMM}(\Delta \Tilde{\mathbf{\thetav}}_1^1, \ldots, \Delta \Tilde{\mathbf{\thetav}}_n^1; M)$  \tcp*{\textbf{1st stage:} \texttt{GMM}}

        set $E_c(\texttt{MPO})$ \tcp*{set $E_c$ (\cref{sec:applicability})}

        Initialize cluster models uniformly: $\thetav_1^{2} = \ldots = \thetav_M^{2} = \mathbf{\thetav}^{\textit{init}}$ 
        
        continue \tcp*{go to round $e=2$}
        
    %     \For{$i \in \{1,\dots,n\}$}
    %     {
    %       $R^{2}(i) \gets m \textit{~ with probability $\pi_i[m]$
    %     }$ \tcp*{sample from GMM for round $2$} 
    % }

}

\uElseIf {$e\in \{2, \ldots, E_c\}$}
    {
        \For{each client $i \in \{1,\dots,n\}$}
        {
          $R^{e}(i) \gets m \textit{~ with probability $\pi_i[m]$}$ \tcp*{\textbf{2nd stage:} soft clustering}
        } 
    }

    \uElse
    {
    on server: broadcast  cluster models $\{\thetav_m^e\}_{m=1}^M$ to all clients
    
    \For{each client  $i \in \{1,\dots,n\}$}
        {   $R^{e}(i) = \argmin_m f_i(\mathbf{\thetav}_m^{e})
            $ \tcp*{\textbf{3rd stage:} private local clustering}
        }
    }

\For{each client $\userind \in \{1,..,n\}$ \textbf{in parallel}}
    { 
       
      $\Delta \Tilde{\mathbf{\thetav}}_i^e \gets$\DPSGD($\mathbf{\thetav}_{R^e(i)}^e, b_i^{>1}, N_i, K, z_i, c$) \tcp*{batch size $b_i^{>1}$}
    }

    on server:

    \For{each client $i \in \{1,\dots,n\}$}
        {   
        {$w_i^e \gets \frac{1}{\sum_{j=1}^n \mathds{1}_{R^e(j)=R^e(i)} }$}
        }
        
    \For{$m \in \{1,\dots,M\}$}
    {
      
      $\mathbf{\thetav}_m^{e+1} \gets \mathbf{\thetav}_m^e + \sum_{i \in \{1,\dots,n\}} \mathds{1}_{R^e(i)=m} w_i^e \Delta \Tilde{\mathbf{\thetav}}_i^e$
    }

}

%\end{algorithmic}
\end{algorithm*}

\section{Methodology and proposed algorithm}
As discussed in \citep{werner2023provably},  existing \emph{non-\DP} clustered \FL algorithms are prone to clustering errors, especially in the first rounds (see \cref{app:grad_loss} for a detailed discussion and an illustrating example). Motivated by this vulnerability, \emph{which will get exacerbated by \DP noise}, we next propose a \DP clustered \FL algorithm which starts with clustering clients based on their model updates for the first several rounds and then switches its strategy to cluster clients based on their loss values. We augment this idea with some other non-obvious techniques to enhance the clustering accuracy.

\begin{figure*}[t]
%\vspace{-3em}
\centering
\includegraphics[width=0.58\columnwidth,height=4.5cm]{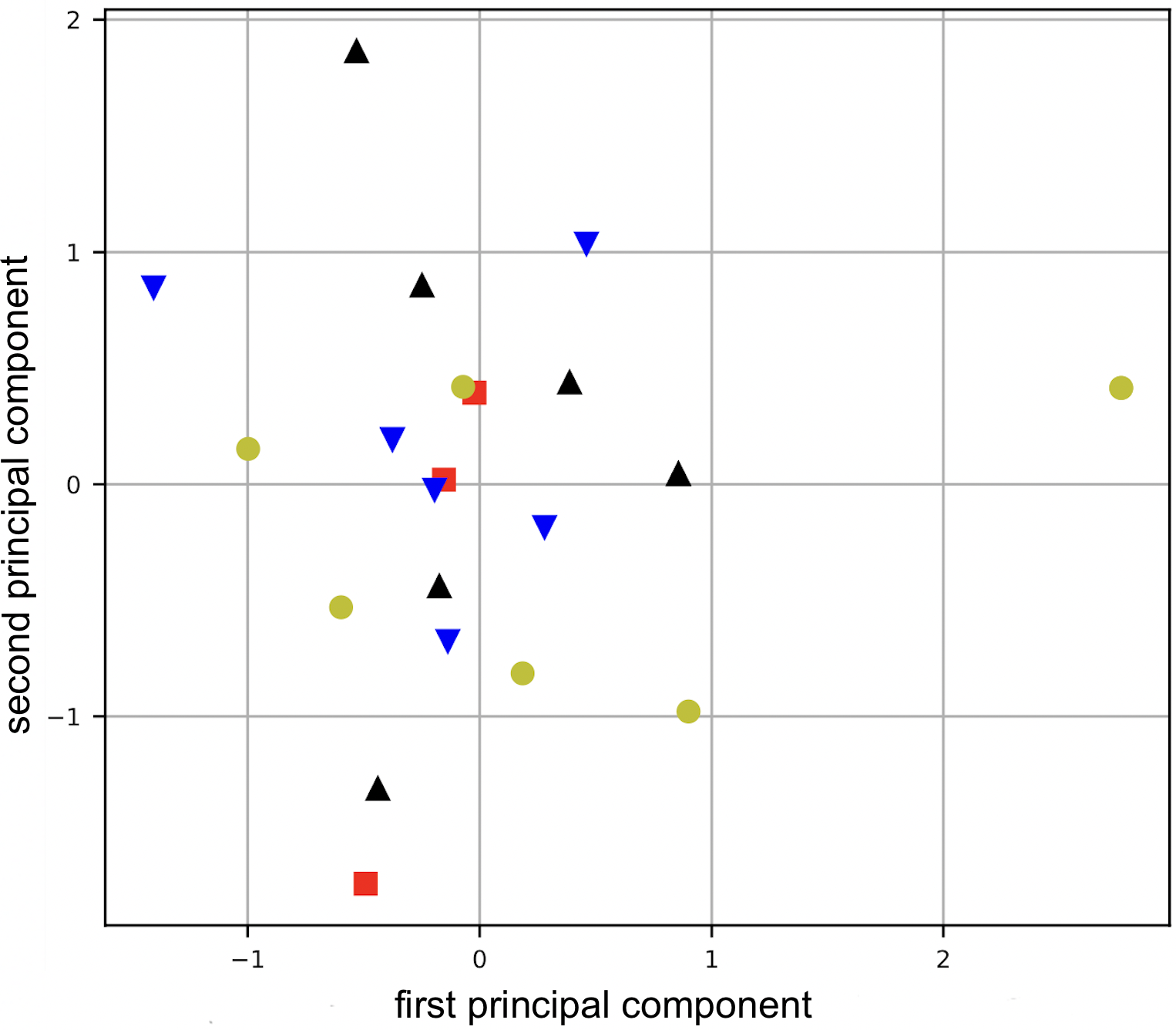}~~~~~~~~
\includegraphics[width=0.63\columnwidth, height=4.5cm]{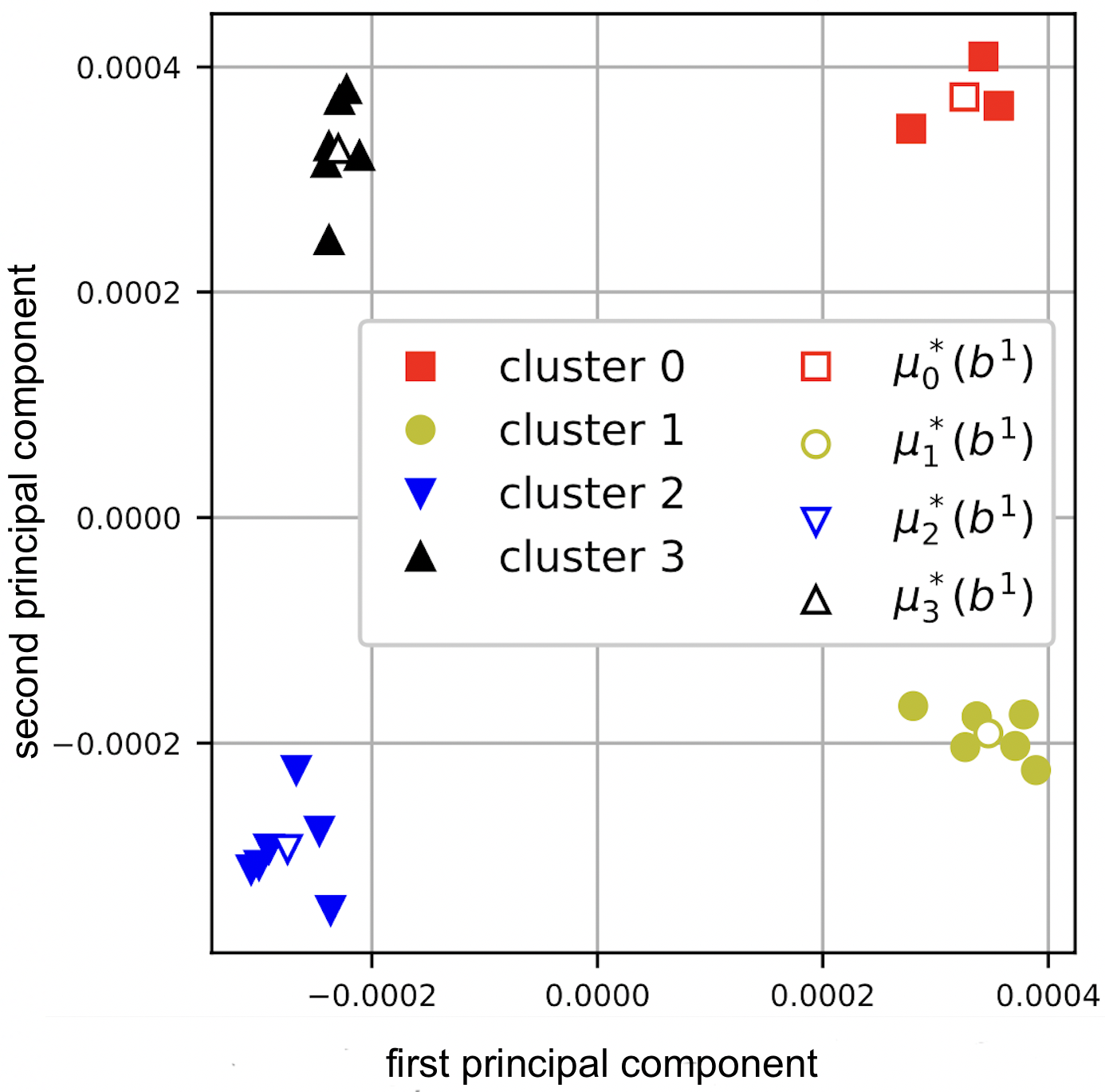}
\caption{PCA visualization of updates $\{\Delta \Tilde{\mathbf{\thetav}}_i^1\}_{i=1}^n$ on 2D space. \textbf{Left:} $\epsilon_i=10$, $b_i^e=32$ for all $i$ and $e$. \textbf{Right:} $\epsilon_i=10$, $b_i^1=b^1=N=6600$,\textbf{ i.e., full batch sizes} (assuming $N_i=N=6600$ for all clients), and  $b_i^{>1}=32$ for all $i$. The empty markers show the centers of the Gaussian components. The model updates are obtained from clients running \DPSGD for $K=1$ epochs locally on CIFAR10 with covariate shift (rotation) across clusters, and under the same values as in \cref{fig:var1var2}.
} 
\label{fig:client_updates_differentbatch}
% \vspace{-1em}
\end{figure*}

\begin{figure*}[t]
\centering
\includegraphics[width=0.8\columnwidth, height=5cm]{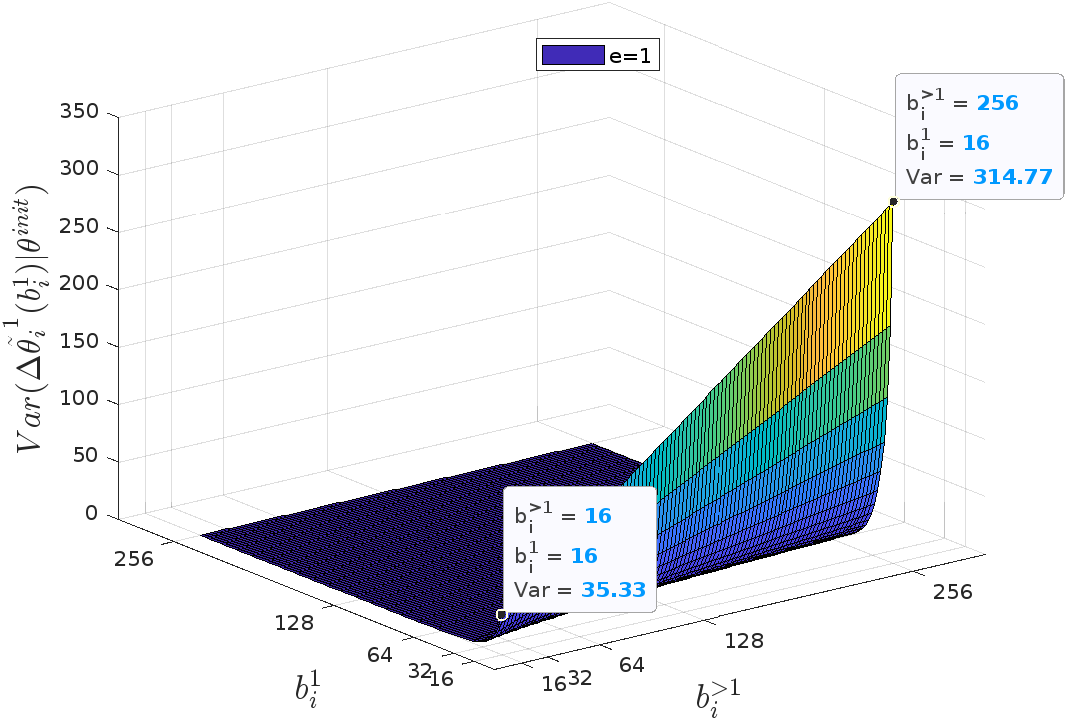}
\includegraphics[width=0.8\columnwidth, height=5cm]{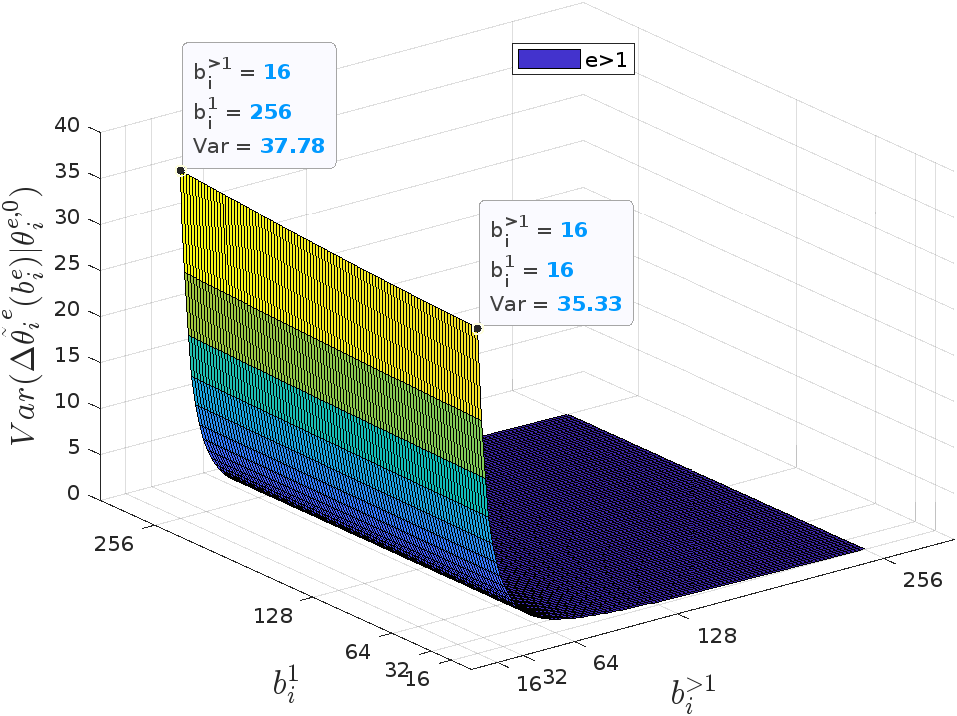}
\caption{Plot of $\texttt{Var}(\Delta \Tilde{\mathbf{\thetav}}_i^1(b_i^1)|\mathbf{\thetav}_{i}^{init})$ (\textbf{left}) and $\texttt{Var}(\Delta \Tilde{\mathbf{\thetav}}_i^e(b_i^e)|\mathbf{\thetav}_{i}^{e,0})$ $(e>1)$ (\textbf{right}) v.s. both $b_i^1$ and $b_i^{>1}$. There are two clear takeaways: 1) for all $e\in \{1, \cdots, E\}$, $\texttt{Var}(\Delta \Tilde{\mathbf{\thetav}}_i^e(b_i^e)|\mathbf{\thetav}_{i}^{e,0})$ decreases with $b_i^e$ steeply (from \cref{lemma:updatesnoise}). 2) The effect of $b_i^{>1}$ on $\texttt{Var}(\Delta \Tilde{\mathbf{\thetav}}_i^1(b_i^1)|\mathbf{\thetav}_{i}^{init})$ (left figure) is considerable. The reason is that $b_i^{>1}$ is used in $E-1$ rounds and affects the noise scale $z_i(\epsilon, \delta, b_i^1, b_i^{>1}, N_i, K, E)$ used by \DPSGD: see \cref{fig:zvsq} in the appendix for the plot of $z_i(\epsilon, \delta, b_i^1, b_i^{>1}, N_i, K, E)$ v.s. $b_i^1$ and $b_i^{>1}$. The results are obtained on CIFAR10 from Renyi-\DP Accountant \citep{ mironov2019renyidifferentialprivacysampled} in a setting with $N_i=6600, \epsilon=5, \delta=10^{-4}, c=3, K=1, E=200, p=11,181,642, \eta_l=5\times 10^{-4}$.} 
\label{fig:var1var2}
\vspace{-1em}
\end{figure*}

\subsection{\algname{R-DPCFL} algorithm}

Our proposed \algname{R-DPCFL} algorithm has three main steps (also see \Cref{fig:security_model}, right and \Cref{alg:R-DPCFL}):
%\textcolor{red}{if knowing M is optional, then we should modify this, the figure, and the algorithm}
\begin{enumerate}
    %\item Initializing clusters uniformly $(\forall m \in [M]: \thetav_m^1 = \thetav^{\textit{init}})$, 
    \item In the first round, clients train the initial model $\thetav^{\textit{init}}$ locally. They use \textbf{full batch sizes} in this round to make their model updates $\{\Delta \Tilde{\mathbf{\thetav}}_i^1\}_{i=1}^n$ less noisy \footnote{Note that even when clients have a limited memory budget, they can still perform \DPSGD with full batch size and no computational overhead by using gradient accumulation technique (see \Cref{app:grad_acc}).}. Then, the server \textbf{soft clusters them by learning \GMM on their model updates}. The number of clusters ($\mathbf{M}$) is either given or can be found by maximizing the confidence of the learned \GMM (\cref{sec:applicability}).
    \item During the subsequent rounds $e \in \{2, \ldots, E_c\}$, the server uses the learned \GMM to \textbf{soft-cluster clients}: client $i$ uses a small batch size $b_i^{>1}$ and contributes to the training of each cluster ($m$) model proportional to the probability of its assignment to that cluster ($\pi_{i,m}$). \textbf{The number of rounds $E_c$ is set based on ``confidence level" of the learned \GMM (\cref{sec:applicability}).} 

    \item After the first $E_c$ rounds, some progress has been made in the training of the cluster models $\{\thetav_m^{E_c}\}_{m=1}^M$. Now, clients' train loss values on cluster models are meaningful and is the right time to \textbf{hard cluster clients based on their loss values} during the remaining rounds to build more personalized models per cluster:
    $R^e(i) = \argmin_m f_i(\thetav_m^{e})$. 
\end{enumerate}

In Sections \ref{sec:batchsz} and \ref{sec:cvgce}, we provide theoretical justifications for the usage of full batch size in the first round.

\subsection{Reducing \GMM uncertainty via using full batch sizes in the first round and small batch sizes in the subsequent rounds}
\label{sec:batchsz}

It is the \DP noise in $\{\Delta \Tilde{\mathbf{\thetav}}_i^1\}_{i=1}^n$ that makes it hard for the server to cluster clients by learning a \GMM on the model updates. \Cref{lemma:updatesnoise} extends a result in \citep{malekmohammadi2024noiseawarealgorithmheterogeneousdifferentially} and shows that the amount of noise in $\Delta \Tilde{\mathbf{\thetav}}_i^e$ at the end of each round $e$ rapidly drops when $b_i^e$ increases.

\begin{restatable}{lem}{batchsizeeffectonnoise}
% Let us assume $\thetav_{i}^{e,0}$ is the model parameter passed to client $i$ at the beginning of round $e$. After $K$ local epochs with step size $\eta_l$, the client generates the noisy \DP model update $\Delta \Tilde{\thetav}_i^e(b_i^e)$ at the end of the round. The amount of noise in the resulting model update can be found as:

Let $\approx$ denote approximation. After $K$ local epochs with step size $\eta_l$, client $i$ generates the noisy \DP model update $\Delta \Tilde{\thetav}_i^e(b_i^e)$ at the end of the round $e$. The amount of noise in the resulting model update can be found as:

% Assuming that the clipping threshold $c$ is effective for all sampled data points in the batch, i.e. $\forall j \in \mathcal{B}_i^{e,t}: c< \|g_{ij}(\thetav)\|=\|\nabla \ell(h(x_{ij},\mathbf{\thetav}), y_{ij})\|$, we have:

\vspace{-2.5em}
\begin{align}\label{eq:sigma_i^2}
    \sigma_i^{e^2}(b_i^e) &:= \texttt{Var}(\Delta \Tilde{\mathbf{\thetav}}_i^e(b_i^e)|\mathbf{\thetav}_{i}^{e,0})\nonumber \\
    & \approx K \cdot N_i \cdot \eta_l^2 \cdot \frac{p c^2 z_i^2(\epsilon, \delta, b_i^1, b_i^{>1}, N_i, K, E)}{b_i^{e^3}}.
\end{align}
\vspace{-2em}
\label{lemma:updatesnoise}
\end{restatable}

% \begin{figure*}[t]
% \centering
% \includegraphics[width=0.28\columnwidth,height=3.3cm]{DPFedAvg10.png}
% \includegraphics[width=0.31\columnwidth, height=3.6cm]{full10.png}
% \includegraphics[width=0.31\columnwidth, height=3.6cm]{full4.png}
% \caption{2D visualization of model updates $\{\Delta \Tilde{\mathbf{\thetav}}_i^1\}_{i=1}^n$ on CIFAR10 with covariate shift (rotation). \textbf{Left:} $\epsilon_i=10$, DPFedAvg with $b_i^e=32$ for all $i, e$ \textbf{Middle:} $\epsilon_i=10$, $b_i^1=b^1=N$ for all $i$, and  $b_i^e=32$ for all $i$ and $e>1$ \textbf{Right:} $\epsilon_i=4$, $b_i^1=b^1=N$ for all $i$, and  $b_i^e=32$ for all $i$ and $e>1$.
% } 
% \label{fig:client_updates_differentbatch}
% \end{figure*}

% \begin{figure*}[t]
% \centering
% \includegraphics[width=0.28\columnwidth,height=3.3cm]{DPFedAvg10.png}
% \caption{csdcscscsc
% } 
% \label{fig:client_updates_differentbatch}
% \end{figure*}

% \emph{We have shown $b_i^e$ as an argument of $\sigma_i^{e^2}(b_i^e)$ to emphasize on its dependence on $b_i^e$.} 

The first conclusion from the lemma is that the noise level in $\Delta \Tilde{\mathbf{\thetav}}_i^e$ rapidly declines as $b_i^e$ increases: See \cref{fig:var1var2} and the effect of batch size $b_i^1$ on $\texttt{Var}(\Delta \Tilde{\mathbf{\thetav}}_i^1|\mathbf{\thetav}_{i}^{init})$ (on the left); and the effect of batch size $b_i^{>1}$ on $\texttt{Var}(\Delta \Tilde{\mathbf{\thetav}}_i^e|\mathbf{\thetav}_{i}^{e,0})$ ($e>1$) (on the right). Let us consider $e=1$ especially: If all clients use full batch sizes in the first round (i.e., $b_i^1=N_i$ for every client $i$), it becomes much easier for the server to cluster them at the end of the first round by learning a \GMM on $\{\Delta \Tilde{\mathbf{\thetav}}_i^1\}_{i=1}^n$, as their updates become more separable. An illustration of this is shown in \Cref{fig:client_updates_differentbatch}. As the second key takeaway, \cref{fig:var1var2}, left shows that \textbf{in order to make $\{\Delta \Tilde{\mathbf{\thetav}}_i^1\}_{i=1}^n$ less noisy, we have to make $\{b_i^1\}_{i=1}^n$ as large as possible and also keep $\{b_i^{>1}\}_{i=1}^n$ small\footnote{In fact, there is a close relation between the result of \cref{lemma:updatesnoise} and the law of large numbers. See \cref{app:lolns} for more details}}. In the next section, we will provide a theoretical justification for the observation in \Cref{fig:client_updates_differentbatch}.
% These interesting results are consistent with the observations in \citep{de2022, anil2021, dormann2021, hoory2021, li2022, luo2021} that increasing the batch size can significantly improve the privacy-utility trade-off of \DPSGD.

\subsubsection{Effect of batch sizes $\{b_i^1\}_{i=1}^n$ on the separation between clusters}\label{subsec:batchsz1}
In order to theoretically understand the reason behind the observation in \Cref{fig:client_updates_differentbatch}, let us assume clients have the same dataset sizes and first batch sizes for simplicity: $\forall i: N_i = N, b_i^1=b^1$. Also, remember that $\mathbf{\thetav}_{i}^{1,0} = \mathbf{\thetav}^{\textit{init}}$. Having uniform privacy parameters $(\epsilon, \delta)$, we have: $\forall i: ~\sigma_i^{1^2} (b^1) := \texttt{Var}[\Delta \Tilde{\mathbf{\thetav}}_i^1(b^1)|\mathbf{\thetav}^{\textit{init}}] = \sigma^{1^2}(b^1)$. Hence, we can consider the model updates $\{\Delta \Tilde{\mathbf{\thetav}}_i^1(b^1)\}_{i=1}^n$ as the samples from a mixture of $M$ Gaussians with mean, covariance matrix, prior probability parameters: $\psi^*(b^1) = \{\mu_m^*(b^1), \Sigma_m^*(b^1), \alpha_m^*\}_{m=1}^M$, where $\forall m: \alpha_m^*>0$ and $\mu_m^*(b^1) \neq \mu_{m'}^*(b^1)$ ($m\neq m'$), due to data heterogeneity:

\vspace{-1em}
\begin{align}
    \mu_m^*(b^1) &:= \mathbb{E}\bigg[\Delta \Tilde{\thetav}_i^1(b_i^1)\bigg|\mathbf{\thetav}^{\textit{init}}, b_i^1=b^1 , s(i)=m\bigg],
    \end{align}
\begin{align}\label{eq:diagonal_covariance}
    \Sigma_m^*(b^1) := ~~\mathbb{E}\bigg[&\big(\Delta \Tilde{\thetav}_i^1(b_i^1) - \mu_m^*(b^1)\big) \big(\Delta \Tilde{\thetav}_i^1(b_i^1) - \mu_m^*(b^1)\big)^\top \nonumber \\ & \bigg|\mathbf{\thetav}^{\textit{init}}, b_i^1=b^1, s(i)=m\bigg] = \frac{\sigma^{1^2}(b^1)}{p} \mathbb{I}_p,
\vspace{-1em} 
\end{align}

where the last equality is from $\texttt{Var}[\Delta \Tilde{\mathbf{\thetav}}_i^1|\mathbf{\thetav}^{\textit{init}}, b_i^1=b^1] = \mathbb{E}[\|\Delta \Tilde{\thetav}_i^1 - \mu_{s(i)}^*(b^1)\|^2] = \sigma^{1^2}(b^1) $ and that the noises existing in each of the $p$ elements of $\Delta \Tilde{\mathbf{\thetav}}_i^1$ are \emph{i.i.d} (hence, $\Sigma_m^*(b^1)$ is a diagonal covariance matrix with equal diagonal elements). Intuitively, we expect more separation between the true Gaussian components $\{\mathcal{N}\big(\mu_m^*(b^1), \Sigma_m^*(b^1)\textbf{}\big)\}_{m=1}^M$, from which clients' updates $\{\Delta \Tilde{\mathbf{\thetav}}_i^1\}_{i=1}^n$ are sampled, to make the model updates more distinguishable for server. Next, we show that the overlap between the Gaussian components $\{\mathcal{N}(\mu_m^*(b^1), \Sigma_m^*(b^1))\}_{m=1}^M$ decreases \emph{fast} with $b^1$:

% \begin{itemize}

%     \item conditional probabilities for the model updates: \\
%     $p(\Delta \Tilde{\mathbf{\thetav}}_i^1|\mu_m(b^1), \Sigma_m(b^1)) = \frac{1}{(2\pi |\Sigma_m(b^1)|)^{\frac{d}{2}}} e^{-\frac{1}{2}(\Delta \Tilde{\mathbf{\thetav}}_i^1(b^1)-\mu_m(b^1))^\top\Sigma_m^{-1}(b^1)(\Delta \Tilde{\mathbf{\thetav}}_i^1(b^1)-\mu_m(b^1))}, \\~i \in \{1, \ldots, n\}, ~m \in \{1, \cdots, M\}$

%     \item posterior probabilities $\pi_i \in \mathbb{R}^M$ for the model update of client $i$, where:\\
%     $\pi_{i,m}(b^1) = \textit{Pr}(s(i)=m|\Delta \Tilde{\mathbf{\thetav}}_i^1(b^1)) = \frac{\alpha_m p(\Delta \Tilde{\mathbf{\thetav}}_i^1(b^1)|\mu_m(b^1), \Sigma_m(b^1))}{\Sigma_{j=1}^M \alpha_j p(\Delta \Tilde{\mathbf{\thetav}}_i^1(b^1)|\mu_j(b^1), \Sigma_j(b^1))}, \\ 
%     i\in\{1, \cdots, n\}, ~m \in \{1, \cdots, M\}$

%     \item likelihood of the model update of client $i$:\\

%     $p(\Delta \Tilde{\mathbf{\thetav}}_i^1(b^1)|\mu_1(b^1), \Sigma_1(b^1), \cdots,\mu_M(b^1), \Sigma_M(b^1)) = \Sigma_{j=1}^M \alpha_j p(\Delta \Tilde{\mathbf{\thetav}}_i^1(b^1)|\mu_j(b^1), \Sigma_j(b^1)), \\
%     i \in \{1, \cdots, n\}$

% \end{itemize}

% We also have $\alpha_m(b^1) = \frac{1}{n}\Sigma_{i=1}^n \pi_{i,m}(b^1), ~~~ i \in \{1, \cdots, n\}, ~~~ m \in \{1, \cdots, M\}$.

\begin{restatable}{lem}{batchsizeeffect}
Let $\Delta_{m,m'}(b^1):=\|\mu^*_m(b^1) - \mu^*_{m'}(b^1)\|$ when $\forall i: b_i^1=b^1$. The overlap between components $\mathcal{N}\big(\mu_m^*(b^1), \Sigma_m^*(b^1)\textbf{}\big)$ and $\mathcal{N}\big(\mu_{m'}^*(b^1), \Sigma_{m'}^*(b^1)\textbf{}\big)$ is $O_{m,m'}=2Q(\frac{\sqrt{p}\Delta_{m,m'}(b^1)}{2\sigma^1(b^1)})$, where $\sigma^{1^2}(b^1) := \texttt{Var}[\Delta \Tilde{\mathbf{\thetav}}_i^1|\mathbf{\thetav}^{\textit{init}}, b_i^1=b^1]$ and $Q(\cdot)$ is the $Q$ function. Furthermore, if we increase $b_i^1=b^1$ to $b_i^1=kb^1\leq N$ (for all $i$), we have $O_{m,m'}\leq 2Q(\frac{\sqrt{kp}\Delta_{m,m'}(b^1)}{2\rho \sigma^1(b^1)})$, where $1\leq \rho \in \mathcal{O}(1)$ is a small constant. 
\label{lemma:localdp}
\end{restatable}

% The lemma states that using a large batch size in the first round results in a \emph{fast} reduction of the overlap between the underlying components, which leads to more distinguishability for $\{\Delta \Tilde{\mathbf{\thetav}}_i^1\}_{i=1}^n$ on the server side (see \Cref{fig:client_updates_differentbatch}, right). 

% The well separation between the Gaussian components also makes \algname{R-DPCFL} robust to the initialization of the \GMM model. 
Note that for a batch size $b^1$, the terms $\Delta_{m,m'}(b^1)$ and $\sigma^1(b^1)$ represent the ``data heterogeneity level across clusters $m$ and $m'$" and ``privacy sensitivity of their clients", respectively. 
% Hence, the first statement of \cref{lemma:localdp} implicitly means that when clusters become more heterogeneous and clients become less privacy-sensitive, it becomes easier for the server to cluster them at the end of the first round based on $\{\Delta \Tilde{\theta}_i^1\}_{i=1}^n$. 
We define their ``separation score" as $\texttt{SS}(m,m') := \frac{\sqrt{p}\Delta_{m,m'}(b^1)}{2\sigma^1(b^1)} = \frac{\Delta_{m,m'}(b^1)}{2\sigma^1(b^1)/\sqrt{p}}$. The larger separation score $\texttt{SS}(m,m')$, the smaller their pairwise overlap $O_{m,m'}= 2Q(\texttt{SS}(m,m'))$. Based on the form of the Q function, an $\texttt{SS}(m,m')$ above 3 can be considered as a complete separation between the corresponding components.

% In the next section, we show that increasing the batch sizes $\{b_i^1\}_{i=1}^n$ also enables us to learn the \GMM faster (in terms of convergence speed of expectation maximization (\EM) ). This shows learning the \GMM on the server in the first round will incur low computational cost. 

\subsection{Convergence rate of \EM for learning \GMM}
\label{sec:cvgce}
Let us define the maximum pairwise overlap between the components in $\psi^*(b^1) = \{\mu_m^*(b^1), \Sigma_m^*(b^1), \alpha_m^*\}_{m=1}^M$, as $O^{\texttt{max}}(\psi^*(b^1)) = \max_{m,m'} O_{m,m'}(\psi^*(b^1))$. According to \Cref{lemma:localdp}, when $b^1$ is large enough, $O^{\texttt{max}}(\psi^*(b^1))$ decreases (like in \Cref{fig:client_updates_differentbatch}, right) and we can expect \EM to converge to the true \GMM parameters $\psi^*(b^1)$. Next, we analyze the local convergence rate of \EM around $\psi^*(b^1)$.

% We now consider some assumptions, regularizing the way $O(\psi^*(b^1))$ tends to  zero as $b^1$ increases (based on $\cref{lemma:localdp}$). 

% \textbf{conditions 1:} $\alpha_m^* \geq \alpha, ~ m \in \{1,2,\cdots,M\},$ where $\alpha$ is a positive number. In other words, the prior probabilities of all the $M$ Gaussian components remain non-zero. 

% \textbf{conditions 2:} $\forall m,m': ~\nu D_{\textit{max}}(\psi^*(b^1)) \leq D_{\textit{min}}(\psi^*(b^1)) \leq \|\mu_m^*(b^1)-\mu_{m'}^*(b^1)\| \leq D_{\textit{max}}(\psi^*(b^1)),$ where $\nu$ is a positive number. It means that the pairwise distance between mean vectors $\{\mu_m^*(b^1)\}_{m=1}^M$ are of the same order and no pair completely overlap as $b^1$ grows. 

% Both the conditions above are satisfied in our setting, as we assume we have $M$ different clusters with heterogeneous data distributions, which make $\{\mu_m^*(b^1)\}_{m=1}^M$ different for all $b^1$. Then, we have:

\begin{restatable}{theorem}{convrate}\citep{Ma2000AsymptoticCR}
Given model updates $\{\Delta \Tilde{\mathbf{\thetav}}_i^1(b^1)\}_{i=1}^n$, as samples from a true mixture of Gaussians $\psi^*(b^1) =\{\mathcal{N}\big(\mu_m^*(b^1), \Sigma_m^*(b^1)\big), \alpha_m^*\}_{m=1}^M$, if $O^{\texttt{max}}(\psi^*(b^1))$ is small enough, then:

\begin{align}
\vspace{-3em} 
    \lim_{r\to\infty} \frac{\|\psi^{r+1}-\psi^*(b^1)\|}{\|\psi^r-\psi^*(b^1)\|} = o\bigg(\big[O^{\texttt{max}}(\psi^*(b^1))\big]^{0.5-\gamma}\bigg), 
    \vspace{-3em} 
\end{align}

as $n$ increases. $\psi^r$ is the \GMM parameters returned by \EM after $r$ iterations. $\gamma$ is an arbitrary small positive number, and $o(x)$ means it is a higher order infinitesimal as $x \to 0: \lim_{x \to 0} \frac{o(x)}{x}=0$. 
\label{theorem:convrate}
\end{restatable}

This means that convergence rate of \EM around the true solution $\psi^*(b^1)$ is faster than how $\big[O^{\texttt{max}}(\psi^*(b^1))\big]^{0.5-\gamma}$ decreases with $b^1$ (from \Cref{lemma:localdp}).
% We showed that $O^{\texttt{max}}(\psi^*(b^1))$ indeed drops fast as $b^1$ increases. Therefore, if clients have a large enough dataset size and use full batch sizes in the first round, convergence rate of $\EM$ approaches approximately $0$. 
\emph{Hence, as an important consequence, the computational complexity of learning the \GMM in the first round also decreases fast as $b^1$ increases}.

\begin{figure*}[t]
\centering
\includegraphics[width=0.4\columnwidth,height=2.5cm]{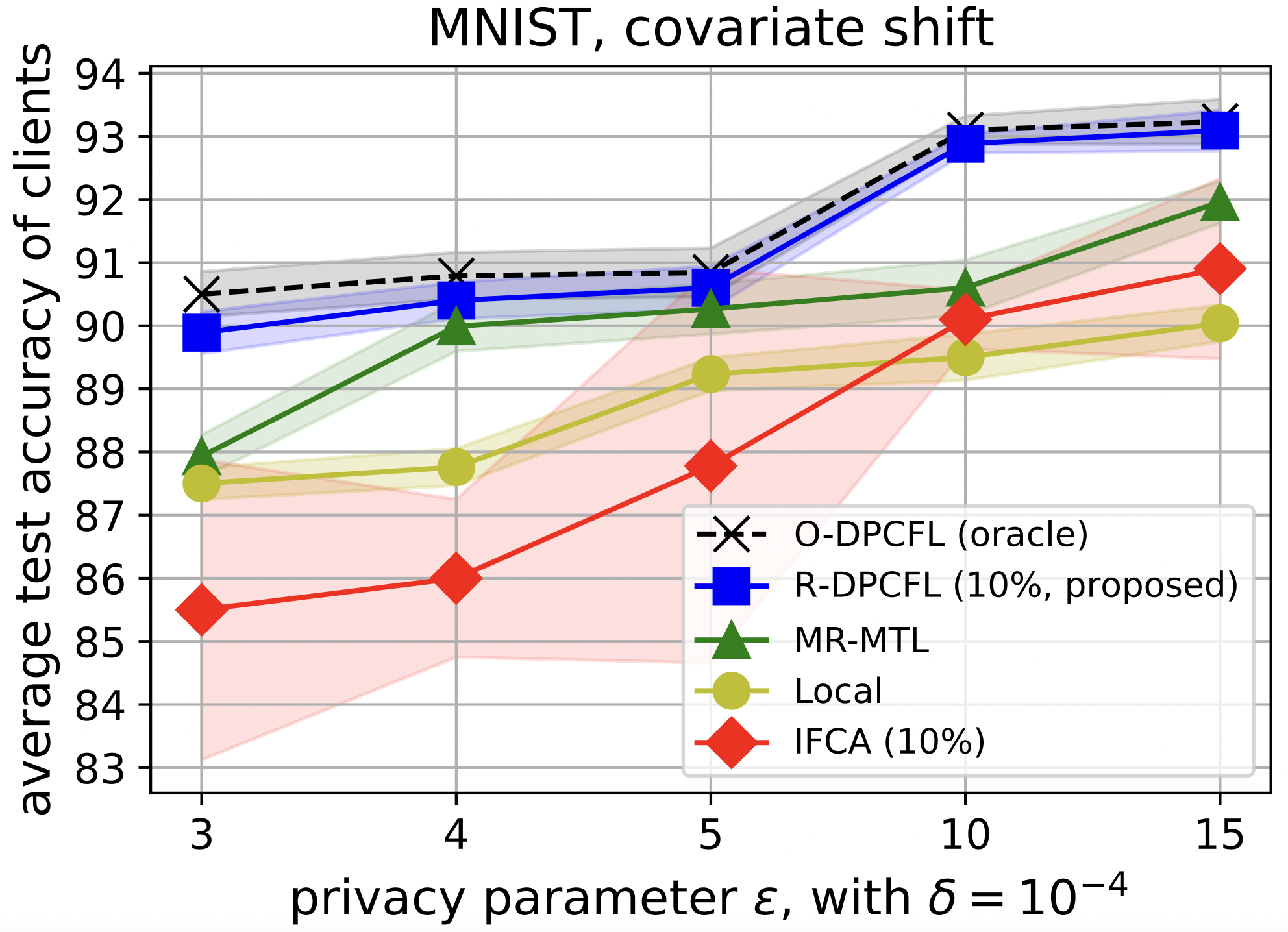}
\includegraphics[width=0.4\columnwidth,height=2.5cm]{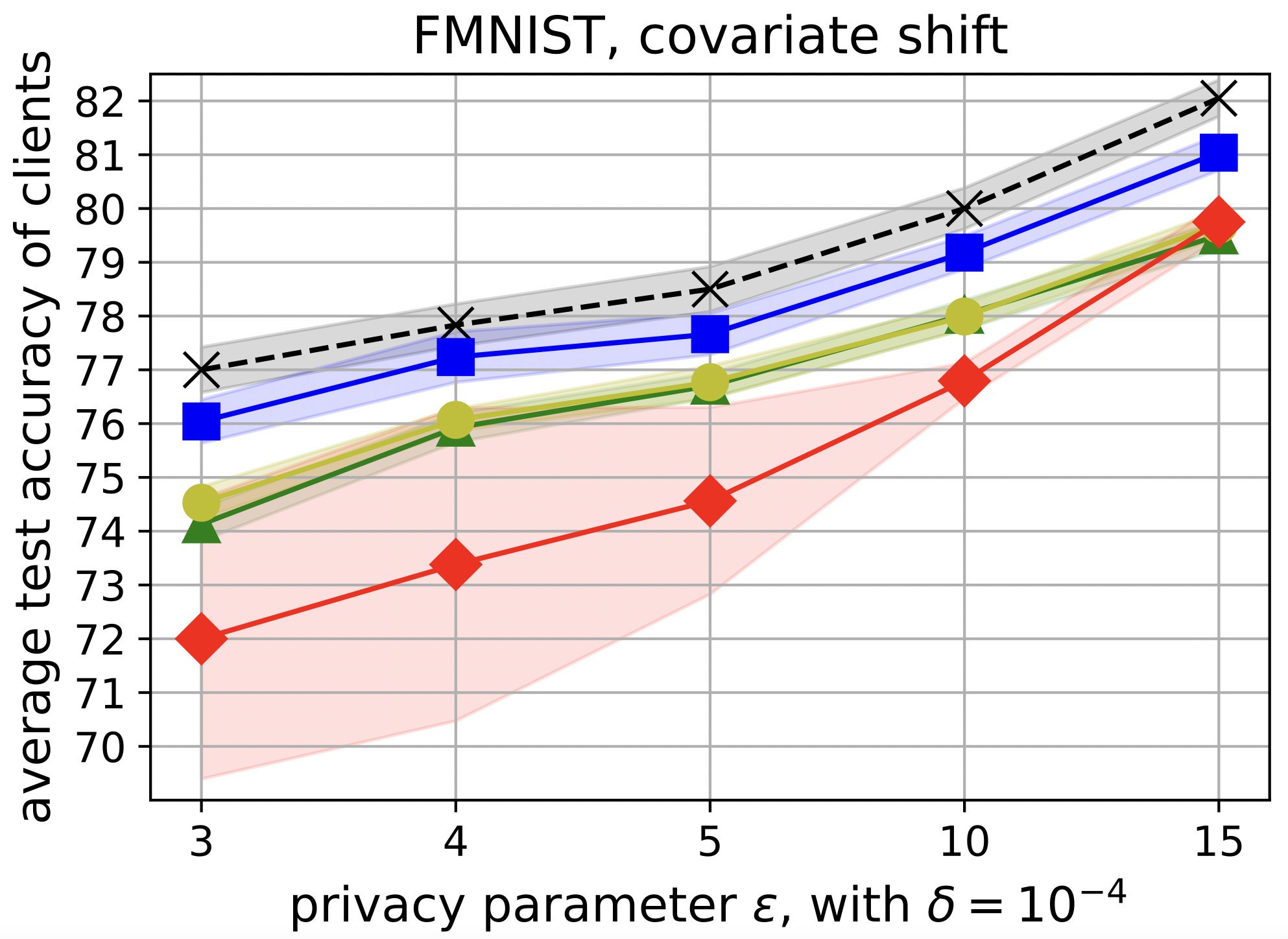}
\includegraphics[width=0.4\columnwidth,height=2.5cm]{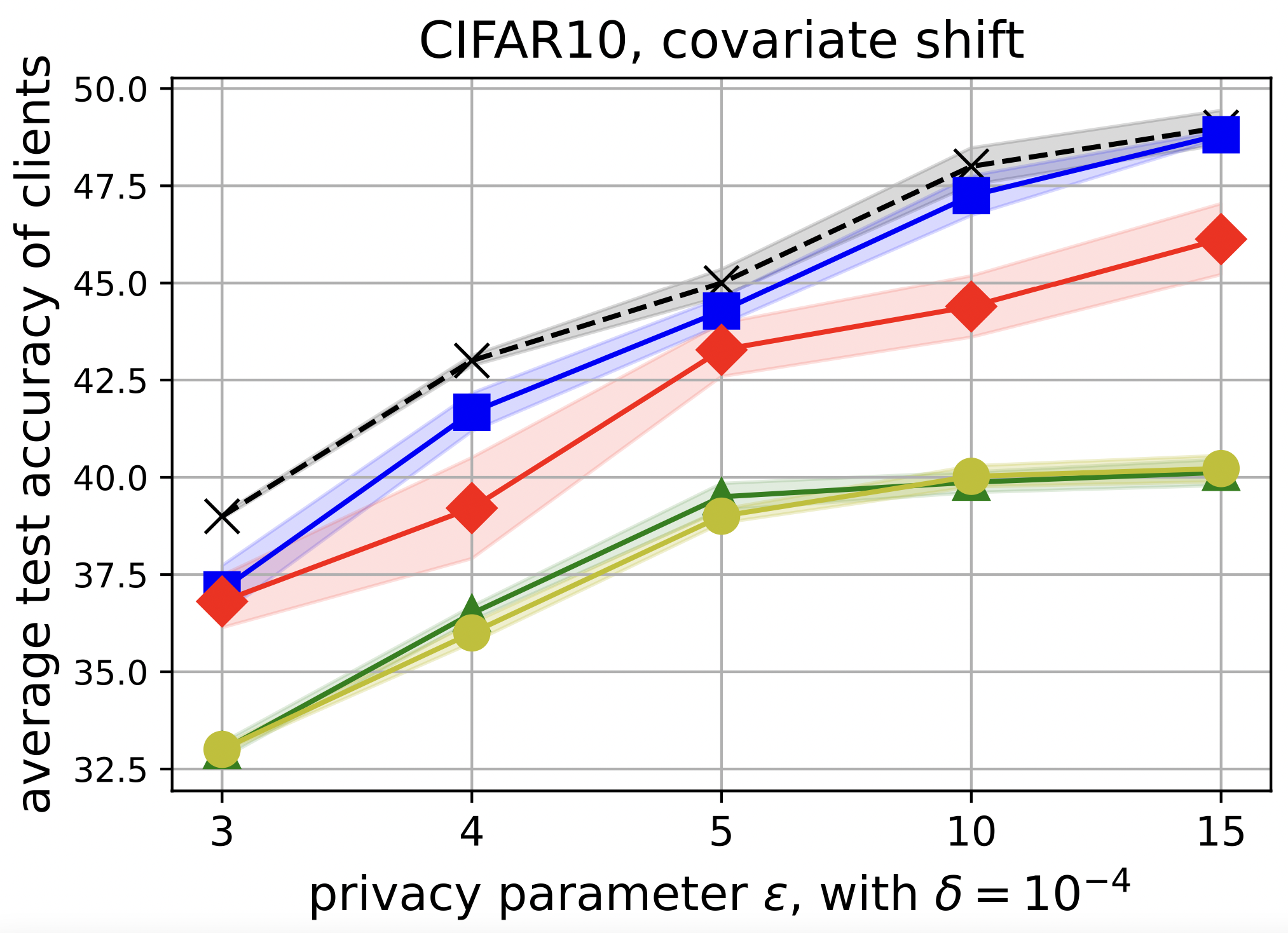}
\includegraphics[width=0.4\columnwidth,height=2.5cm]{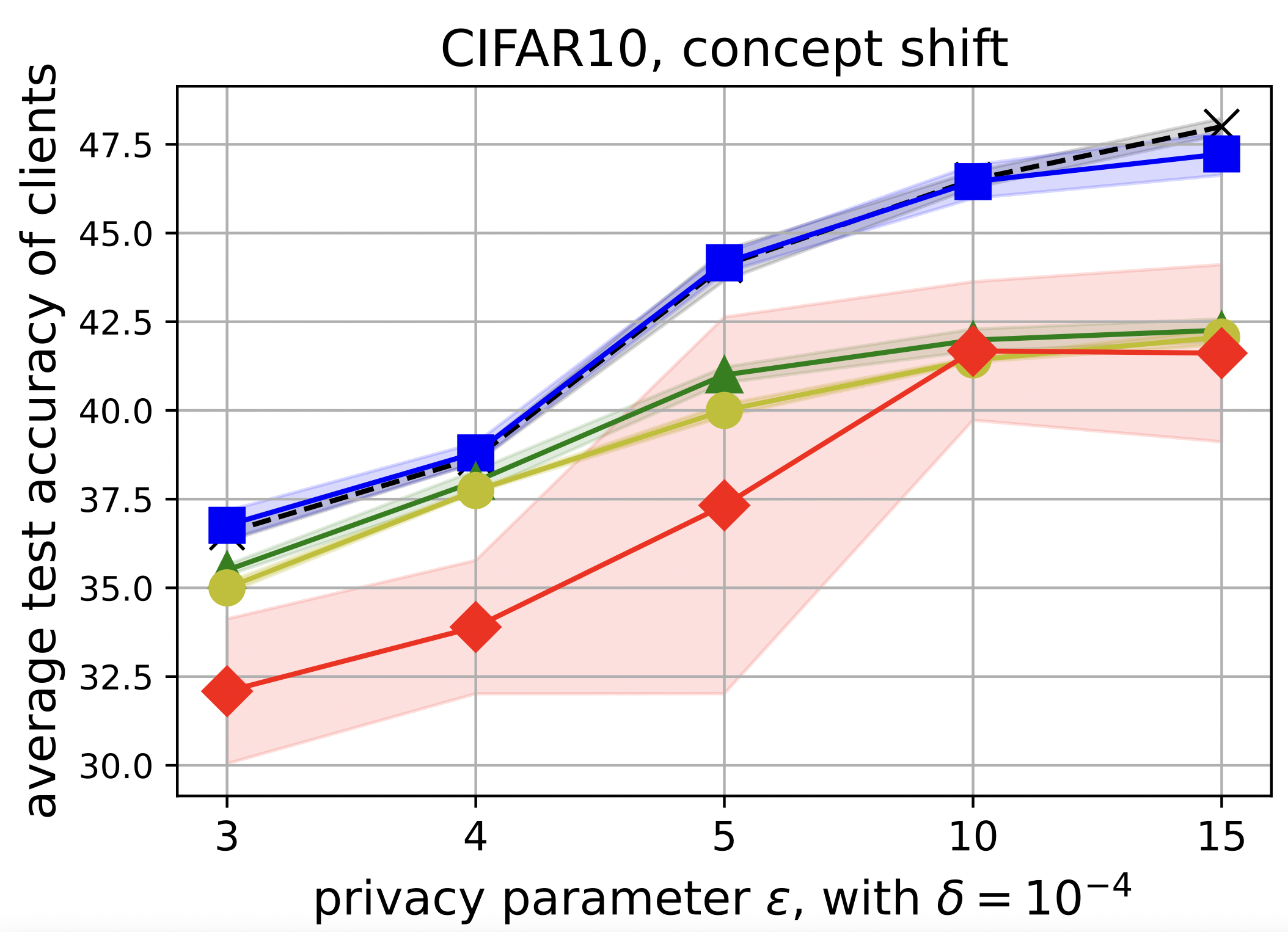}\\
\includegraphics[width=0.4\columnwidth,height=2.5cm]{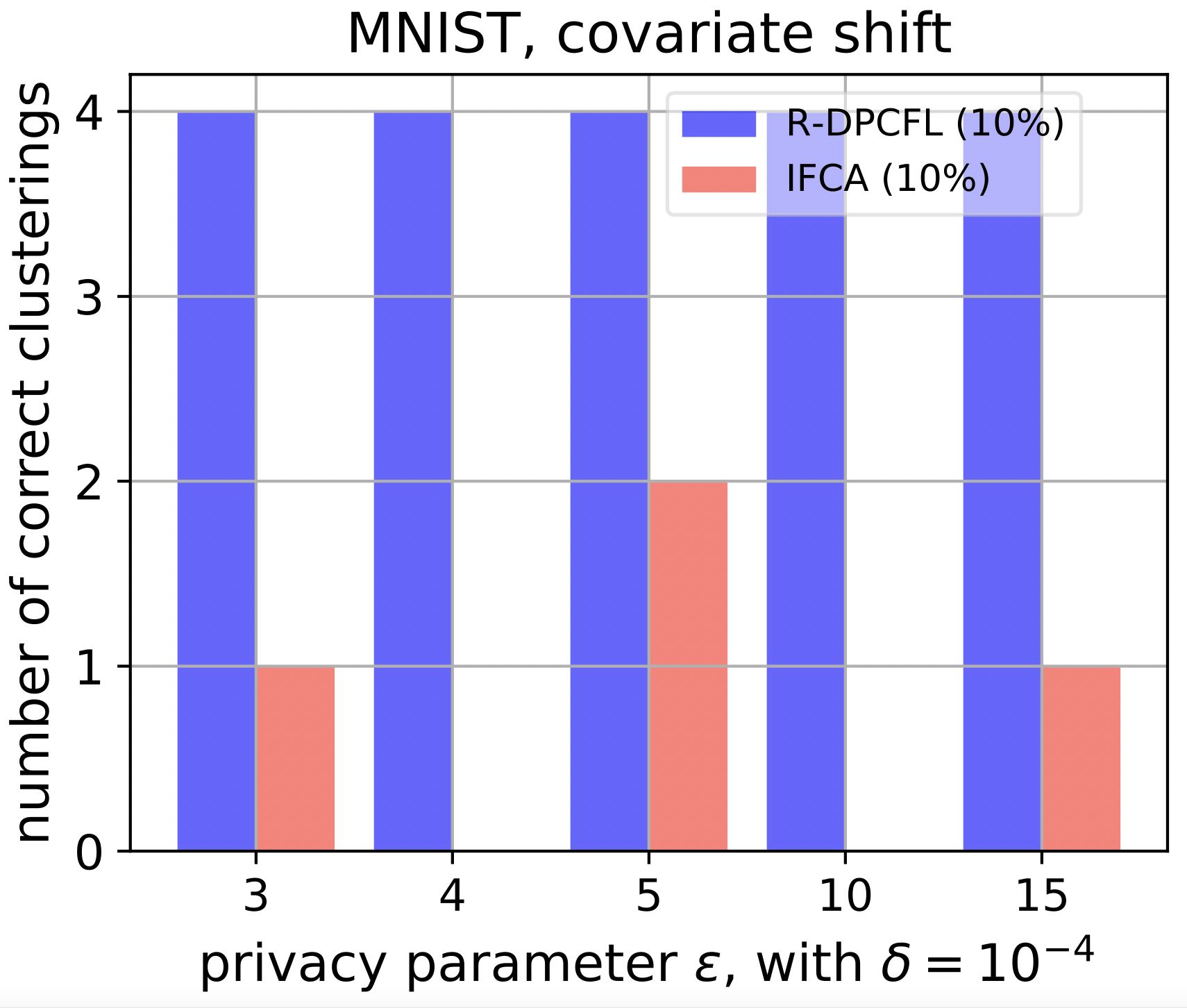}
\includegraphics[width=0.4\columnwidth,height=2.5cm]{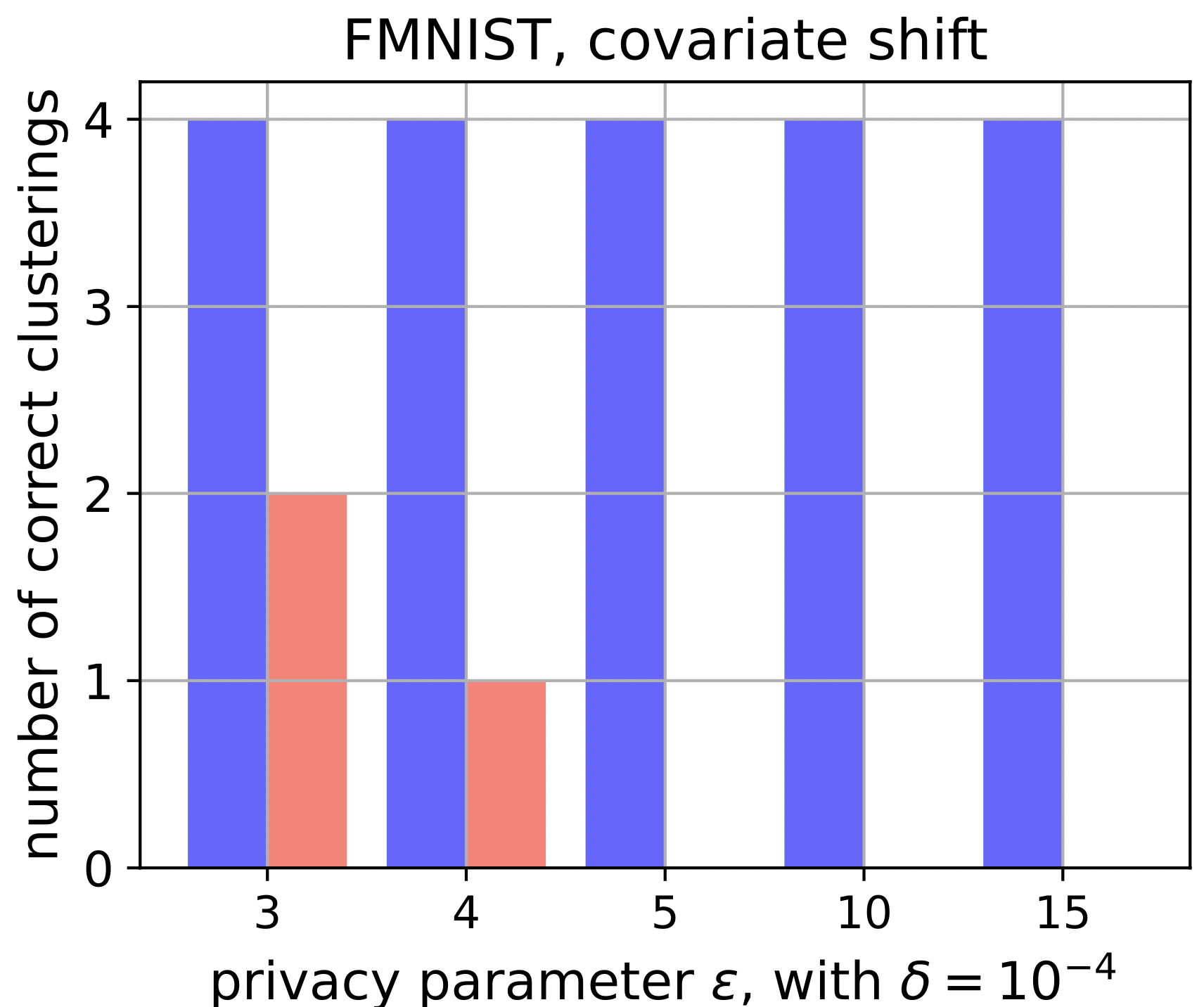}
\includegraphics[width=0.4\columnwidth,height=2.5cm]{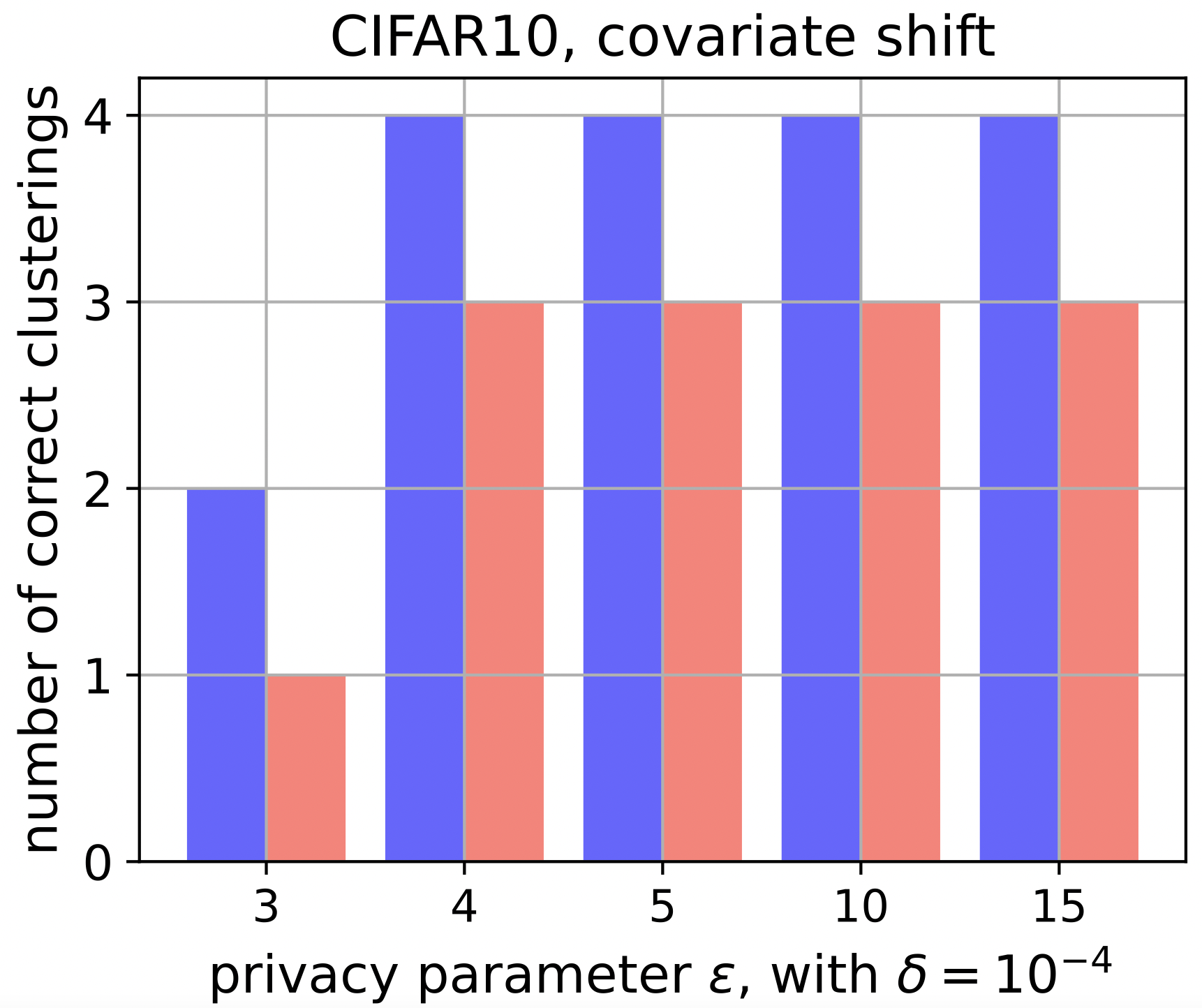}
\includegraphics[width=0.4\columnwidth,height=2.5cm]{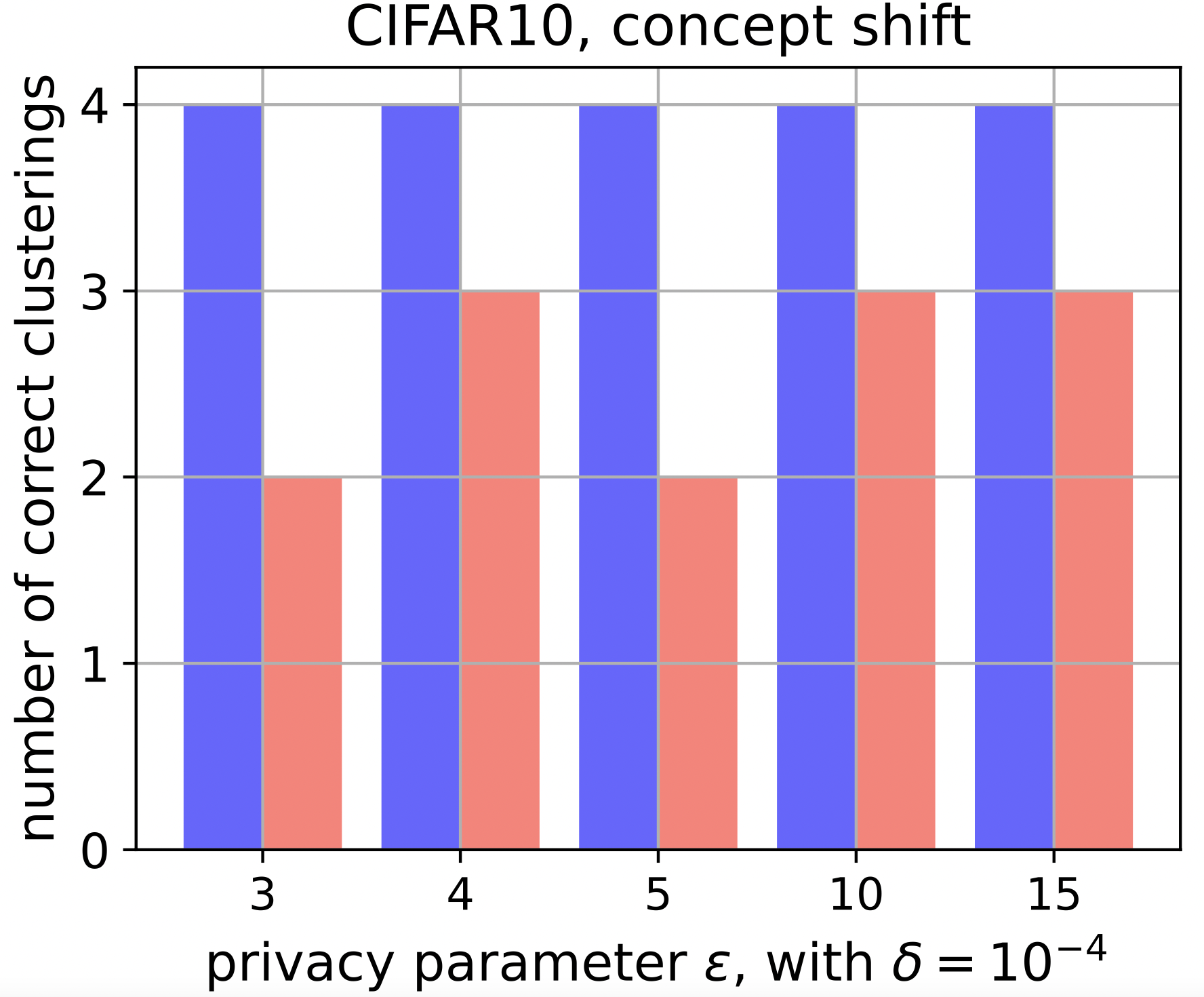}

\caption{\textbf{Top:} Average test accuracy across clients for different total privacy budgets $\epsilon$. Results are from four different runs. $10\%$ means performing local clustering by clients only in $10\%$ of the total number of rounds; i.e., rounds $E_c \leq e \leq E_c + \lfloor \frac{E}{10}\rfloor$ for \algname{R-DPCFL} and rounds $1 \leq e \leq 1+\lfloor \frac{E}{10}\rfloor$ for \algname{IFCA} (see \Cref{app:EMimplementation}). \textbf{\Cref{fig:avg_test_acc_allalgs} in the appendix includes the \algname{Global} baseline too}. \textbf{Bottom:} Number of times (out of 4 runs) that \algname{R-DPCFL} and \algname{IFCA} successfully detect the underlying cluster structure of all existing clients.}
\label{fig:avg_test_acc}
\vspace{-0.5em}
\end{figure*}

\subsection{Applicability of \algname{R-DPCFL}}\label{sec:applicability}

As we observed in \cref{lemma:localdp}, the separation score $\texttt{SS}(m,m')$ (the overlap $O_{m,m'}$) increases (decreases) as $b^1$ increases. Remember that $\texttt{SS}(m,m') = \frac{\Delta_{m,m'}(b^1)}{2\sigma^1(b^1)/\sqrt{p}}$, and note that $\sigma^{1^2}(b^1)/p$ is the value of diagonal elements of covariance matrices of Gaussian components, which the \GMM aims to learn (see \Cref{eq:diagonal_covariance}). Therefore, when the \GMM is learned, we can use its parameters to get an estimate score $\hat{\texttt{SS}}(m,m')$ for every pair of clusters $m$ and $m'$. Then, we can define the ``minimum pairwise separation score" as $\texttt{MSS}(\epsilon, \delta, \{b_i^1\}_{i=1}^n, \{b_i^{>1}\}_{i=1}^n) = \min_{m,m'} \hat{\texttt{SS}}(m,m') \in [0, +\infty)$ as a \textbf{measure of confidence} of the learned \GMM in its identified clusters. The larger the \texttt{MSS} of a learned \GMM, the more ``confident" it is in its clustering decisions. For instance, if we learn a \GMM on \cref{fig:client_updates_differentbatch} left, it will have a much smaller $\texttt{MSS}$ than when we learn a \GMM on \cref{fig:client_updates_differentbatch} right. We can similarly define the estimated ``maximum pairwise overlap" for a learned \GMM as $\texttt{MPO} = 2Q(\texttt{MSS}) \in [0,1)$, as a \textbf{measure of uncertainty} of the learned \GMM. 
We can use \texttt{MSS} and \texttt{MPO} of the learned \GMM to set the switching time $E_c$, batch sizes $\{b_i^{>1}\}_{i=1}^n$ and even the number of underlying clusters $M$ (when it is unknown).  We refer to \cref{app:RDPCFL_hparams} for a detailed explanation.

%Even when the number of the underlying clusters ($M$) is not known beforehand, we can find it with high accuracy based on the confidence metric $\texttt{MSS} \in (0, +\infty)$ defined above (line 9 of \cref{alg:R-DPCFL}). Intuitively, we choose the $M$ which yields to the largest confidence level $\texttt{MSS}$ for the resulting \GMM. We have provided further details about how to find $M$ in these scenarios in \cref{sec:M}. The strategy switching time  $E_c$ can also be set using the uncertainty metric $\texttt{MPO} \in [0,1)$, e.g., $E_c = (1-\texttt{MPO})\frac{E}{2}$. Intuitively, if the learned \GMM is not certain about its clustering decisions, R-DPCFL should not rely on its decisions for a large $E_c$, and vice versa. Furthermore, we already know that in order to have a quality client clustering at the end of the first round, $\{b_i^{>1}\}_{i=1}^n$ should be small (from \cref{fig:var1var2}). We refer to \cref{app:RDPCFL_hparams} for a more detailed discussion on how to tune the three parameters above. 

\section{Evaluation}\label{sec:evaluation}

\begin{figure*}[t]
\centering
\includegraphics[width=0.4\columnwidth,height=2.5cm]{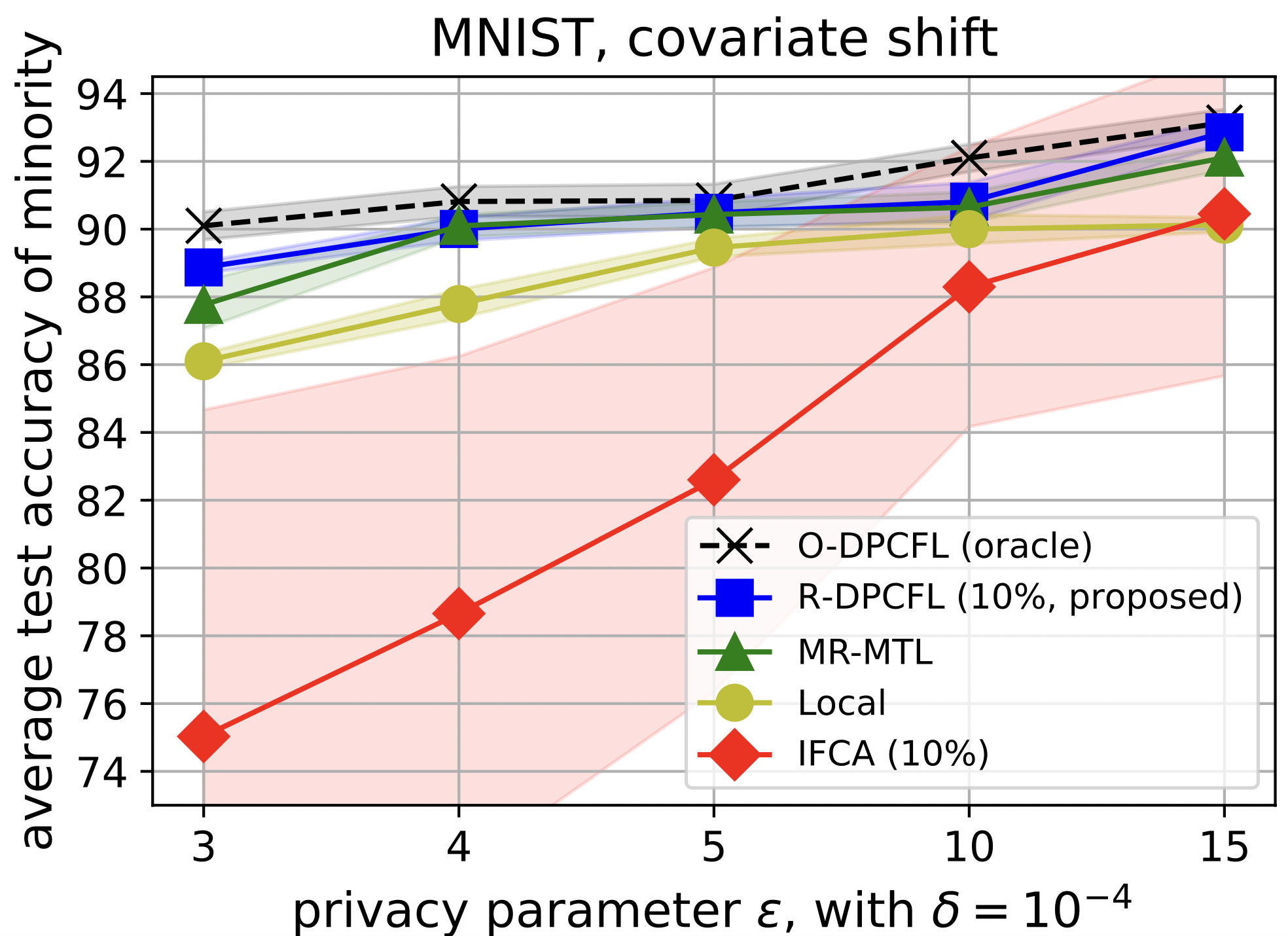}
\includegraphics[width=0.4\columnwidth,height=2.5cm]{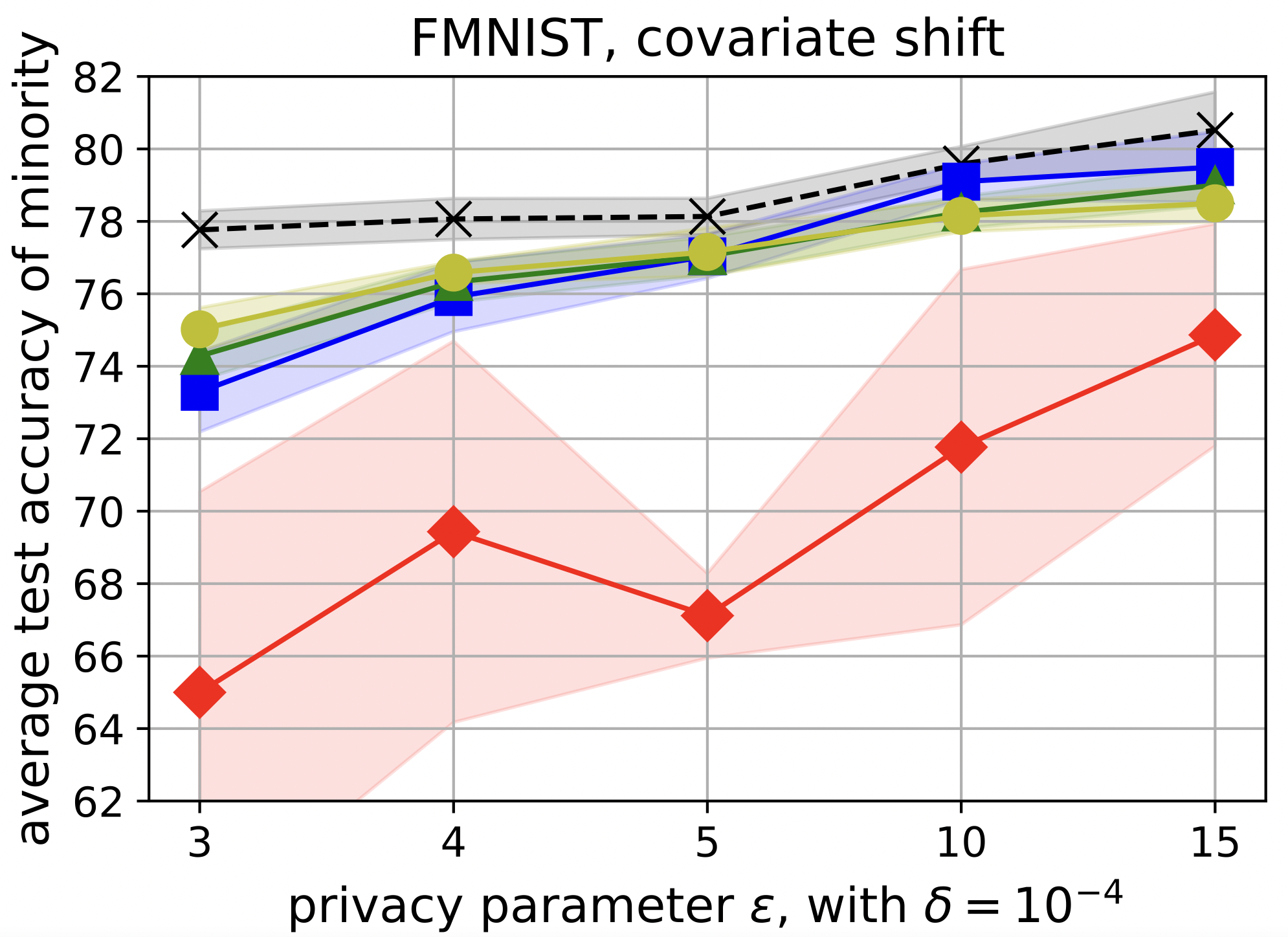}
\includegraphics[width=0.4\columnwidth,height=2.5cm]{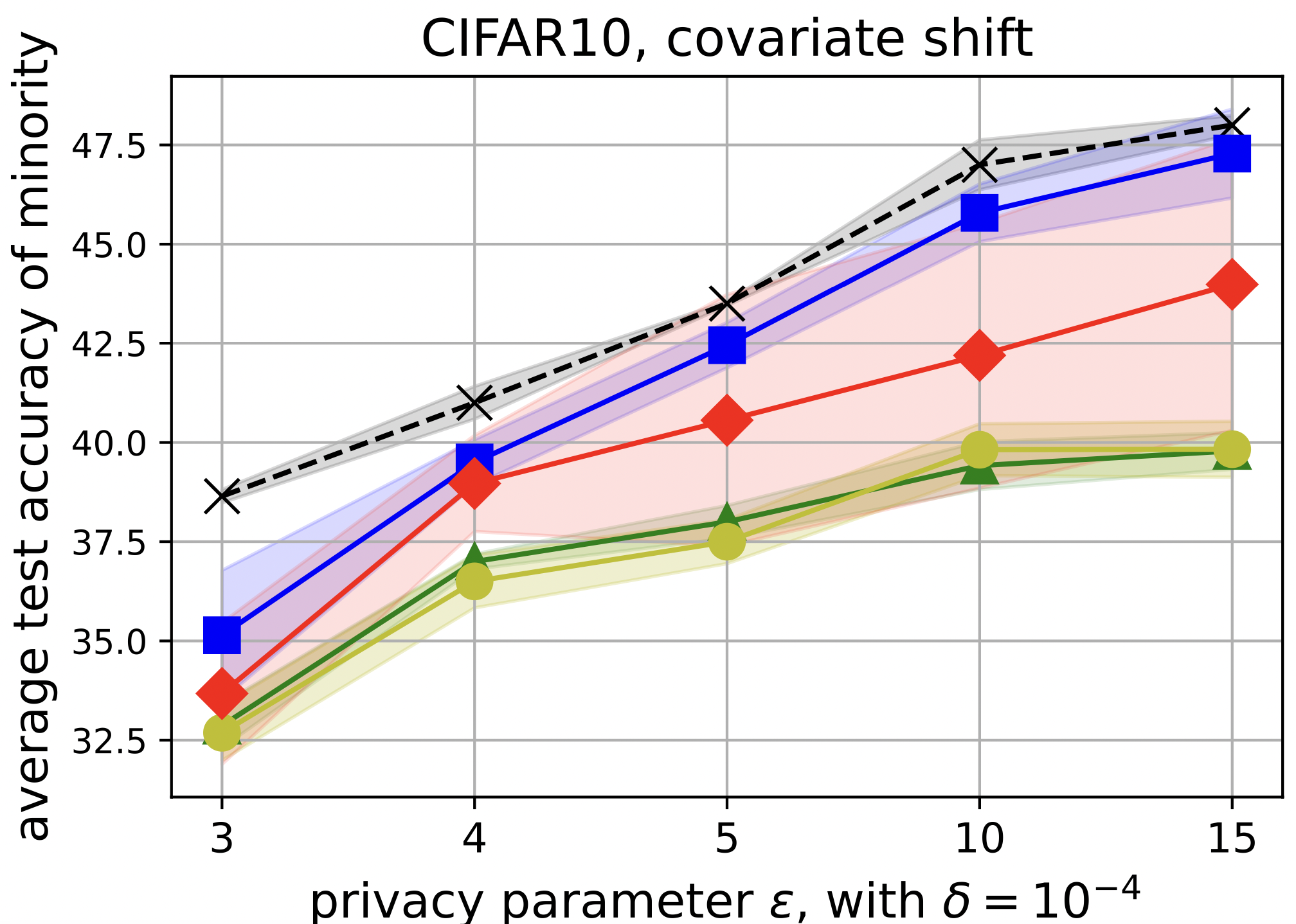}
\includegraphics[width=0.4\columnwidth,height=2.5cm]{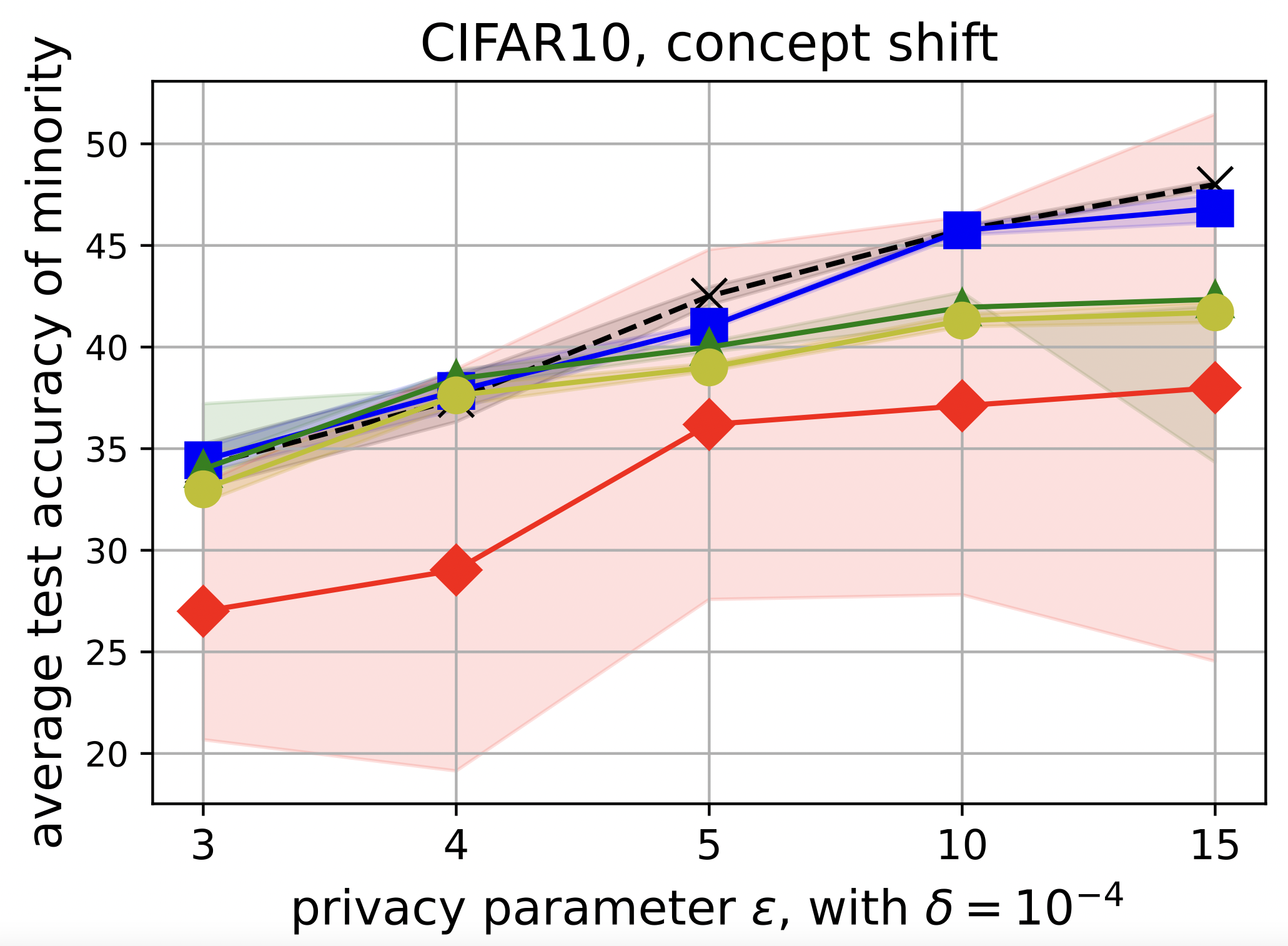}\\
\includegraphics[width=0.4\columnwidth,height=2.5cm]{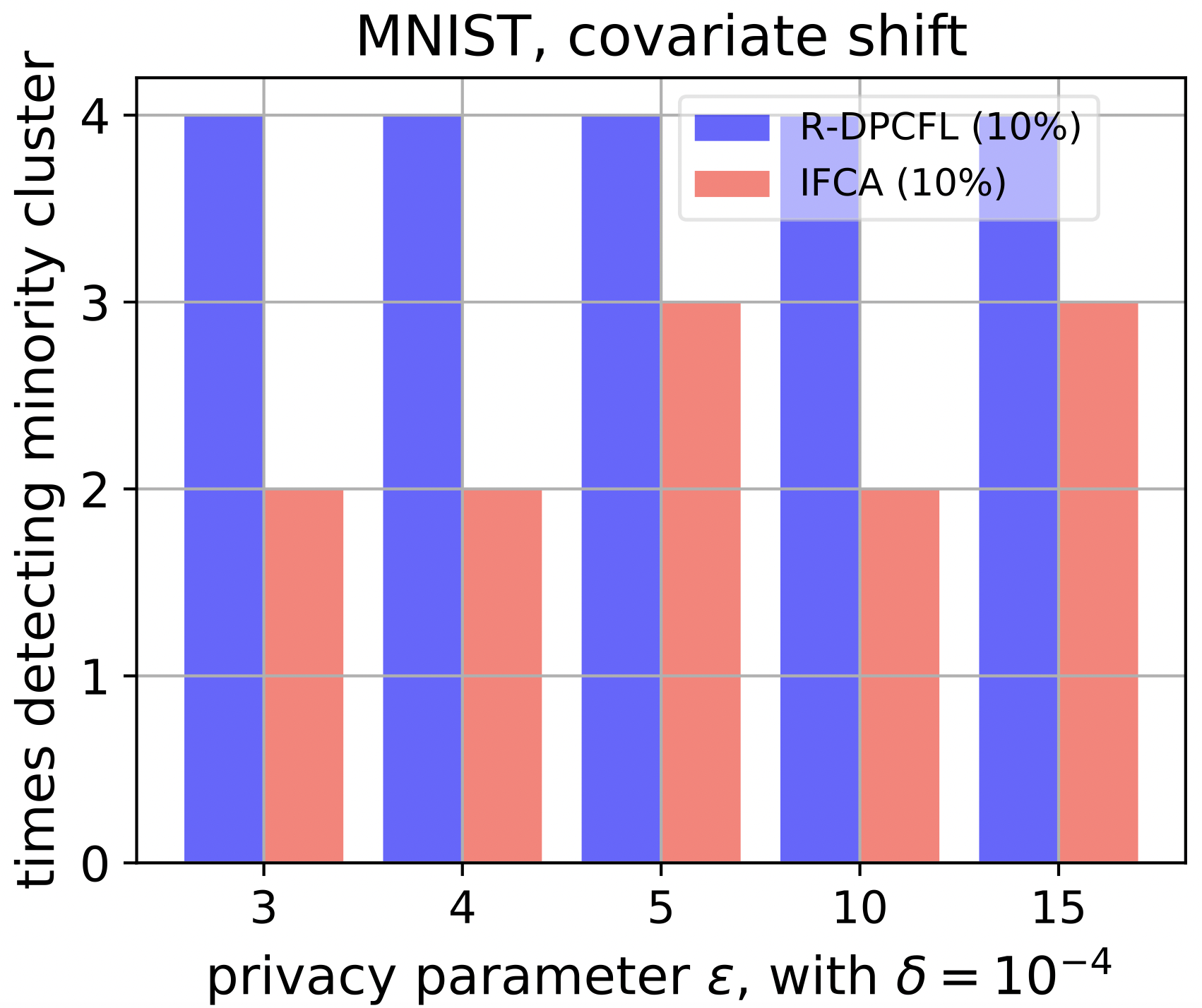}
\includegraphics[width=0.4\columnwidth,height=2.5cm]{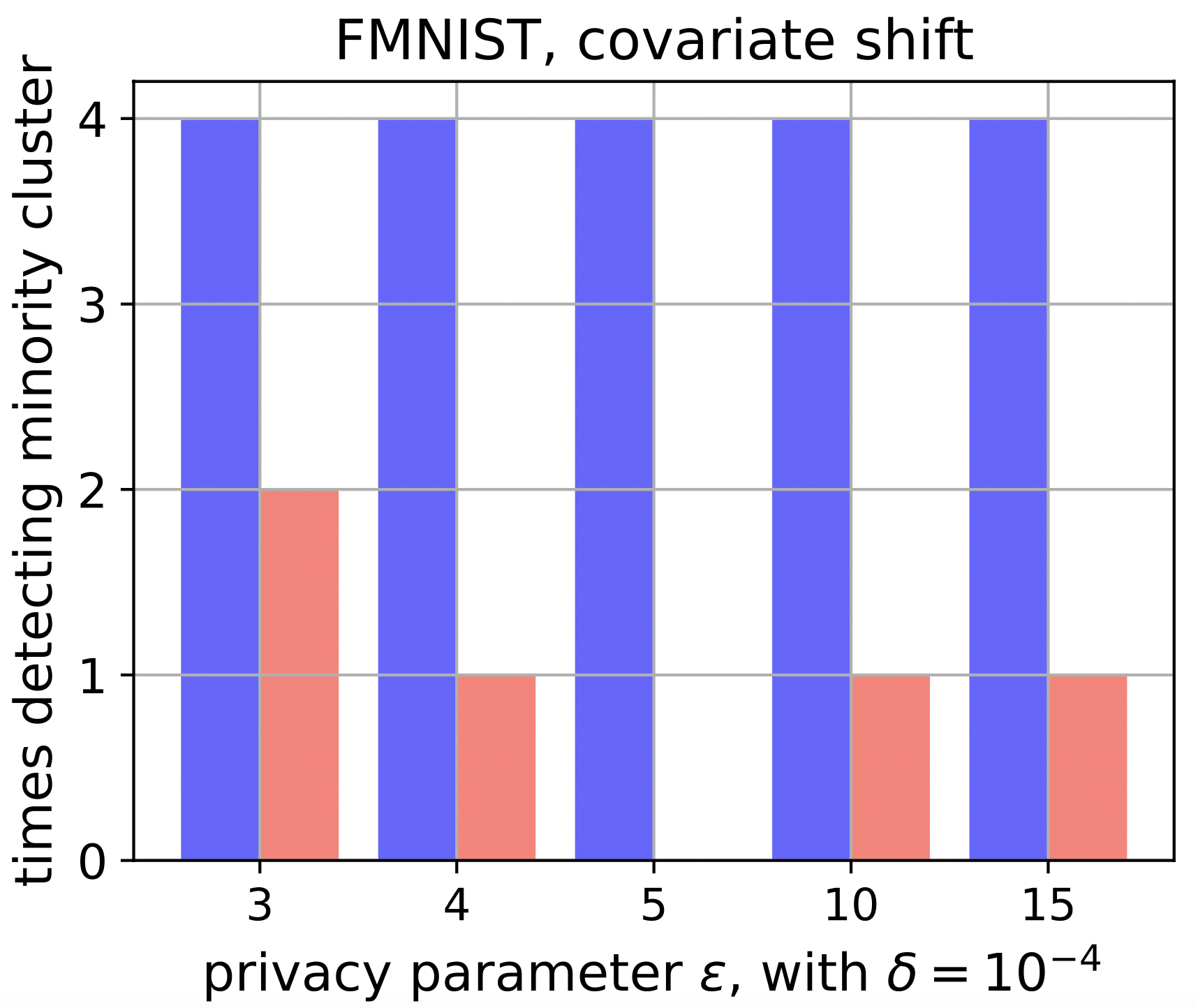}
\includegraphics[width=0.4\columnwidth,height=2.5cm]{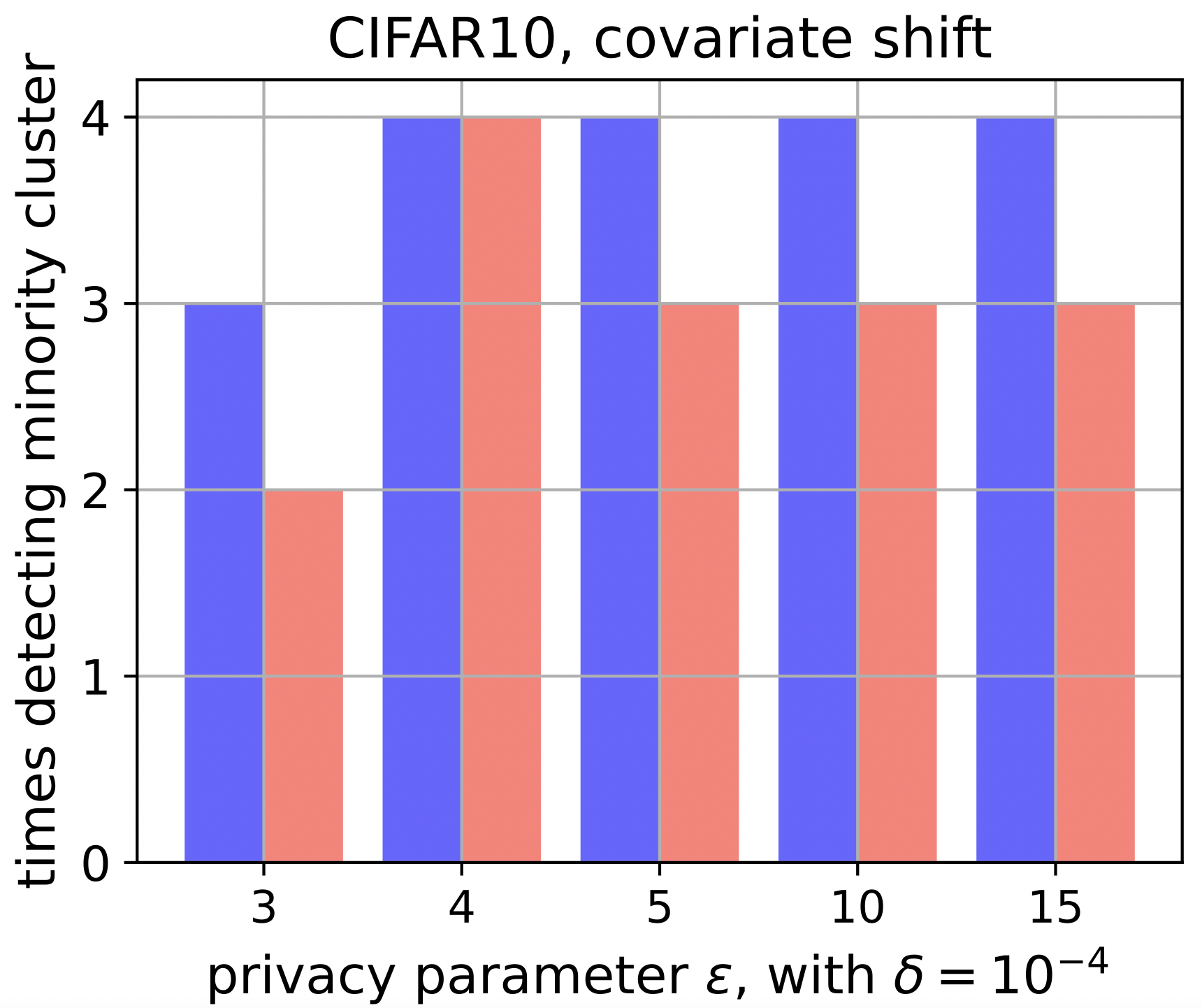}
\includegraphics[width=0.4\columnwidth,height=2.5cm]{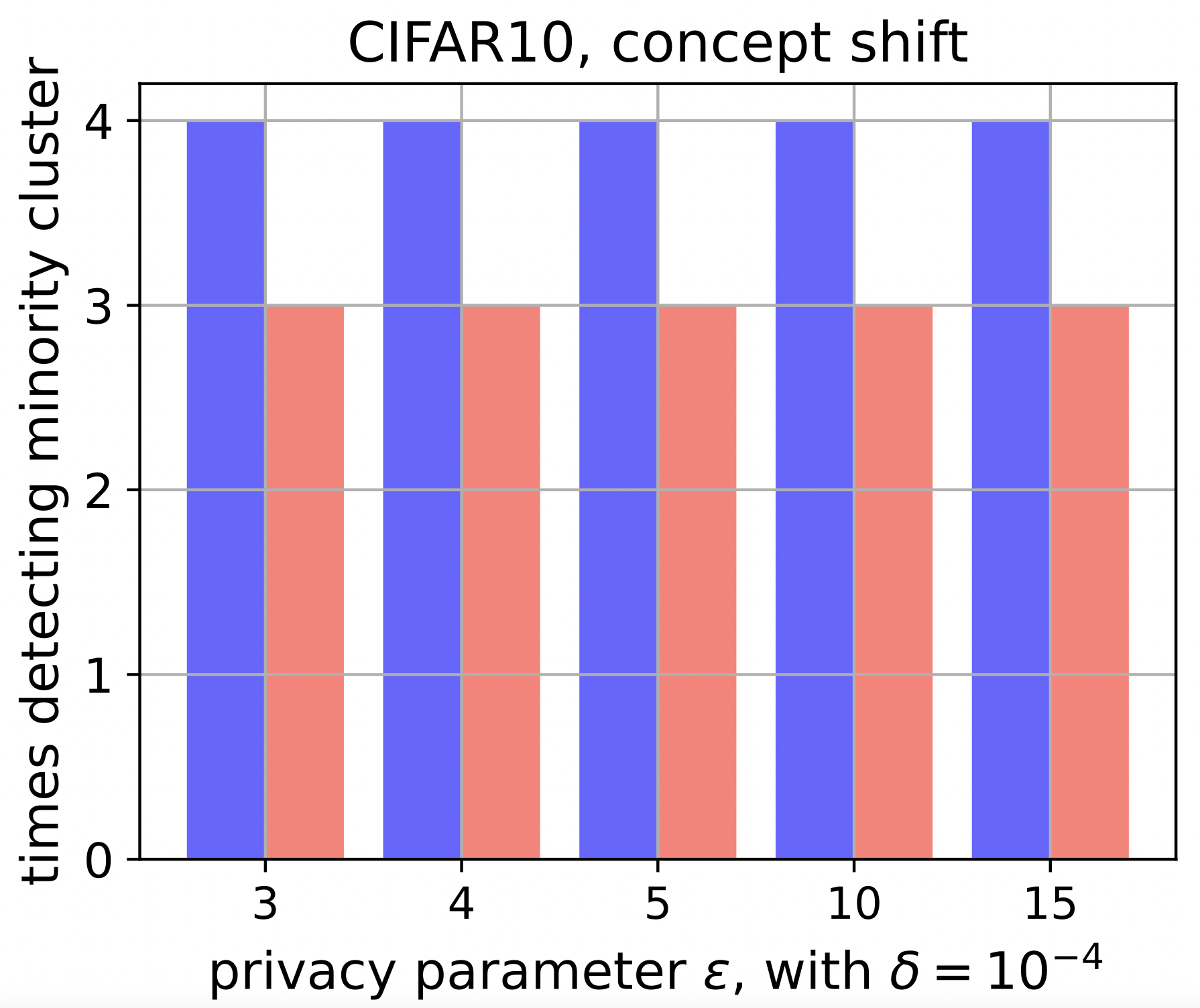}
\caption{\textbf{Top:} Average test accuracy across clients belonging to the minority cluster for different total privacy budgets $\epsilon$, and four different runs. \textbf{Bottom:} Number of times (out of 4 runs) that \algname{R-DPCFL} and \algname{IFCA} successfully detect the minority cluster.}
\vspace{-1em}
\label{fig:avg_test_acc_minority}
\vspace{-0.5em}
\end{figure*}

\textbf{Datasets, models and baseline algorithms:}
We evaluate our proposed method on three benchamark datasets, including: MNIST \citep{mnist}, FMNIST \citep{fmnist} and CIFAR10 \citep{cifar10}, with heterogeneous data distributions from covariate shift (rotation; $P_i(x)$ varies across clusters) \citep{Kairouz2021AOPFL, werner2023provably} and concept shift (label flip; $P_i(y|x)$ varies across clusters) \citep{werner2023provably}, which are the commonly used data splits for clustered \FL (see \Cref{app:exp_setup}). We consider four clusters of clients indexed by $m\in\{0,1,2,3\}$ with $\{3,6,6,6\}$ clients, where the smallest cluster is considered as the minority cluster. We compare our method with most recent related \DPFL algorithms under an equal total sample-level privacy budget $\epsilon$: 1. \algname{Global} \citep{DPSCAFFOLD2022}: clients run \DPSGD locally and send their model updates to the server for aggregation and learning one global model 2. \algname{Local} \citep{liu2022csfl}: clients do not participate \FL and learn a local model by running \DPSGD on their local data 3. A \DP extension of \algname{IFCA} \citep{ghosh2020,liu2022csfl}: local loss/accuracy-based clustering performed by clients on existing cluster models 4. \algname{MR-MTL} \citep{liu2022csfl}: uses model personalization to learn one model for each client 5. \algname{O-DPCFL}: an oracle algorithm which has the knowledge of the true clusters from the first round. For \algname{R-DPCFL} and \algname{IFCA}, we use exponential mechanism \citep{EM}, which satisfies zero concentrated \DP (z-\CDP) \citep{zcdp}, to privatize clients' local cluster selections. See also \Cref{app:full_results} for further experimental results.

%\paragraph{Evaluation metrics:}
%We consider the following evaluation metrics: 1. average test accuracy, as the overall performance 2. average test accuracy of the clients belonging to the minority cluster, as a measure of effect of \DP on minority clients 3. success rate in detecting the underlying cluster structure of clients (minority and overall), as a measure of effectiveness of the clustering algorithm.

\subsection{Results}

\citet{liu2022csfl} observed that under sample-level differential privacy and ``mild" data heterogeneity, federation is more beneficial than local training, because, despite the data heterogeneity across the clients, the model aggregation (averaging) on the server results in reduction of the \DP noise in clients' model updates. However, when there is a high structured data heterogeneity across clusters of clients, the level of heterogeneity is remarkable. Hence, learning one global model through \FL is not beneficial, as one single model can barely adapt to the high level of data heterogeneity across the clusters. Therefore, in \DP clustered \FL systems, local training and model personalization can be better options than global training, as they diminish the adverse effect of the high data heterogeneity. Furthermore, \emph{if one can detect the underlying clusters}, one can perform \FL in each of them separately and will simultaneously benefit from 1. eliminating the effect of data heterogeneity across clusters; 2. reduction of the \DP noise by running \FL aggregation on the server within each cluster. Hence, \emph{if the clustering task is done accurately}, we can expect a further improvement over local training and model personalization. This is exactly what \algname{R-DPCFL} is designed for.

\textbf{RQ1: How does \algname{R-DPCFL} perform in practice?}
\Cref{fig:avg_test_acc} shows the average test accuracy across clients for four datasets. As can be observed, \algname{R-DPCFL} outperforms the baseline algorithms, and this can be attributed to the robust clustering method of \algname{R-DPCFL}, which additionally benefits from the unused information in clients' model updates in the first round and leads to correct clustering of clients (\cref{fig:avg_test_acc}, bottom row). While \algname{R-DPCFL} performs close to the oracle algorithm, \algname{IFCA} has a lower performance due to its errors in detecting the underlying true clusters. For instance, \algname{IFCA} has a clearly low clustering accuracy on MNIST and FMNIST, which leads it to perform even worse than \algname{Local} and \algname{MR-MTL}. In contrast, it has a better clustering performance on CIFAR10 (covariate and concept shifts) and outperforms the two baselines. On the other hand, the reason behind the low performance of \algname{MR-MTL} is that it performs personalization on a global model, which in turn has a low quality due to being obtained from federation across ``all" clients (hence adversely affected by the high data heterogeneity). Similarly, \algname{Local}, which performs close to \algname{MR-MTL}, cannot outperform \algname{R-DPCFL}, as it does not benefit from the \DP noise reduction by \FL aggregation within each cluster.

\textbf{RQ2: How does the minority cluster benefit from \algname{R-DPCFL}?}
\cref{fig:avg_test_acc_minority} compares different algorithms based on the average test accuracy of the clients belonging to the minority cluster. 
%The correct and robust cluster detection by \algname{R-DPCFL} results in separating the minority clients from the rest of the clients. 
\algname{R-DPCFL} leads to a better overall performance for the minority clients, by virtue of its correct and robust cluster detection. Correct detection of the minority cluster prevents it from getting mixed with other majority clusters and leads to a utility improvement for its clients. In contrast, \algname{IFCA} has a lower success rate in detecting the minority cluster (\cref{fig:avg_test_acc_minority}, bottom row) and provides a lower overall performance for them. Similarly, \algname{Local} and \algname{MR-MTL} lead to a low performance for the minority, as they are conditioned on a global model that is learned from federation across all clients and provides a low performance for the minority. In \DP clustered \FL, detecting and improving the performance of minority clusters is important, as failure in detecting them correctly leads to low performance for the clusters with smaller sizes.

\begin{figure}[t]
\centering
\includegraphics[width=0.85\columnwidth, height=4cm]{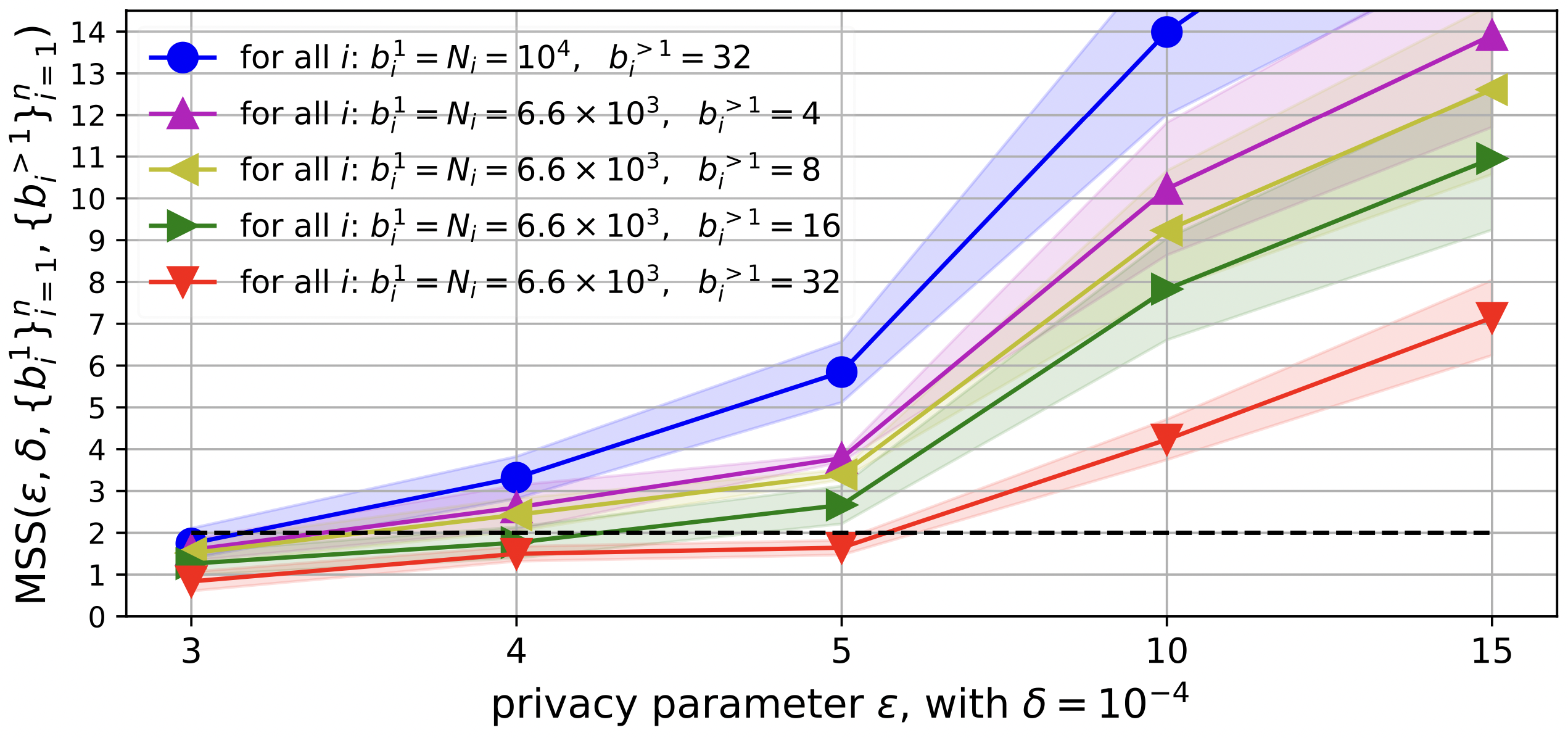}

\caption{\texttt{MSS} score v.s. $\epsilon$ for two different local dataset sizes. A small local dataset size can be compensated for by using smaller batch sizes $\{b_i^{>1}\}_{i=1}^n$ to get a larger \texttt{MSS} score.} 
\label{fig:MSS_batch}
\vspace{-1.2em}
\end{figure}

\textbf{RQ3: What if clients have small local datasets?}
While we envision the proposed approach being more applicable to cross-silo \FL, where datasets are large, it is still worth exploring how beneficial it can be under scarce local data. In the previous sections, we analyzed the benefits of using a full batch size ($b_i^1=N_i$) in the first round and found that it leads to a \GMM with a higher \texttt{MSS} confidence score. The score of the learned \GMM can strongly predict whether the underlying true clusters will be detected: an \texttt{MSS} above 2 almost always yields to correct detection of the underlying clusters (see \Cref{fig:effect_of_MSS1} for experimental results). On the other hand the \texttt{MSS} score depends on four sets of parameters: $\epsilon, \delta, \{b_i^1\}_{i=1}^n$, and $\{b_i^{>1}\}_{i=1}^n$. For fixed $(\epsilon, \delta)$, larger $\{b_i^1\}_{i=1}^n$ and smaller $\{b_i^{>1}\}_{i=1}^n$ increase \texttt{MSS} (\cref{lemma:updatesnoise}). When we ue full batch sizes in the first round, we have $b_i^1=N_i$ (for all $i$). Hence, smaller local datasets result in lower confidence in the learned \GMM. Nevertheless, this can be compensated for by using even smaller $\{b_i^{>1}\}_{i=1}^n$. \Cref{fig:MSS_batch} compares two different dataset sizes under varying $\epsilon$. As observed, for smaller local dataset sizes, reducing $\{b_i^{>1}\}_{i=1}^n$ can help obtain less noisy model updates $\{\Delta \Tilde{\mathbf{\thetav}}i^1\}_{i=1}^n$, improve the \texttt{MSS} score of the learned \GMM and consequently, enable a successful client clustering.

%

%we studied the benefits of using a full batch size ($b_i^1=N_i$) in the first round both empirically and theoretically, and that it results in a \GMM with a larger \texttt{MSS} confidence score. As observed in \Cref{fig:effect_of_MSS1}, the magnitude of \texttt{MSS} score of the learned \GMM, which can be computed at the end of the first round, can be a good indicative value of whether the underlying clusters will be detected correctly or not. Having this in mind, \texttt{MSS} is a function of four set of variables: $\epsilon, \delta, \{b_i^1\}_{i=1}^n$ and $\{b_i^{>1}\}_{i=1}^n$. For fixed privacy parameters $(\epsilon, \delta)$, the larger the values in $\{b_i^1\}_{i=1}^n$ and the smaller the values in $\{b_i^{>1}\}_{i=1}^n$, the larger the \texttt{MSS} of the resulting \GMM. Therefore, when using full batch sizes in the first round $(b_i^1=N_i)$, the smaller the local dataset size of clients, the less confident the resulting \GMM. \Cref{fig:MSS_batch} has compared two scenarios with different local dataset size for clients for different values of privacy parameter $\epsilon$. As we observed in \Cref{lemma:updatesnoise} and \Cref{fig:var1var2}, if the local dataset sizes are not large enough, we could further reduce the batch size $\{b_i^{>1}\}_{i=1}^n$ to get a less noisy set of model updates $\{\Delta \Tilde{\mathbf{\thetav}}_i^1\}_{i=1}^n$ at the end of the first round and subsequently, learn a \GMM with a \texttt{MSS} score above/close to 2 on them in order to have a successful clustering of clients at the end.

\section{Conclusion}
We proposed a \DP clustered \FL algorithm, which addresses sample-level privacy in \FL systems with structured data heterogeneity. By clustering clients based on both their model updates and training loss/accuracy values, and mitigating noise impacts with large initial batch sizes, our approach enhances clustering accuracy and mitigates \DP's disparate impact on utility, all with minimal computational overhead. Moreover, the robustness to noise and easy parameter selection of the proposed approach shows its applicability to \DP clustered \FL settings. While envisioned for \DPFL systems with large local datasets, the method is capable of compensating for moderate dataset sizes by using smaller batch sizes after the first round. In the future, we aim to extend this approach to be suitable for scarce data scenarios.

\bibliography{example_paper}
\bibliographystyle{icml2025}

% %!TEX root = main.tex
\appendix
\onecolumn
\newpage
\begin{center}
\Large
\bf
Appendix for \emph{Differentially Private Clustered Federated Learning}
\end{center}

\section{Notations}
\Cref{tab:notations} summarizes the notations used in the paper.

\begin{table}[hbt!]
    \caption{Used notations}
    \begin{tabularx}{\columnwidth}{p{0.15\columnwidth}X}
    \toprule 
      % \multicolumn{2}{l}{{\underline{Indices}: }} \\
      $n$ & number of clients, which  are indexed by $i$\\
      $x_{ij}, y_{ij}$ & $j$-th data point of client $i$ and its label \\ 
      $\mathcal{D}_i, N_i$ & local train set of client $i$ and its size \\
      $\mathcal{D}_{i,aug}$ & augmented local train set of client $i$ \\
      $\mathcal{B}_i^{e,t}$& the train data batch used by client $i$ in round $e$ and at the $t$-th gradient update \\
      $b_i^e$ & batch size of client $i$ in round $e$: $|\mathcal{B}_i^{e,t}|=b_i^e$\\
      $b_i^1$ & batch size of client $i$ in the first round $e=1$\\
      $b_i^{>1}$ & set of batch sizes of client $i$ in the rounds $e>1$\\
      $\epsilon, \delta$ & desired \DP privacy parameters \\
      $E$ & total number of global communication rounds in the \DPFL system, indexed by $e$\\
      $\thetav_m^e$ & model parameter for cluster $m$, at the beginning of global round $e$ \\
      $K$ & number of local train epochs performed by clients during each global round $e$\\
      $\eta_l$ & the common learning rate used for \DPSGD\\
      $h$ & predictor function, e.g., CNN model, with parameter $\thetav$ \\
      $\ell$ & cross entropy loss\\
      $s(i)$ & the true cluster of client $i$\\
      $R^e(i)$ & the cluster assigned to client $i$ in round $e$\\
      $\thetav_i^{e,0}$& the model parameter passed to client $i$ at the beginning of round $e$ to start its local training\\
      $\Delta \Tilde{\thetav}_i^e$& the noisy model update of client $i$ at the end of round $e$, starting from $\thetav_i^{e,0}$\\
      $\sigma_{i}^{e^2}$ & conditional variance of the noisy model update $\Delta \Tilde{\thetav}_i^e$ of client $i$: $\texttt{Var}(\Delta \Tilde{\thetav}_i^e|\thetav_i^{e,0})$ \\
      $\mu_m^*(b^1)$ & the center of the $m$-th cluster (when all clients use batch size $b^1$ in the first round)\\
      $\Sigma_m^*(b^1)$ & the covariance matrix of the $m$-th cluster (when all clients use batch size $b^1$ in the first round)\\
      $\alpha_m^*$ & the prior probability of the $m$-th cluster\\

      \bottomrule
     \end{tabularx}
     \label{tab:notations}
\end{table}

\begin{figure}[hbt!]
    \centering
    \hspace{3em}\includegraphics[width=0.4\columnwidth,height=4cm]{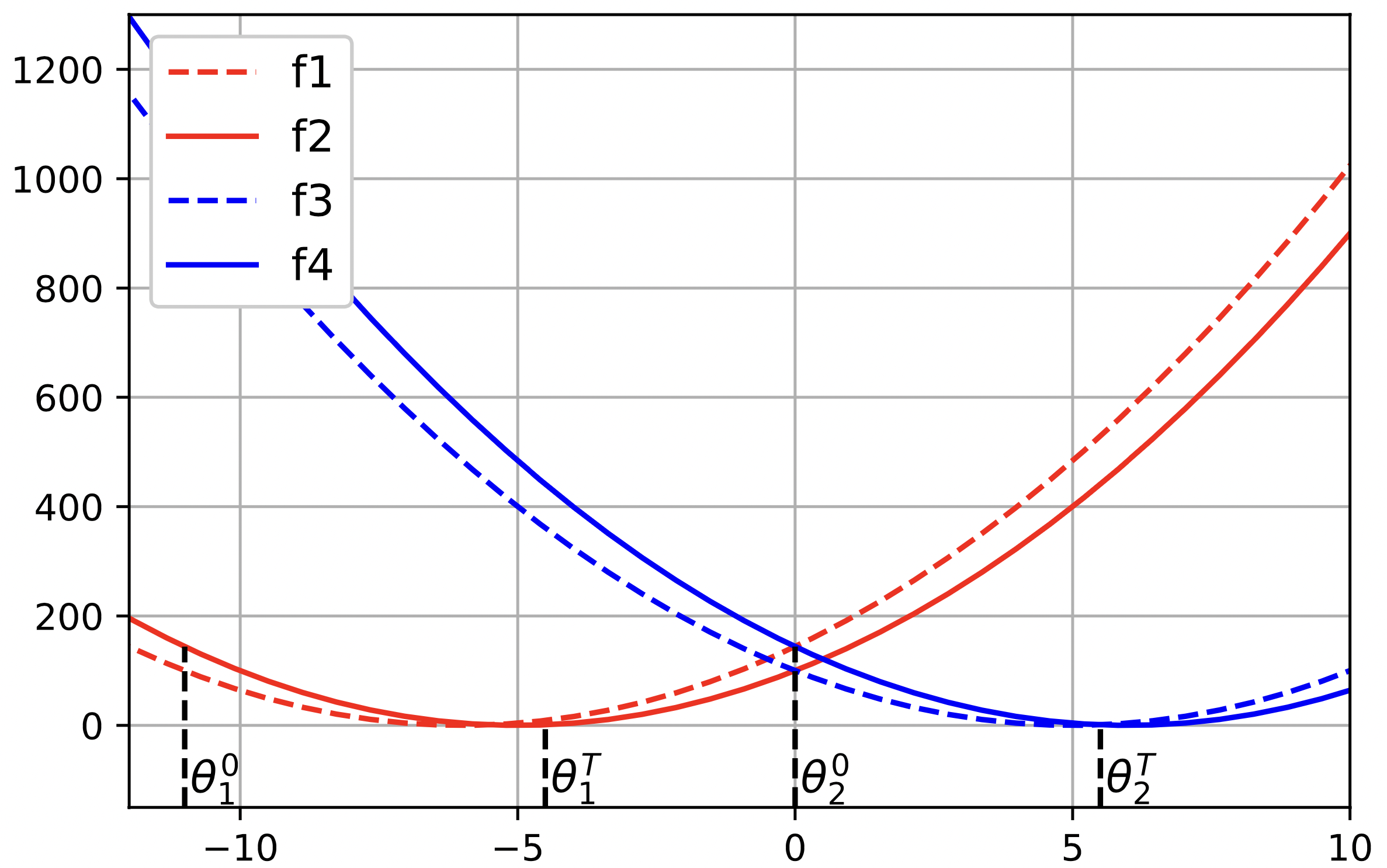}
    \caption{Loss-based clustering algorithms miscluster in the initial rounds, due to model initialization. Also, even with the assumption of perfect clustering of clients in the first rounds, clustering algorithms based on gradients (model updates) lead to clustering errors in the last rounds, due to the gradients approaching to zero.}
    \label{fig:grad_loss}
\end{figure}

\section{Vulnerability of existing clustered \FL algorithms}\label{app:grad_loss}

As discussed in \citep{werner2023provably}, clustered \FL algorithms which cluster clients based on their loss values \citep{mansour2020approaches, ghosh2020, fedsoft2022}, i.e., assign client $i$ to cluster $R^e(i) = \argmin_m f_i(\mathbf{\thetav}_m^e)$ at the beginning of round $e$, are prone to clustering errors in the first few rounds, mainly due to random initialization of cluster models $\{\thetav_m^e\}_{m=1}^M$. On the other hand, clustering clients based on their model updates (gradients) \citep{werner2023provably, briggs2020federated, Sattler2019ClusteredFL} makes sense only when the updates are obtained on the same model initialization. Additionally, even if we assume these algorithms can initially cluster clients perfectly in each round $e$, the clients' model updates (gradients) will approach zero as the clusters' models converge to their optimum parameters. Hence, clients from different clusters may appear to belong to the same cluster, which results in clustering mistakes.

We now provide an example to elaborate that why clustering clients based on their losses (model updates) is prone to errors in the first (last) rounds. For example, consider \Cref{fig:grad_loss}, where there are $M=2$ clusters (red and blue) and $n=4$ clients. The clients in the red cluster have loss functions $f_1(\theta)=4(\theta + 6)^2$ and $f_2(\theta)=4(\theta + 5)^2$ with optimum cluster parameter $\theta_1^\infty=-5.5$. Also, the the clients in the blue cluster have loss functions $f_3(\theta)=4(\theta - 5)^2$ and $f_4(\theta)=4(\theta - 6)^2$ with optimum cluster parameter $\theta_2^\infty=5.5$. Clustering algorithms, which cluster clients based on their loss values on clusters' models, are vulnerable to model initialization. For example, in \cref{fig:grad_loss}, if we initialize the clusters' parameters with $\theta_1^0=-11$ and $\theta_2^0=0$ (shown in the figure), all four clients will initially select cluster 2, since they have smaller losses on its parameter. At $\theta_2^0=0$, the average of clients' gradients (model updates) is zero, so all clients will remain stuck at $\theta_2^0$ and will always select cluster 2.

 On the other hand, clustering clients based on their model updates (gradients) \citep{werner2023provably, briggs2020federated, Sattler2019ClusteredFL} have clearly issues. One of these issues appears after some rounds of training. For instance, even if we assume these algorithms can initially cluster clients ``perfectly" in each round $e$, the clients' model updates (gradients) will approach zero as the clusters' models converge to their optimum parameters. Hence, clients from different clusters may appear to belong to the same cluster, which results in clustering mistakes. For example, as shown in \cref{fig:security_model}, right, let us assume after $T$ rounds of ``correct" clustering of clients, the clusters' parameters get to $\theta_1^T=-4.5$ and $\theta_2^T=5.5$ (shown in the figure). At this parameters, clients 1 and 2 (which have been ``correctly" assigned to cluster 1 so far) will have gradients $f_1'(\theta_1^T)=12$ and $f_2'(\theta_1^T)=4$. Similarly, clients 3 and 4 (which have been ``correctly" assigned to cluster 2 so far) will have $f_3'(\theta_2^T)=4$ and $f_4'(\theta_2^T)=-4$. We see that $f_2'$ is closer to $f_3'$ and $f_4'$ than to $f_1'$, and in the next round it will wrongly be assigned to wrong cluster 2. This happens while the clients are clearly distinguishable based on their losses, as some progress in training has been made after $T$ rounds: $f_1(\theta_1^T)=9$, while $f_1(\theta_2^T)=23^2$, which clearly means that client 1 correctly belongs to cluster 1. Therefore, after making some progress in training the clusters' models, it makes more sense to use a loss-based clustering strategy than using a strategy based on clients' gradients (model updates).

\section{Background}
\subsection{Renyi Differential Privacy (RDP)}
We have used a relaxation of Differential Privacy, named Renyi \DP (RDP) for tight privacy accounting of different algorithms \citep{mironovRDP}. It is defined as follows:

\begin{definition}[Renyi Differential Privacy (RDP) \citep{mironovRDP}]
\label{def:rdp}
A randomized mechanism $\mathcal{M}:\mathcal{A}\to \mathcal{R}$ with domain $\mathcal{D}$ and range $\mathcal{R}$ satisfies $(\alpha, \epsilon)$-\RDP with order $\alpha>1$ if for any two adjacent inputs $\mathcal{D}$, $\mathcal{D}'\in \mathcal{A}$, which differ only by a single record,
\begin{align*}
    D_{\alpha}\big(\mathcal{M}(\mathcal{D})||\mathcal{M}(\mathcal{D}')\big) \leq \epsilon,
\end{align*}
\end{definition}
where $D_{\alpha}(P||Q)$ is the Renyi divergence between distributions $P$ and $Q$:

\begin{align}
    D_{\alpha}(P||Q) := \frac{1}{\alpha - 1} \log \mathbb E_{x\sim p} \bigg[ \bigg(\frac{P(x)}{Q(x)}\bigg)^{\alpha - 1}\bigg] ~~~~~ (\alpha >1).
\end{align}
For $\alpha = 1$, we have $D_{1}(P||Q) := \mathbb E_{x\sim p} \bigg[\log \big(\frac{P(x)}{Q(x)}\big)\bigg]$, which is the KL divergence between $P$ and $Q$. \RDP can be used for composition of private mechanisms \citep{mironovRDP}:

\begin{restatable}{proposition}{prop1}
Let $\mathcal{M}_1: \mathcal{A} \to \mathcal{R}_1$ be $(\alpha, \epsilon_1)$-\RDP and $\mathcal{M}_2: \mathcal{R}_1 \times \mathcal{A} \to \mathcal{R}_2$ be $(\alpha, \epsilon_2)$-\RDP. Then, the mechanism $\mathcal{M}_3 = (X, Y)$ defined as $X \gets \mathcal{M}_1(\mathcal{A})$ and $Y \gets \mathcal{M}_2(\mathcal{A}, X)$ satisfies $(\alpha, \epsilon_1 + \epsilon_2)$-\RDP. 
\label{prop:1mironov}
\end{restatable}

Therefore, if an algorithm has $E$ steps and each step satisfies $(\alpha, \epsilon)$-\RDP, the algorithm satisfies $(\alpha, E\epsilon)$-\RDP. 
\RDP can also be used for composition of \emph{heterogeneous} private mechanisms, e.g., for accounting privacy of \algname{R-DPCFL}, which uses different batch sizes in the first and the subsequent rounds. The following lemma is about conversion of $(\alpha, \epsilon)$-\RDP to standard $(\epsilon, \delta)$-\DP (\cref{def:epsilondeltadp}).

\begin{restatable}{lem}{rdptodp}
If a mechanism $\mathcal{M}$ satisifes $(\alpha, \epsilon(\alpha))$-\RDP, then for any $\delta>0$, it satisfies $(\epsilon(\delta), \delta)$-\DP, where 
\begin{align}
    \epsilon(\delta) = \inf_{\alpha>1} \epsilon(\alpha) + \frac{1}{\alpha - 1} \log \big(\frac{1}{\alpha \delta}\big) + \log \big(1-\frac{1}{\alpha}\big).
\end{align}
\label{lemma:rdptodp}
\end{restatable}

The Gaussian mechanism satisfies $(\alpha, \epsilon)$-\RDP, based on the following Proposition from \citep{mironovRDP}:

\begin{restatable}{proposition}{prop7}
If $f: \mathcal{A} \to \mathcal{R}$ has sensitivity $c$, then its randomization with a Gaussian mechanism with noise variance $\sigma_{DP}^2$ satisfies $(\alpha, \frac{\alpha c^2}{2\sigma_{DP}^2})$-\RDP.
\label{prop:7mironov}
\end{restatable}

\color{black}
Some accounting routines have been implemented in open source libraries for accounting privacy of \RDP mechanisms. We use TensorFlow Privacy implementation \citep{mcmahan2019generalapproachaddingdifferential} in this work.

\subsection{Zero Concentrated Differential Privacy (z-\CDP)}
Another relaxed definition of differential privacy is zero concentrated differential privacy (z-\CDP) \citep{zcdp}. Being $\rho$-zCDP is equivalent to being $(\alpha, \rho \alpha)$-RDP simultaneously for all $\alpha > 1$. Therefore, standard RDP accountants, e.g. the aforementioned TensorFlow Privacy RDP accountant \citep{mcmahan2019generalapproachaddingdifferential}, can be use for accounting mechanism satisfying zCDP as well.

\subsection{Exponential Mechanism for Private Selection}\label{app:EM}
Exponential Mechanism is a standard for private selection from a set of candidates. The selection is based on a score, which is assigned to every candidate \citep{EM}. Let us assume there is a private dataset $\mathcal{D}$ and a score function $s: \mathcal{D} \times [M] \to \mathbb R$, which evaluates a set of $M$ candidates on the dataset $\mathcal{D}$. The goal is to select the candidate with the highest score, i.e., $\argmax_{m \in [M]} s(\mathcal{D}, m)$. Exponential mechanism performs this selection privately as follows. It sets the probability of choosing any candidate $m \in [M]$ as:

\begin{align}
    \texttt{Pr}[m] = \frac{\text{exp}(\frac{\epsilon_{\text{select}}}{2\Delta}\cdot s(\mathcal{D}, m))}{\sum_{m' \in [M]} \text{exp}(\frac{\epsilon_{\text{select}}}{2\Delta}\cdot s(\mathcal{D}, m'))},
\end{align}
where $\delta$ is the sensitivity of the scoring function $s$ to the replacement of a data sample in $\mathcal{D}$. It can be shown that the private selection performed by exponential mechanism satisfies $\frac{1}{8}\epsilon_{\text{select}}^2$-zCDP with respect to $\mathcal{D}$ \citep{zcdp}, which from the last paragraph, we know satisfies $(\alpha, \frac{\alpha}{8} \epsilon_{\text{select}}^2)$-\RDP for $\alpha>1$. We implement exponential mechanism by noisy selection with Gumbel noise: we add independent noises from Gumbel distribution with scale $\frac{2\Delta}{\epsilon_{\text{select}}}$ to candidate scores $s(\mathcal{D}, m)$, for $m \in [M]$, and select the candiate with the maximum noisy score. The larger the sensitivity $\Delta$ of score $s$ to replacement of a single sample in $\mathcal{D}$, the required larger noise scale. For further details about how we implement exponential mechanism for \algname{IFCA} and \algname{R-DPCFL}, see \cref{app:EMimplementation}.

\subsection{Privacy Budgeting}
In order to have a fair comparison between our algorithm and the baselines, we align them all to have the same ``total" privacy budget $\epsilon$ and satisfy $(\epsilon, \delta)$-\DP for a fixed $\delta$. In order to account the privacy of an algorithm, we compose the RDP guarantees of all private operations in the algorithm and then convert the resulting RDP guarantee to approximate $(\epsilon, \delta)$-DP using \cref{lemma:rdptodp}. 
The \DPSGD performed by different algorithms for local training benefits from privacy amplification by subsampling \citep{mironov2019renyidifferentialprivacysampled}. Algorithms that have privacy overheads, e.g., \algname{IFCA} and \algname{R-DPCFL} which need to privatize their local clustering as well, will have less privacy budget left for training. In other words, for the same total privacy budget $\epsilon$, \algname{IFCA} and \algname{R-DPCFL} will use a larger amount of noise when running \DPSGD, compared to \algname{MR-MTL} that has zero privacy overhead.

\section{Experimental setup}\label{app:exp_setup}

\subsection{Datasets}\label{appendix:datasets}
\paragraph{Data split:}
\label{sec:mnist_exp_setup}
 We use three datasets MNIST, FMNIST and CIFAR10, and consider a distributed setting with $21$ clients. In order to create majority and minority clusters, we consider 4 clusters with different number of clients $\{3,6,6,6\}$ (21 clients in total). The first cluster with the minimum number of clients is the ``minority" cluster, and the last three are the ``majority" ones. The data distribution $P(x,y)$ varies across clusters. We use two methods for making such data heterogeneity: 1. \textbf{covariate shift} 2. \textbf{concept shift}. In covariate shift, we assume that features marginal distribution $P(x)$ differs from one cluster to another cluster. In order to create this variation, we first allocate samples to all clients in an \emph{uniform} way. Then we rotate the data points (images) belonging to the clients in cluster $k$ by $k*90$ degrees. For concept shift, we assume that conditional distribution $P(y|x)$ differs from one cluster to another cluster, and we first allocate data samples to clients in a uniform way, and flip the labels of the points allocated to clients: we flip $y_{ij}$ (label of the $j$-th data point of client $i$, which belongs to cluster $k$) to $(y_{ij} + k)~ \textit{mod} ~10$, The local datasets are balanced--all users have the same amount of training samples.  The local data is split into train and test sets with ratios $80$\%, and $20$\%, respectively.  In the reported experimental results, all users participate in each communication round.

\begin{table}[th]
\footnotesize	
\centering
\caption{CNN model for classification on MNIST/FMNIST datasets \label{table:mnist_fmnist_model}}
\begin{tabular}{lcccc} \toprule
          Layer &  Output Shape &  $\#$ of Trainable Parameters & Activation & Hyper-parameters  \\\midrule
           Input & $(1, 28, 28)$ & $0$ &  &  \\
           Conv2d & $(16, 28, 28)$ & $416$ & ReLU & kernel size =$5$; strides=$(1, 1)$ \\
           MaxPool2d & $(16, 14, 14)$ & $0$ &  & pool size=$(2, 2)$ \\
           Conv2d & $(32, 14, 14)$ & $12,\!832$ & ReLU & kernel size =$5$; strides=$(1, 1)$ \\
           MaxPool2d & $(32, 7, 7)$ & $0$ &  & pool size=$(2, 2)$ \\
        %   Dropout2d & $(20, 4, 4)$ & $0$ &  & $p=0.5$ \\
           Flatten & $1568$ & $0$ & & \\
            Dense &  $10$ & $15,\!690$ & ReLU & \\
            \midrule
          Total & & $28,\!938$  & & \\ \bottomrule
\end{tabular}
% \vspace{5pt}
\end{table}

\subsection{Models and optimization}
\label{sec:cifar_exp_setup}
We use a simple 2-layer CNN model with ReLU activation, the detail of which can be found in \Cref{table:mnist_fmnist_model} for MNIST and FMNIST. Also, we use the residual neural network (ResNet-18) defined in \citep{resnet}, which is a large model. To update the local models allocated to each client during each round, we apply \DPSGD \citep{Abadi2016} with a noise scale $z$ which depends on some parameters, as in \cref{eq:sigma_i^2}.

\begin{table*}[ht]
\centering
\caption{Details of the used datasets in the main body of the paper. ResNet-18 is the residual neural networks defined in \cite{resnet}. CNN: Convolutional Neural Network defined in \Cref{table:mnist_fmnist_model}. }
\label{tab:datasets}
\small
\setlength\tabcolsep{2pt}
\begin{tabular}{ccccccc}
\toprule
\bf{Datasets} & \bf{Train set size} & \bf{Test set size} & \bf{Data Partition method} & \bf{\# of clients} & \bf{Model} & \bf{\# of parameters} % & VPF
\\ 
\midrule
MNIST & 48000 & 12000 & covariate shift & $\{3,6,6,6\}$ & CNN & 28,938\\

FMNIST & 48000 & 12000 & covariate shift & $\{3,6,6,6\}$ & CNN & 28,938
\\

CIFAR10 & 40000 & 10000 & covariate and concept shift & $\{3,6,6,6\}$ & ResNet-18  & 11,181,642
\\

\bottomrule
\end{tabular}
\label{table:split_uniform}
\end{table*}

In order to simulate a \FL setting, where clients (silos) have large local datasets and there is a structured data heterogeneity across clusters, we split the full dataset between the clients belonging to each cluster. This way, each client gets $8,000$ train and $1,666$ test samples for MNIST and FMNIST. Also, each client gets $10,000$ and $1,666$ train and test samples for CIFAR10 dataset (both covariate shift and concept shift).

\subsection{Baseline selection}\label{app:baseline_selection} When extending existing model personalization and clustered \FL algorithms to \DPFL settings, we are mostly interested in those with little to no additional local dataset queries to prevent extra noise for \DPSGD under a fixed total privacy budget $\epsilon$. For instance, the family of mean-regularized multi-task learning methods (\algname{MR-MTL}) \citep{mrmtl1, mrmtl2, mrmtl3, mrmtl4} provide model personalization \emph{without an additional privacy overhead}. Despite this, it is noteworthy that \algname{MR-MTL} relies on optimal hyperparameter tuning which leads to a potential privacy overhead \citep{liu2022csfl, liu2018privateselectionprivatecandidates, papernot2022hyperparametertuningrenyidifferential}. While resembling \algname{MR-MTL}, \algname{Ditto} \citep{ditto} has extra local computations, which makes it a less attractive personalization algorithm. Hence, we adopt \algname{MR-MTL} \citep{liu2022csfl} as a baseline personalization algorithm. Similarly, multi-task learning algorithms of \cite{Smith2017FederatedML} and \cite{ Marfoq2021FederatedML} as well as gradient-based clustered \FL algorithm of \cite{Sattler2019ClusteredFL} benefit from additional training and training restarts, which lead to high privacy overhead for them, making them less attractive. In contrast, the aforementioned loss-based clustered \FL algorithms \citep{mansour2020approaches, ghosh2020, fedsoft2022} can be managed to have a low privacy overhead (see \Cref{app:EMimplementation}), and we use it as a clustered \DPFL baseline.

\subsection{\algname{MR-MTL} formulation}\label{app:mrmtl}
The objective function of Mean-Regularzied Multi-Task Learning (\algname{MR-MTL}) can be expressed as:

\begin{align}
    \min_{\thetav_i, i \in \{1, \cdots, n\}} \sum_{i=1}^n g_i(\thetav_i) ~~~~~ \text{with} ~~~~~~  g_i(\theta_i) = f_i(\thetav_i) + \frac{\lambda}{2}\|\thetav_i - \bar \thetav\|_2^2,
\end{align}

where $\bar \thetav = \frac{1}{n} \sum_{i=1}^n \thetav_i$ is the average model parameter across clients and $f_i(\thetav_i)$ is the loss function of personalized model parameter $\theta_i$ of client $i$ on its local dataset $\mathcal{D}_i$. With $\lambda=0$, \algname{MR-MTL} reduces to local training. A larger regularization term $\lambda$ encourages local models to be closer to each other. However, \algname{MR-MTL} may not recover \algname{FedAvg} \citep{mcmahan2017communication} as $\lambda \rightarrow \infty$. See section E.2 and algorithm A1 in \citep{liu2022csfl} for more details about \algname{MR-MTL}.

\subsection{Tuning hyperparameters of baseline algorithms}

\Cref{app:baseline_selection} explains our criteria for baseline selection. We compare our \algname{R-DPCFL} algorithm, which benefits from robust clustering, with four baseline algorithms, including: 1) \algname{DPFedAvg} \citep{DPSCAFFOLD2022}, which learns one global model for all clients, and is called \algname{Global} in the paper 2) \algname{Local}, in which clients do not participate \FL and run \DPSGD locally to train a model solely on their local dataset 3) \algname{MR-MTL} personalized \FL algorithm \citep{liu2022csfl}, which learns a global model and one personalized model for each client 4) A \DP extension of the clustered \FL algorithm \algname{IFCA} \citep{ghosh2020} to \DPFL systems enhanced with exponential mechanism (see \cref{app:EMimplementation}) 5) An oracle algorithm, which has the knowledge of the true underlying clients' clusters, which we call \algname{O-DPCFL}.

For all algorithms and all datasets, we set total number of rounds $E$ to $200$ and per-round number of local epochs $K$ to $1$. Following \citep{Abadi2016}, we set the batch size of each client such that the number of batches per epoch is in the same order as the total number of epochs: $N_i/b_i^e = E\cdot K = 200$. For MNIST and FMNIST, this leads to batch sizes $b_i^e=32$ for all clients $i$ and every round $e$ for the baseline algorithms. For CIFAR10 (covariate shift and concept shift), this leads to batch size $b_i^e=64$ for all clients $i$ and every round $e$ for the baseline algorithms. While \algname{R-DPCFL} uses full batch sizes in the first round (i.e., $b_i^1=N_i$ for all $i$), it needs to use small batch sizes in the next rounds. We have further explained about this in \Cref{app:RDPCFL_hparams}.

Having determined the batch size for all algorithms, clipping threshold $c$ and learning rate $\eta_l$ are determined via a grid search on clients' validation sets. For each algorithm and each dataset, we find the best learning rate from a grid: the one which results in the highest average accuracy at the end of \FL training on a ``validation set" with size $1666$ samples for each client. We use the grid $\eta_l \in $\texttt{\{5e-4, 1e-3, 2e-3, 5e-3, 1e-2, 2e-2, 5e-2, 1e-1\}} for all datasets and all algorithms. Similarly, we use the grid $c \in$ \texttt{\{1, 2, 3, 4, 5\}} for setting the clipping threshold for all datasets and all algorithms based on the clients' validation sets.

\subsection{Implementation of private local clustering for \algname{IFCA} and \algname{R-DPCFL}}\label{app:EMimplementation}
In every round of \algname{IFCA} and during the rounds $e>E_c$ of \algname{R-DPCFL}, the server sends $M$ cluster models to all clients, and they evaluate them on their local datasets. Then, each client $i$ selects the model with the lowest loss on its local dataset $\mathcal{D}_i$, trains it for $K$ local epochs and sends the result back to the server. This model selection performed by each client can lead to privacy leakage w.r.t its local dataset, if it is not privatized. In order to protect data privacy, clients need to privatize their local clustering by using exponential mechanism and accounting its privacy using z-\CDP, explained in \cref{app:EM}. Assuming a total privacy budget $\epsilon$ for a client $i$, it has to split the budget between private clustering and \DPSGD. Naive split of privacy budget can lead to very noisy \DPSGD steps or very noisy local cluster selection by clients. Following \cite{liu2022csfl}, we use two strategies to mitigate the privacy overhead of local clustering performed by \algname{IFCA} and \algname{R-DPCFL}:

\begin{itemize}
    \item \textbf{Clients use model accuracy, instead of loss, as the score function for model selection:} clients use model accuracy as score function $s(\mathcal{D}_i, m)$ evaluating cluster model $m$ on client $i$'s dataset. The reason is that while, loss function has practically an unbounded sensitivity to individual samples in the clients' datasets, model accuracy is a low-sensitivity function, espcially in cross-silo \FL settings with large local datasets. More specifically, let us assume client $i$ with local dataset $\mathcal{D}_i$ (which has size $N_i$) uses models' accuracy on $\mathcal{D}_i$ for model selection. It can be shown that under all add/remove/replace notions of dataset neighborhood, sensitivity of model accuracy (as score function) is bounded as follows \citep{liu2022csfl}:

    \begin{align}
        \Delta_{\textit{acc}}  = \max_{m \in [M]} \max_{\mathcal{D}_i, \mathcal{D}'_i} |s(\mathcal{D}_i, m) - s(\mathcal{D}_i, m)|\leq \frac{1}{N_i - 1}\cdot
    \end{align}
Since local dataset sizes are usually large, especially in cross-silo \FL, the sensitivity of model accuracy is much smaller than that of model loss. Therefore, following \cite{liu2022csfl}, we set the per-round privacy budget of private model selection to a very smalle value $\epsilon_{\text{select}} = 0.03 \cdot\epsilon$ (3\% of the total privacy budget). Yet, the cost of private selection by clients can grow quickly if clients naively run local clustering for ``many" rounds. Therefore, we use the following strategy as well. It is noteworthy that in our experiments, we observed that \algname{IFCA} baseline algorithm performs better when clients use model train accuracy (instead of train loss) for cluster selection.

    \item \textbf{Reduce the number of rounds with local clustering on clients' side:} Clients run local clustering for less rounds. Following \cite{liu2022csfl}, we let clients run local clustering for the $10\%$ of the total number of rounds $E$. For example, \algname{IFCA} runs local clustering during the first $\lfloor \frac{E}{10}\rfloor$ rounds, and fixes clients' cluster assignments afterwards. Similarly, \algname{R-DPCFL} lets clients run local clustering during rounds $E_c \leq e \leq E_c + \lfloor \frac{E}{10}\rfloor$, and fixes clients' cluster assignments afterwards. 
\end{itemize}

The privacy overhead of private model selection can still grow and leave a low privacy budget for training with \DPSGD. Choosing a small selection budget $\epsilon_{\text{select}} $ leaves most of the total privacy budget $\epsilon$ for training with \DPSGD, but leads to noisy and inaccurate cluster selection by clients. Similarly, a large $\epsilon_{\text{select}} $ leads to more noisy gradient steps by \DPSGD.

\subsection{\DP privacy parameters}
For each dataset, 5 different values of $\epsilon$ (the total privacy budget) from set $\{3,4,5,10,15\}$ are used.  We fix $\delta$ for all experiments to $10^{-4}$, which satisfies $\delta<N_i^{-1}$ for every client $i$. We use the Renyi DP (\RDP) privacy accountant (TensorFlow privacy implementation \citep{mcmahan2019generalapproachaddingdifferential}) during the training time. This accountant is able to handle the difference in the batch size of \algname{R-DPCFL} between the first round $e=1$ and the next rounds $e>1$ by accounting the composition of the corresponding \emph{heterogeneous} private mechanisms.

\subsection{Gaussian Mixture Model}
We use the Gaussian Mixture Model of Scikitlearn, which can be found here: \url{https://scikit-learn.org/dev/modules/generated/sklearn.mixture.GaussianMixture.html}. The \GMM model has three hyper-parameters: 

1) parameter initialization, which we set to ``\texttt{k-means++}". This is because this type of initialization leads to both low time to initialize and low number of \EM iterations for the \GMM to converge \citep{Arthur2007kmeansTA, Biernacki2003ChoosingSV}.

2) Type of the covariance matrix, which we set to ``\texttt{spherical}", i.e., each component has a diagonal covariance matrix with a single value as its diagonal elements. This is in accordance with \cref{eq:diagonal_cov} and that we know the covariance matrices should be diagonal. 

3) Finally, the number of components (clusters) is either known or it is unknown. In the latter case, we have explained in \cref{sec:M} how we can find the true number of clusters by using the confidence level (\texttt{MSS}) of the \GMM model.

\section{Setting hyper-parameters of \algname{R-DPCFL}}\label{app:RDPCFL_hparams}

As explained in the paper, \algname{R-DPCFL} has three hyperparameters, which we explain how to set in the following: 

\begin{figure*}
%\vspace{-3em}
\centering
\includegraphics[width=0.47\columnwidth, height=5cm]{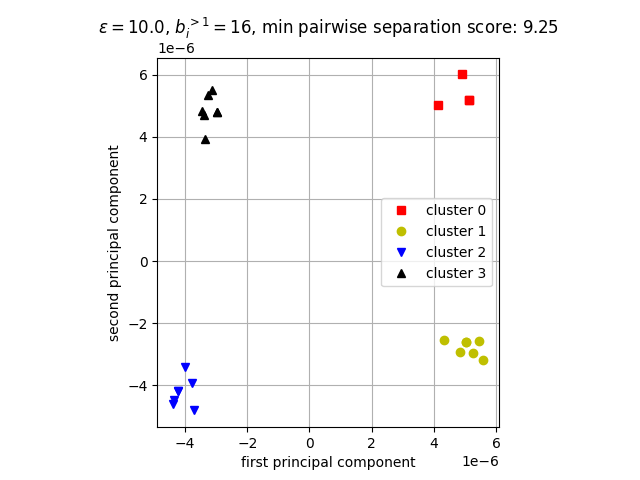}
\includegraphics[width=0.47\columnwidth, height=5cm]{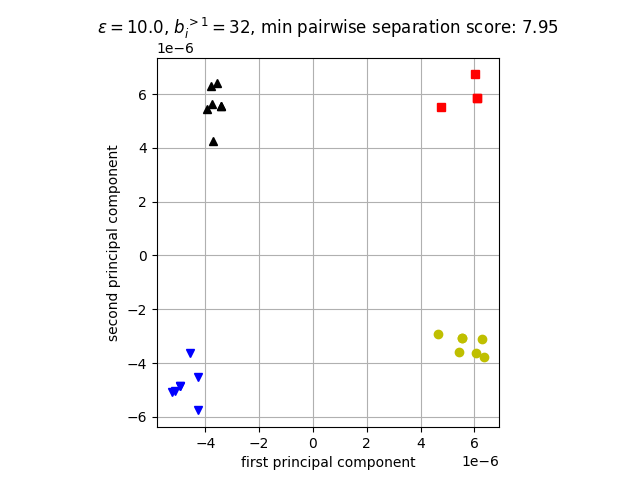}

\includegraphics[width=0.47\columnwidth, height=5cm]{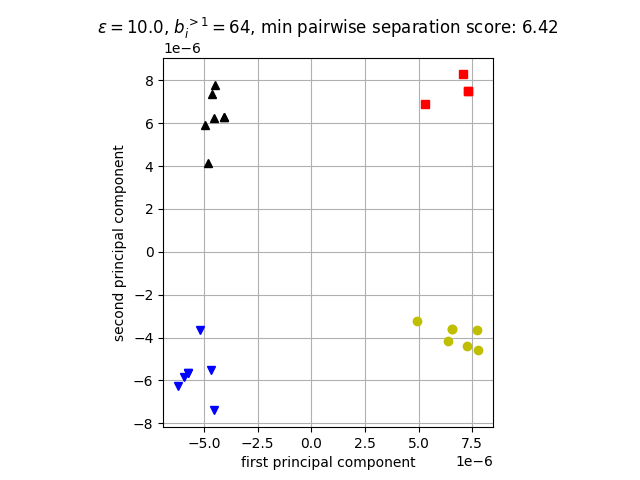}
\includegraphics[width=0.47\columnwidth, height=5cm]{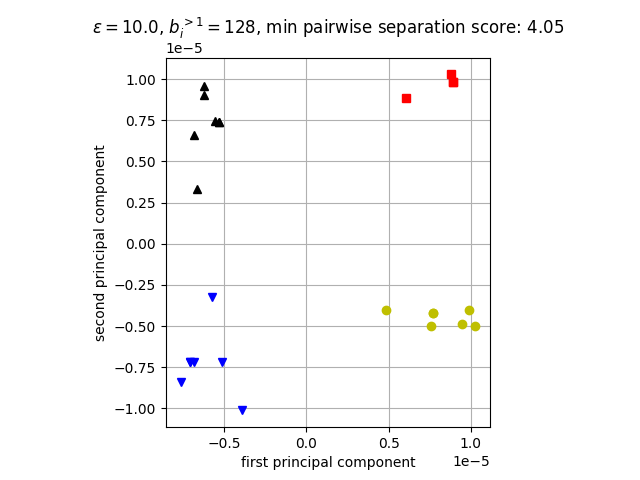}

\includegraphics[width=0.47\columnwidth, height=5cm]{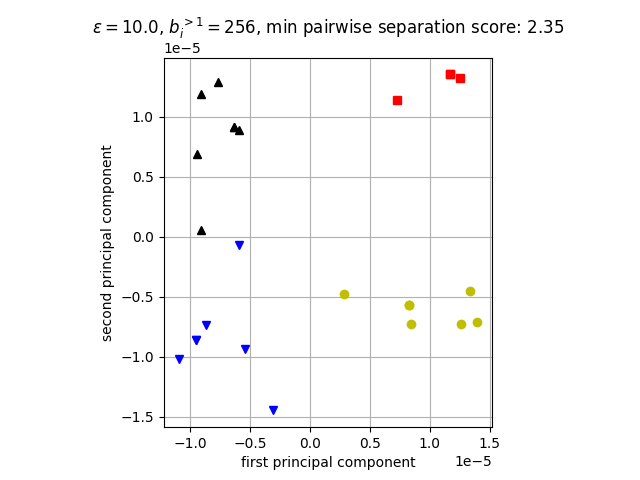}
\includegraphics[width=0.47\columnwidth, height=5cm]{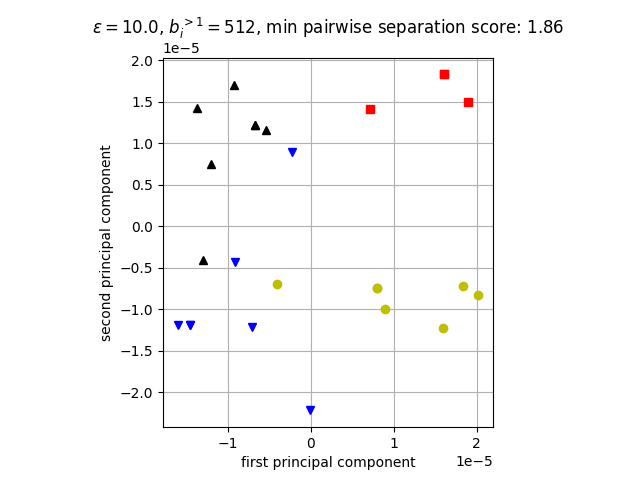}

\caption{The effect of increasing the batch size after the first round, i.e., $b_i^{>1},$ on the model updates $\{\Delta \Tilde{\mathbf{\thetav}}_i^1\}_{i=1}^n$ at the end of the first round. All clients have used full batch sizes in the first round, i.e., $\forall i: b_i^1=N_i$. The level of noise in $\{\Delta \Tilde{\mathbf{\thetav}}_i^1\}_{i=1}^n$ affects the quality and confidence of the client clustering that the server performs at the end of the first round. As can be observed, for a fixed $\epsilon=10$, the model updates scatter further in space as $b_i^{>1}$ increases and different clusters get less separated. This leads to a decrement in the confidence level or the (\texttt{MSS}) score of the resulting \GMM, as $b_i^{>1}$ increases (mentioned on the top of each plot). All the results are obtained on CIFAR10 with covariate shift (rotation) across clusters.} 
\label{fig:b2andnoise}
\end{figure*}

\subsection{Batch size $b_i^{>1}$}\label{sec:b2}
Batch size $b_i^{>1}$, which is the batch size used by \algname{R-DPCFL} during the rounds $e>1$, has to be set to a small value, as observed in \cref{fig:var1var2} right. \algname{R-DPCFL} is not sensitive to this parameter, as long as a small value is chosen for it. For the results in the paper, we use $b_i^{>1}=32$ for all experiments with \algname{R-DPCFL}. We further explain about the effect of this parameter as follows:

As we observed in \cref{lemma:updatesnoise} and \cref{fig:var1var2} left, $\texttt{Var}(\Delta \Tilde{\mathbf{\thetav}}_i^1(b_i^1)|\mathbf{\thetav}^{init})$ is an increasing function of $b_i^{>1}$. More generally, the effect of increasing $b_i^{>1}$ is on three things: 1) increasing noise variance $\texttt{Var}(\Delta \Tilde{\mathbf{\thetav}}_i^1(b_i^1)|\mathbf{\thetav}^{init})$ (as shown in \cref{fig:var1var2}, left) 2) decreasing noise variance $\texttt{Var}(\Delta \Tilde{\mathbf{\thetav}}_i^1(b_i^e)|\mathbf{\thetav}_i^{e,0})$ (as shown in \cref{fig:var1var2}, right) 3) decreasing number of gradient steps during each round $e$ for $e>1$. While the first one is only limited to the first round $e=1$, the last two affect the remaining $E-1$ rounds and have conflicting effects on the final accuracy. However, an important point about the problem of \DP clustered \FL is that finding the true structure of clusters in the first round is a prerequisite for making progress in the next rounds. Therefore, increment in noise variance $\texttt{Var}(\Delta \Tilde{\mathbf{\thetav}}_i^1(b_i^1)|\mathbf{\thetav}^{init})$ (the first effect) is the most important one. We have demonstrated this effect in \cref{fig:b2andnoise}, which shows that how increasing $b_i^{>1}$ adversely affects the clustering done at the end of the first round. Note how \texttt{MSS} score of the learned \GMM increases as $b_i^{>1}$ increases. Therefore, in order to have a reliable client clustering at the end of the first round, we need to keep the value of $b_i^{>1}$ to small values: the smaller total privacy budget $\epsilon$, the smaller value should be used for $b_i^{>1}$. Following this observation, we have fixed $b_i^{>1}$ to 32 in all our experiments with \algname{R-DPCFL}.

\subsection{The strategy switching time $E_c$}

The strategy switching time  $E_c$ can also be set by using the uncertainty metric $\texttt{MPO} \in [0,1)$. Intuitively, if the learned \GMM is not certain about its clustering decisions, \algname{R-DPCFL} should not rely on its decisions for a large $E_c$, and vice versa. Hence, we can set $E_c$ as a decreasing function of \texttt{MPO}. For instance, $E_c = (1-\texttt{MPO})\frac{E}{2}$, which is used in this work, means that if a \GMM is completely confident about its clusterings, e.g., what happens in \cref{fig:client_updates_differentbatch} right, the server changes the clustering strategy to loss-based after the first half of rounds. As the uncertainty increases, this change happens earlier (e.g., when $\epsilon$ is small), and  \algname{R-DPCFL} slowly gets close to the existing loss-based clustering methods like \algname{IFCA} \citep{ghosh2020}.

\subsection{The number of clusters $M$}\label{sec:M}

Knowing the number of clusters is broadly accepted and applied in the clustered FL literature \citep{ghosh2020,fedsoft2022,briggs2020federated}. \textbf{This is the assumption of our baseline algorithms too}. Yet, techniques to determine the number of clusters can enable our approach to be more widely adopted. In this section, we show that how we can find the true number of clusters ($M$) when it is not given. Our method relies on the \texttt{MSS} score (confidence level) defined in \cref{sec:applicability}: $\texttt{MSS} = \min_{m,m'} \hat{\texttt{SS}}(m,m') \in [0, +\infty)$ (please see the detailed explanations in \cref{sec:applicability}). Consider the \cref{fig:client_updates_differentbatch} right as an example. There is a good separation between the $M=4$ existing clusters, thanks to clients using full batch sizes in the first round. Fitting a \GMM with 4 components to the model updates results in the highest \texttt{MSS} for the learned \GMM model: remember that \texttt{MSS} was the maximum pairwise separation score between the different components of the learned \GMM. In contrast, if we fit a \GMM with $3$ components (less than the true number of components) to the same model updates in the figure, then two clusters will be merged into one component (for examples clusters 0 and 1) leading to a high radius for one of the three components of the resulting \GMM. This leads to a low \texttt{MSS} (confidence level) for the resulting \GMM. Similarly, if we fit a \GMM with 5 components, one of the four clusters (for example cluster 1) will be split between two of the 5 components (call them $m$ and $m'$), which leads to a low inter-component distance ($\Delta_{m,m'}$) for the pair of components. This also leads to a low \texttt{MSS} for the resulting \GMM. However, \textbf{fitting a \GMM with $M=4$ components leads to a well separation between all the true components and maximizes the resulting \texttt{MSS}}. Based on this very intuitive observation, we propose the following method for setting $m$ at the end of the first round: We select the number of clusters/components, which leads to the maximum \texttt{MSS} for the resulting \GMM. More specifically:

\begin{align}
    M = \argmax_{m \in S} \texttt{MSS}\Big(\GMM(\Delta \Tilde{\mathbf{\thetav}}_1^1, \ldots, \Delta \Tilde{\mathbf{\thetav}}_n^1; m)\Big), (\text{line 10 of \cref{alg:R-DPCFL}})
\end{align}
where $S$ is a set of candidate values for $M$: at the end of the first round and on the server, we learn one \GMM for each candidate value in $S$ on the same received model updates $\{\Delta \Tilde{\thetav}_i^1\}_{i=1}^n$. Finally, we choose the value resulting in the \GMM with the highest \texttt{MSS} (confidence). Therefore, this method is run on the server and does not incur any additional privacy overheads. It is also noteworthy that we know from \cref{lemma:localdp} that learning the \GMM does not incur much computational cost when large enough (and small enough) batch sizes are used in the first round (subsequent rounds).

\begin{figure*}[t]
\centering
\includegraphics[width=0.47\columnwidth, height=5.2cm]{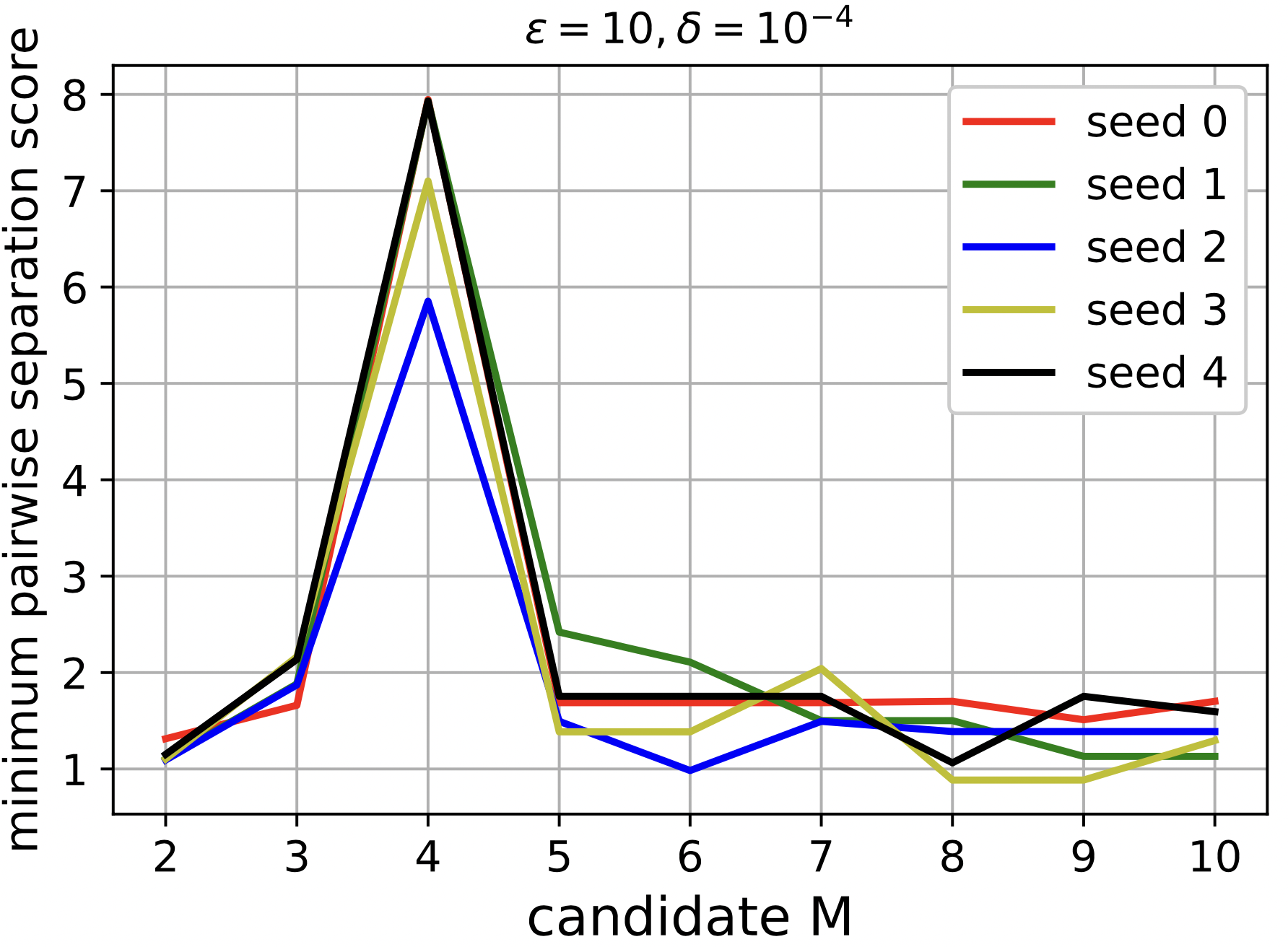}
\includegraphics[width=0.47\columnwidth, height=5.2cm]{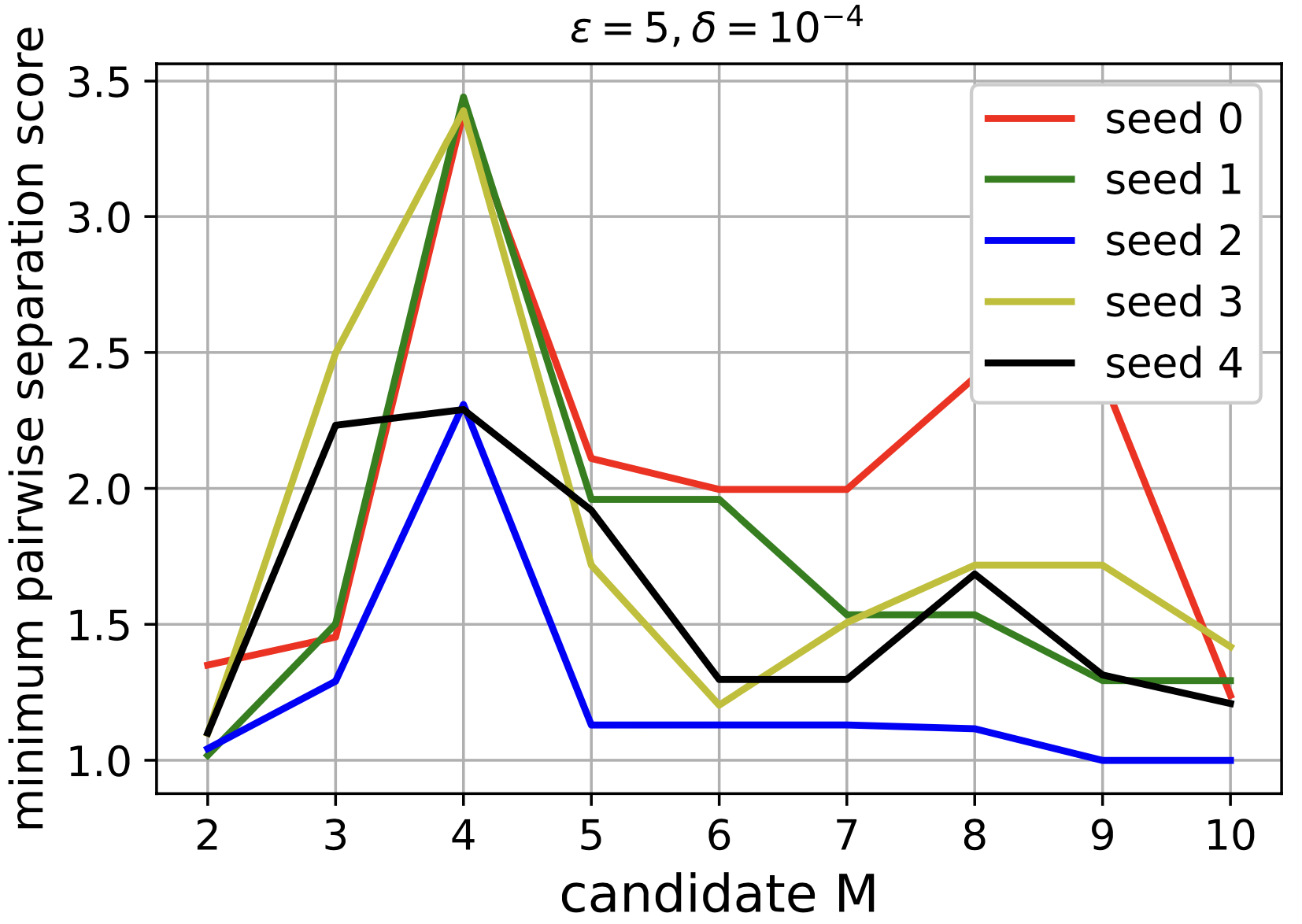}

\includegraphics[width=0.47\columnwidth, height=5.2cm]{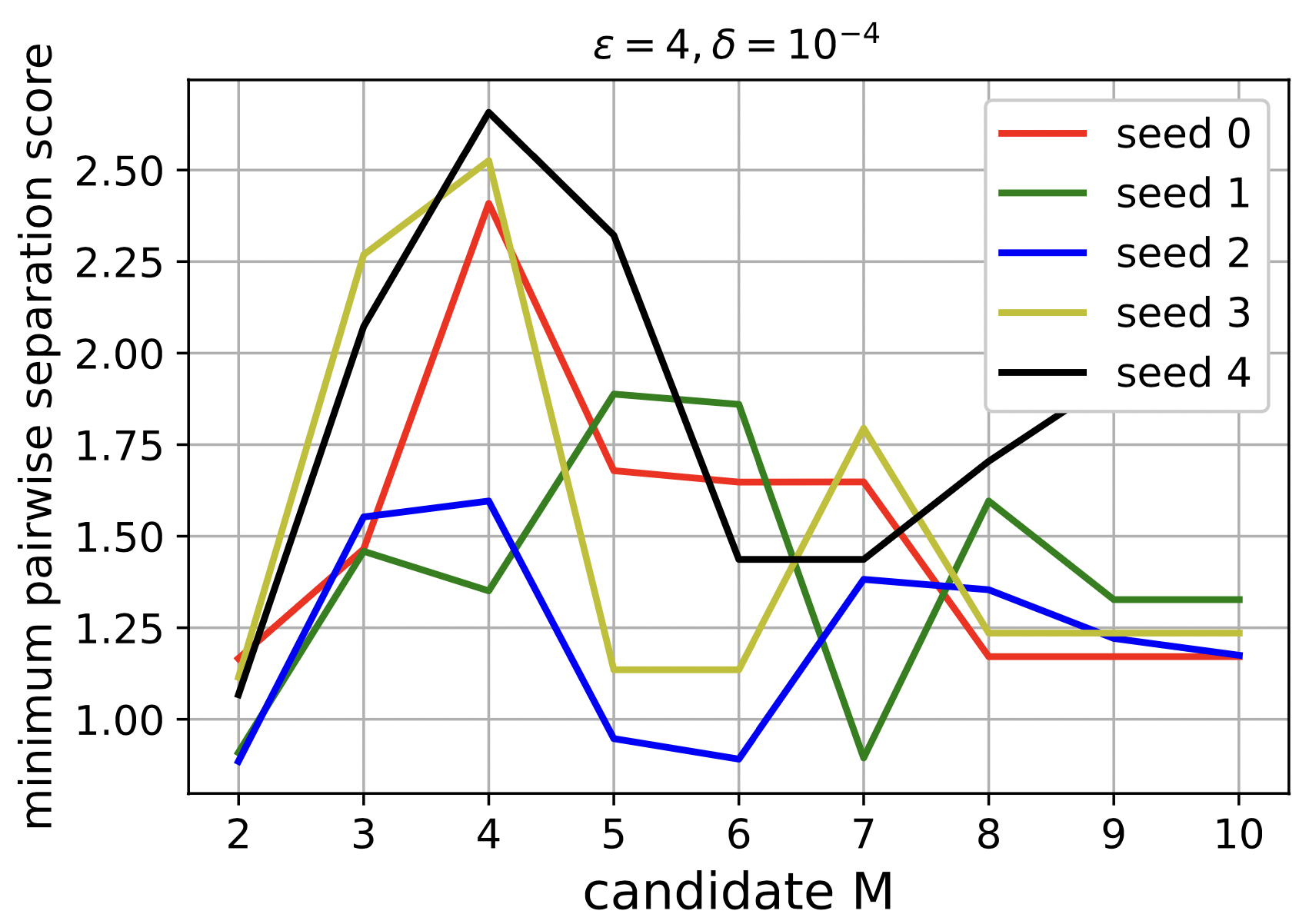}
\includegraphics[width=0.47\columnwidth, height=5.2cm]{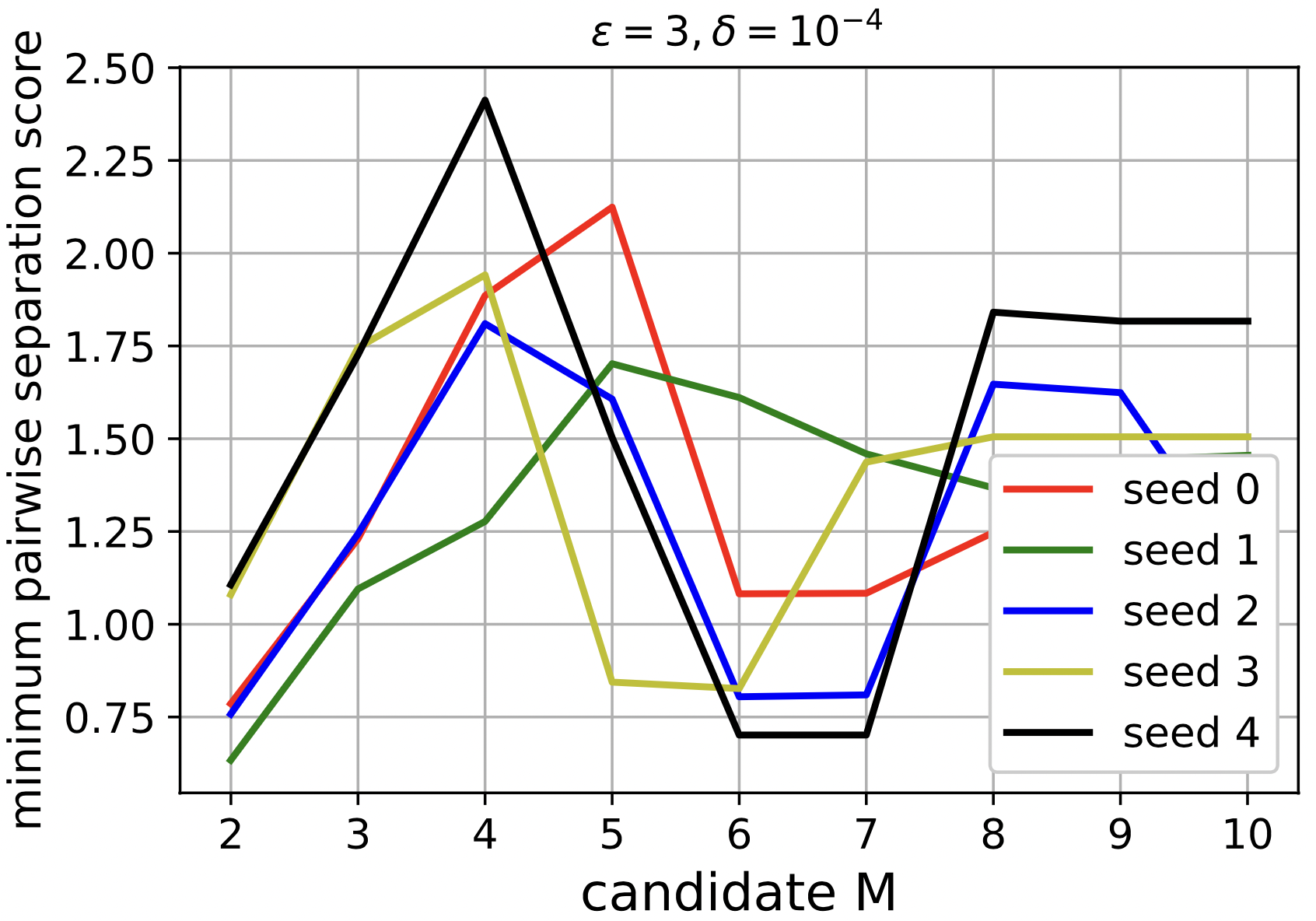}

\caption{The minimum pairwise separation score (\texttt{MSS}) or confidence of the \texttt{GMM} learned on $\{\Delta \Tilde{\mathbf{\thetav}}_i^1\}_{i=1}^n$ peaks at the true cluster number, which is equal to 4 in all the plots above. Each figure is for a different value of $\epsilon$ (mentioned on top of each figure), and are obtained on CIFAR10 with covariate shift (rotation) across clusters, and 5 different random data splits (5 seeds). All the results are obtained with full batch sizes in the first round and $b_i^{>1}=32$ for all $i$. We can use this observation as a method to find the true number of clusters ($M$) when it is not given. For larger $\epsilon$, this method work perfectly and even when $\epsilon$ is too small, e.g., $\epsilon=3$, this method works well and predicts the true number of clusters correctly most of the times: 3 out of the 5 curves in the bottom right plot have a peak at $M=4$ (the true cluster number). and the other 2 curves predict $5$ as the true number, which is the closest and the best alternative for the true value $M=4$.}
\label{fig:findingM}
\end{figure*}

We have evaluated this method on multiple data splits and different privacy budgets ($\epsilon$) on CIFAR10, MNIST and FMNIST. The method could predict the number of underlying clusters with 100\% accuracy for the MNIST and  FMNIST datasets for all values of $\epsilon$. Results for CIFAR10 are shown in \cref{fig:findingM}. As can be observed, the method has made only one mistake for $\epsilon=4$ (seed 1) and two mistakes for $\epsilon=3$ (seeds 0 and 1), out of 20 total experiments. Even in those three cases, it has predicted $M$ as 5, which is closest to the true value ($M=4$) and does not lead to much performance drop (because having $M=5$ splits an existing cluster into two and it is better than predicting for example $M=3$, which results in "mixing" two clusters with heterogeneous data). Even in this cases, we can improve the prediction accuracy further by using smaller values of $b_i^{>1}$ (simultaneously with full batch sizes $b_i^1=N_i$), e.g., $b_i^{>1}=16$ or $b_i^{>1}=8$, instead of $b_i^{>1}=32$ in the figure above. This improvement  happens as reducing $b_i^{>1}$ constantly enhances the separation between the underlying components (See \Cref{fig:b2andnoise}), which leads to higher accuracy in prediction of the true $M$.

Finally, note that none of the existing baseline algorithms has such an easy and applicable strategy for finding $M$. This shows another useful feature of the proposed \algname{R-DPCFL}, which makes it more applicable to \DP clustered \FL settings.

\section{More experimental results}\label{app:full_results}

The results shown in \Cref{fig:avg_test_acc_allalgs} and \cref{fig:avg_test_acc_minority_allalgs} include the results for the \algname{Global} baseline and are the more complete versions of the figures in the paper (\Cref{fig:avg_test_acc} and \cref{fig:avg_test_acc_minority}).

\begin{figure*}[h]
\centering
\includegraphics[width=0.24\columnwidth,height=3.5cm]{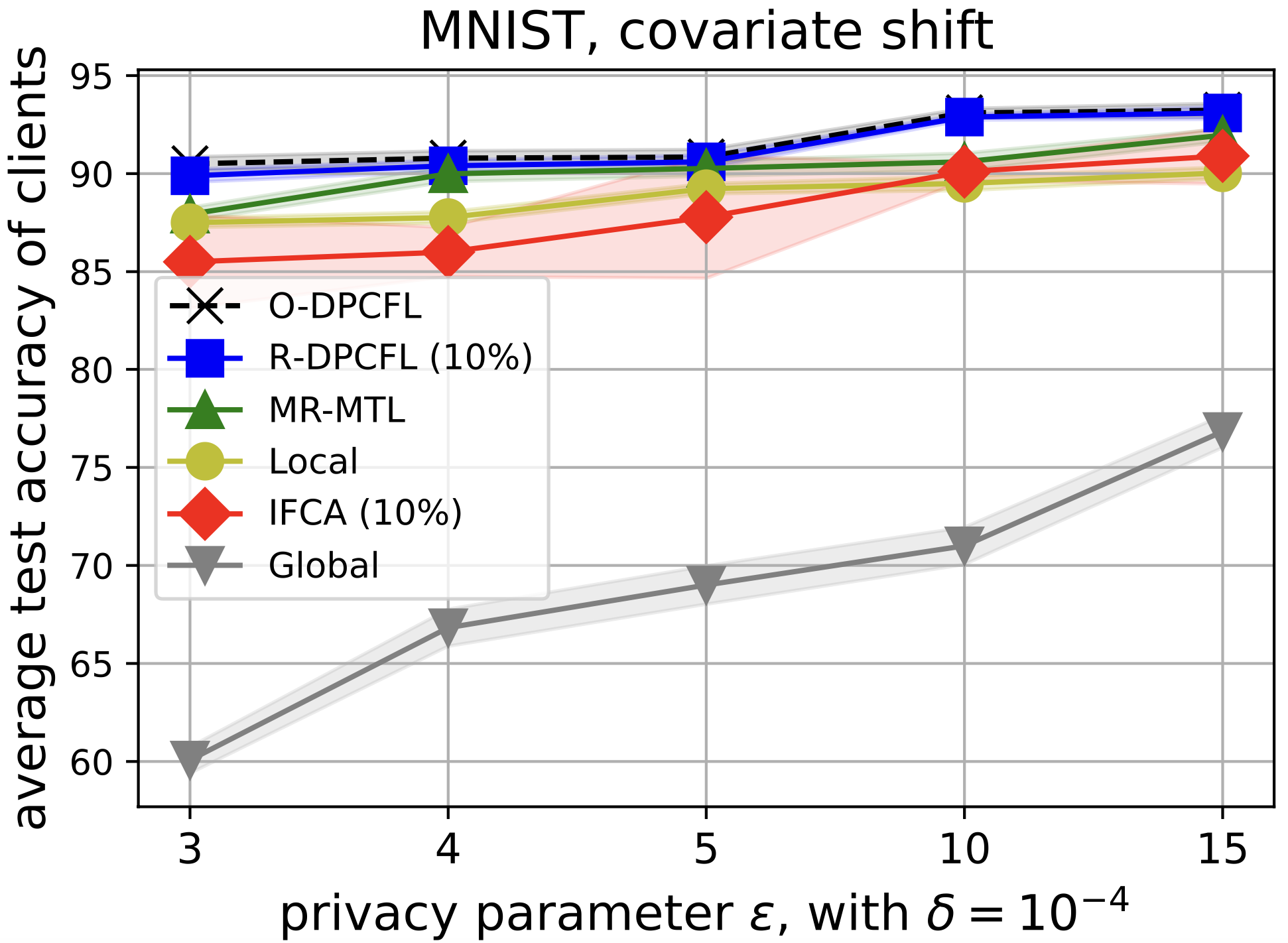}
\includegraphics[width=0.24\columnwidth,height=3.5cm]{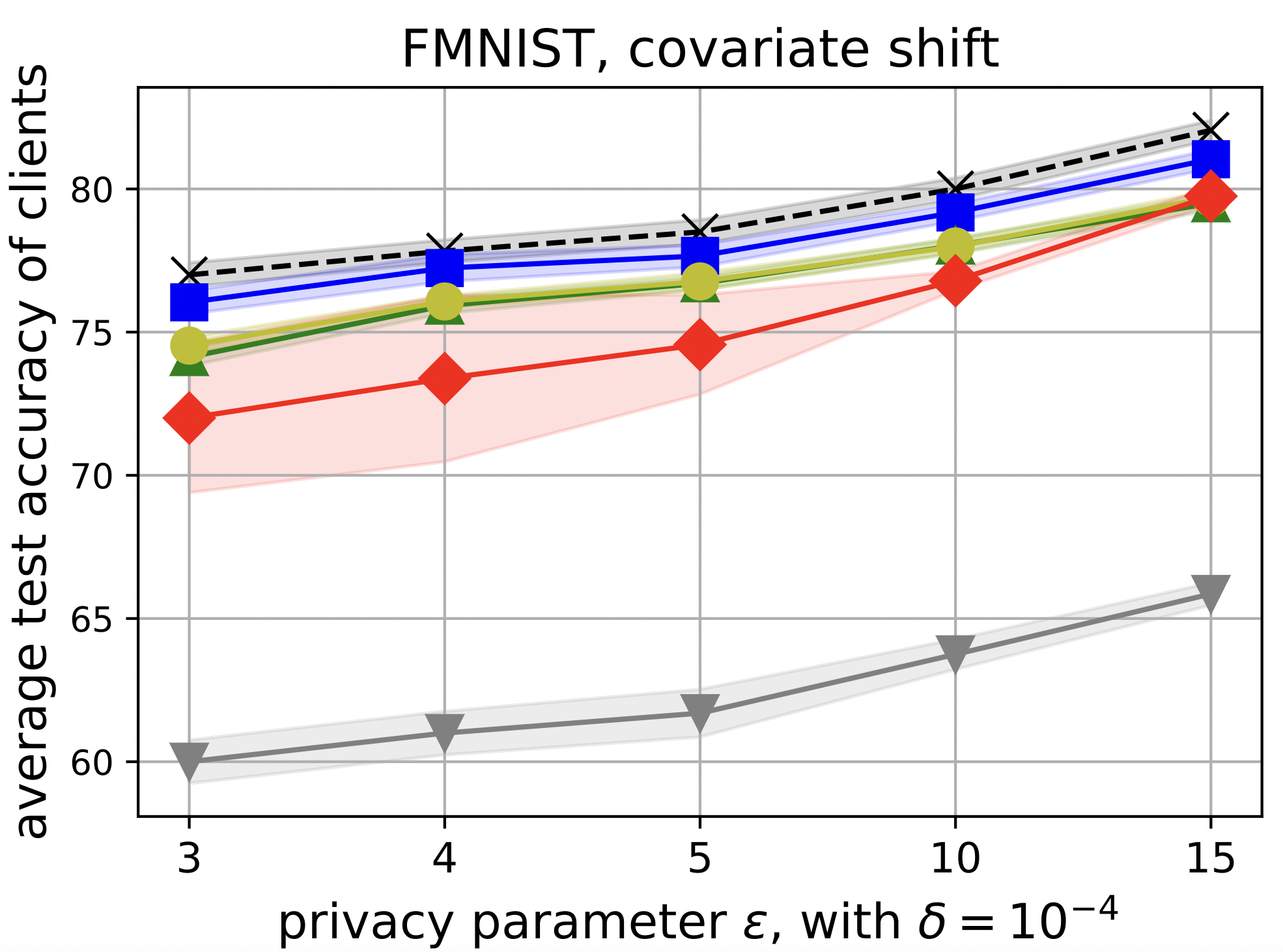}
\includegraphics[width=0.24\columnwidth,height=3.5cm]{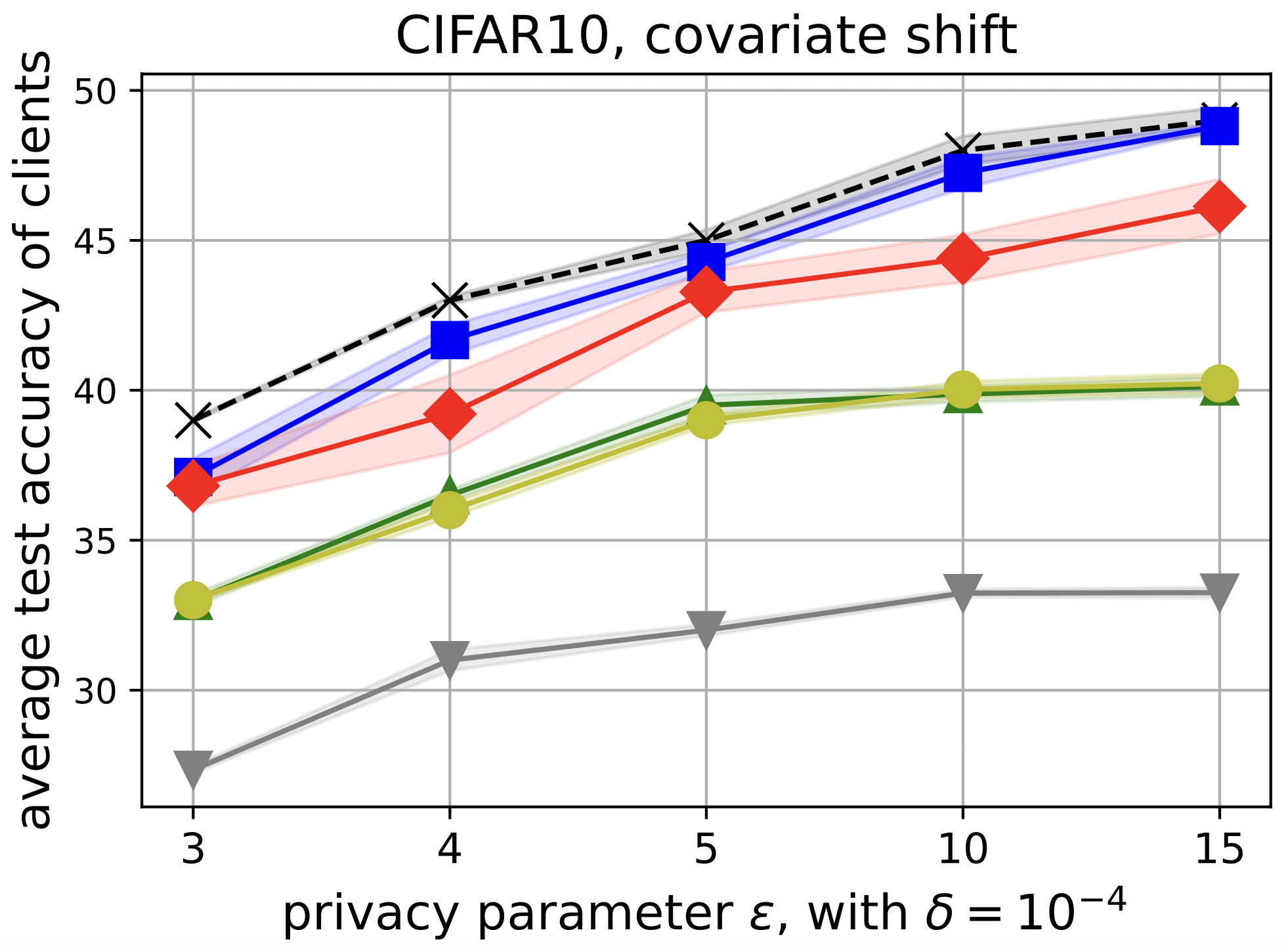}
\includegraphics[width=0.24\columnwidth,height=3.5cm]{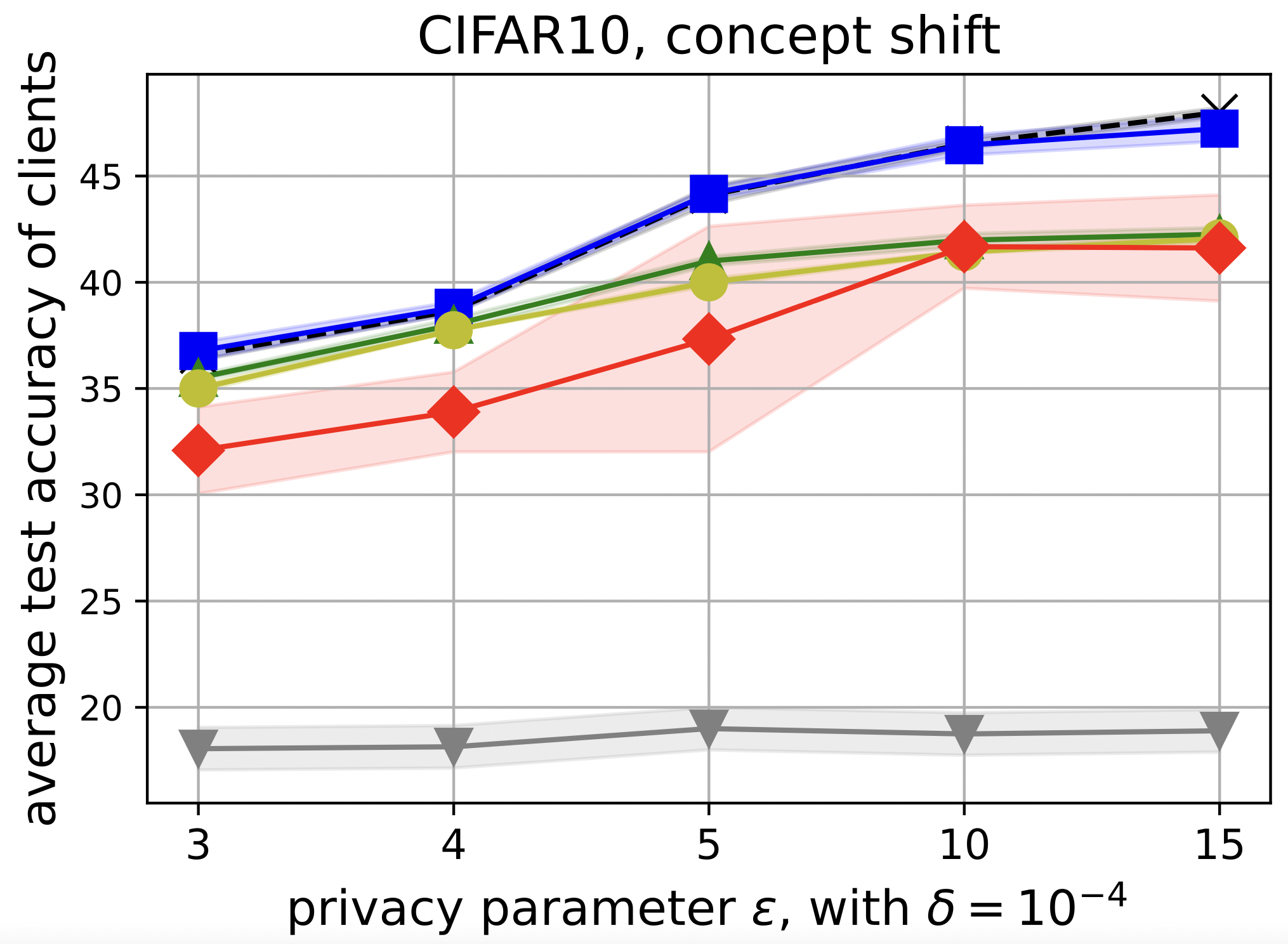}
\caption{Average test accuracy across clients for different total privacy budgets $\epsilon$ (results are obtained from 4 different random seeds). $10\%$ means performing loss-based clustering by clients only in $10\%$ of the total rounds $(E)$. }
\label{fig:avg_test_acc_allalgs}
\end{figure*}

\begin{figure*}[h]
\centering
\includegraphics[width=0.24\columnwidth,height=3.5cm]{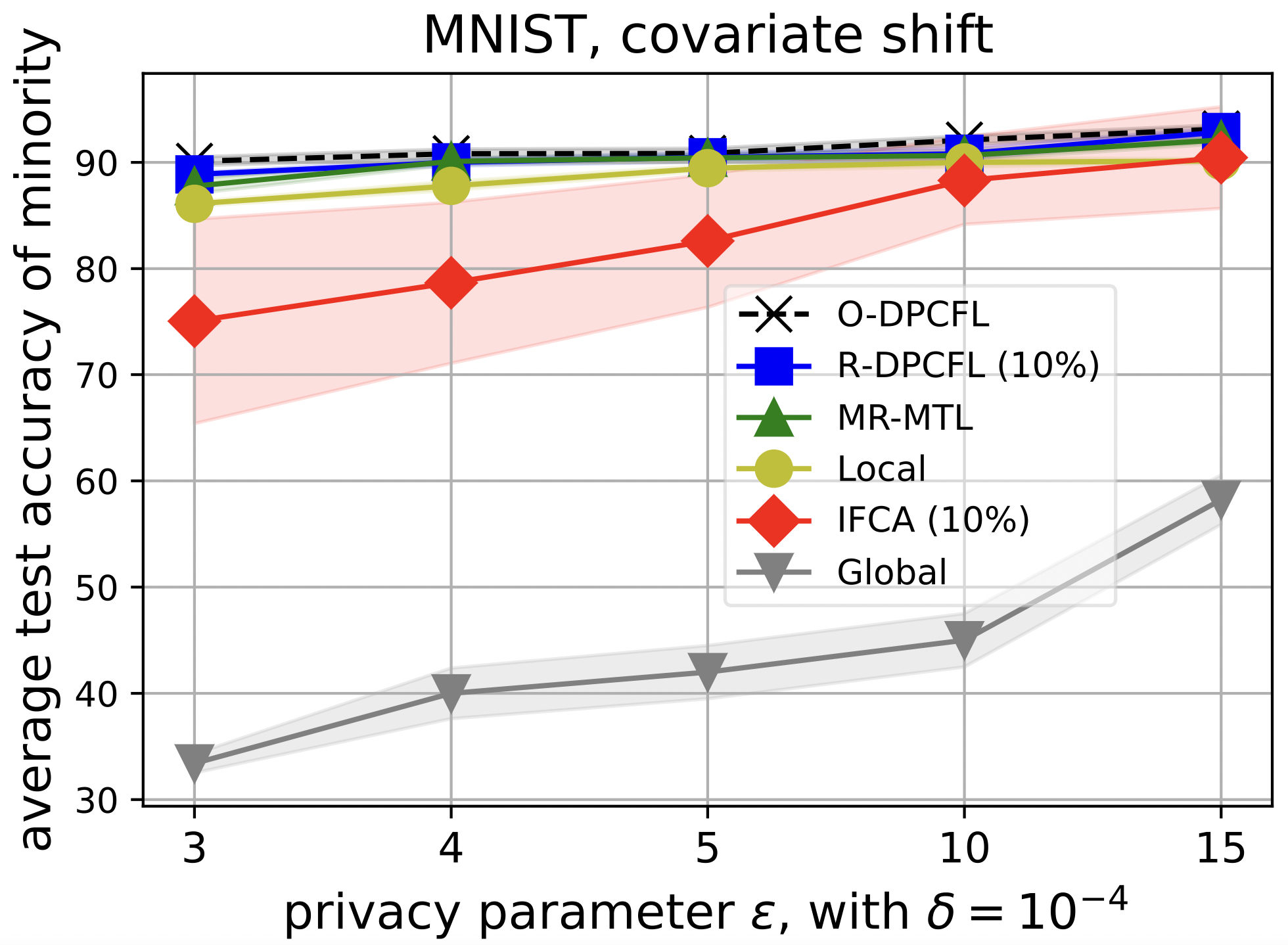}
\includegraphics[width=0.24\columnwidth,height=3.5cm]{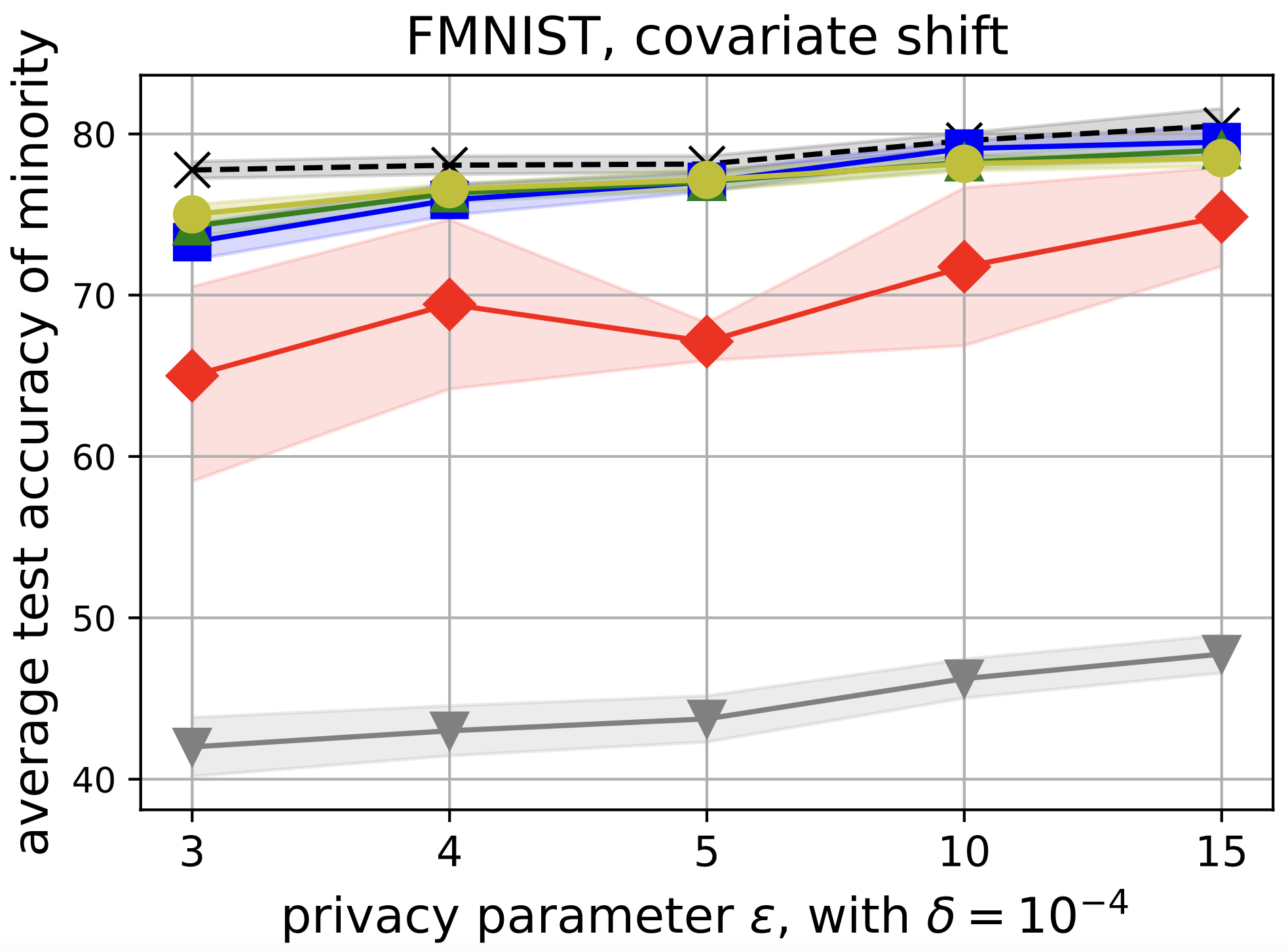}
\includegraphics[width=0.24\columnwidth,height=3.5cm]{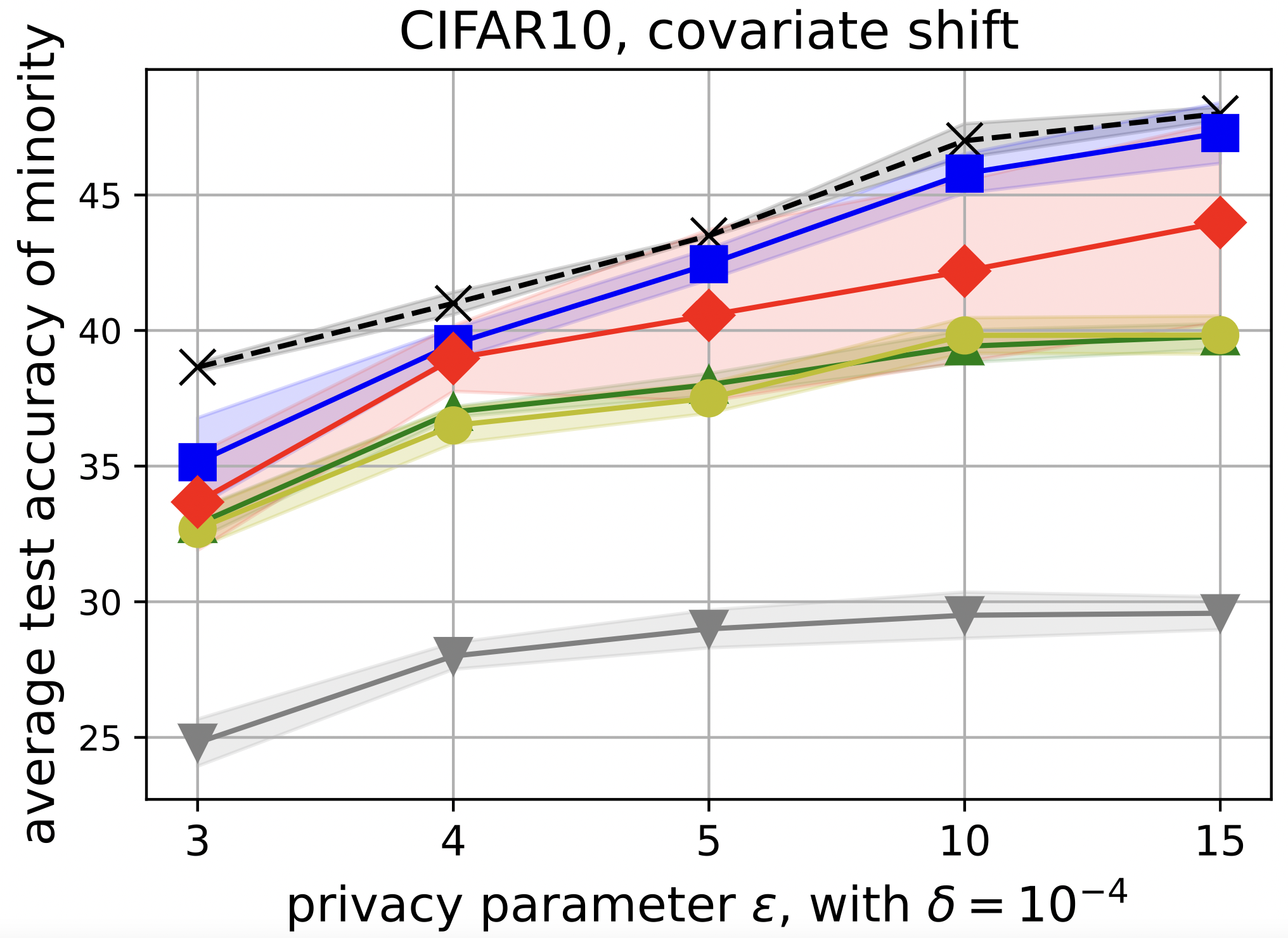}
\includegraphics[width=0.24\columnwidth,height=3.5cm]{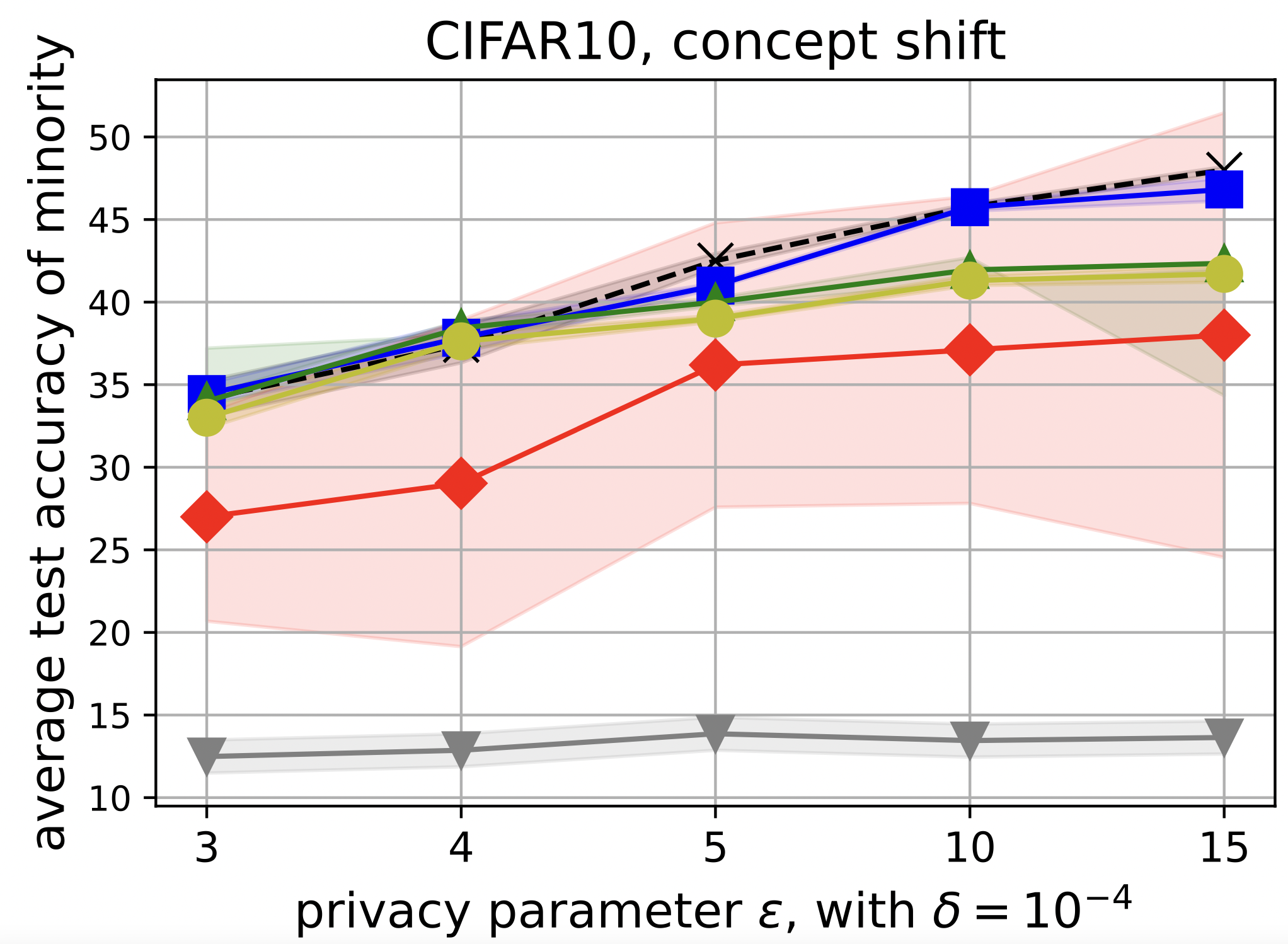}
\caption{Average test accuracy across clients belonging to the minority cluster for different total privacy budgets $\epsilon$ (results are obtained from 4 different random seeds). $10\%$ means performing loss-based clustering by clients only in $10\%$ of the total rounds $(E)$.}
\label{fig:avg_test_acc_minority_allalgs}
\end{figure*}

\Cref{fig:effect_of_MSS1} shows how the \texttt{MSS} score of the learned \GMM at the first round can be indicative of whether the true clients' clusters will be detected correctly or not. An \texttt{MSS} score above 2 almost always yields to correct detection of all clusters. 

\begin{figure}[h]
\centering
\includegraphics[width=0.31\columnwidth, height=4cm]
{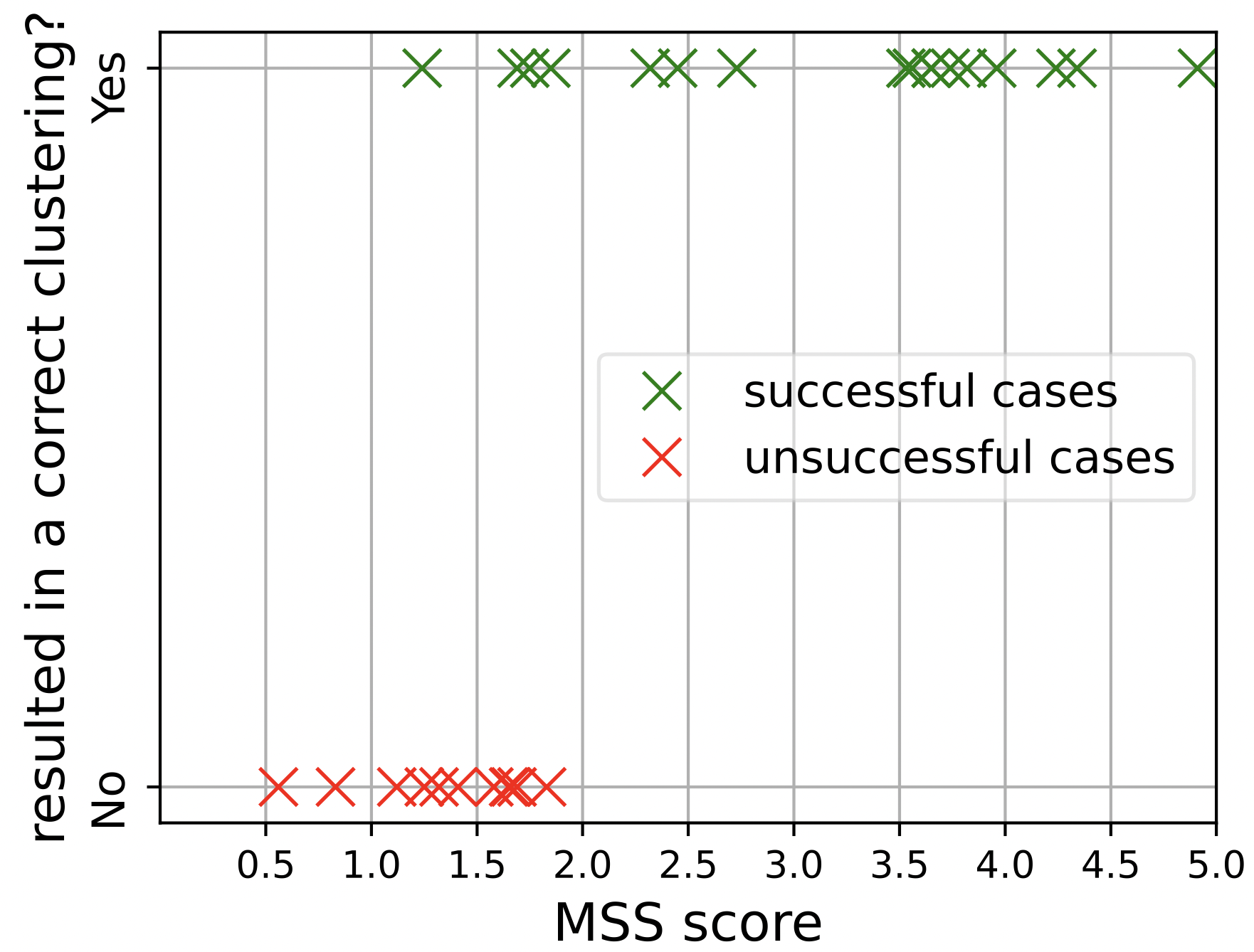}

\caption{The \texttt{MSS} score of the learned \GMM is indicative of whether the true underlying clusters will be detected or not: an \texttt{MSS} score above 2 always leads to correct detection of clusters. Each point is the result of one independent experiment.} 
\label{fig:effect_of_MSS1}
% \vspace{-1em}
\end{figure}

The results in \Cref{fig:EM_convrate} supports \Cref{theorem:convrate} by showing that if $b_i^1$ is large enough, the convergence rate of \texttt{EM} algorithm for learning the \GMM at the end of the first round increases with $b_i^1$.
\begin{figure}[h]
\centering
\includegraphics[width=0.31\columnwidth, height=4cm]
{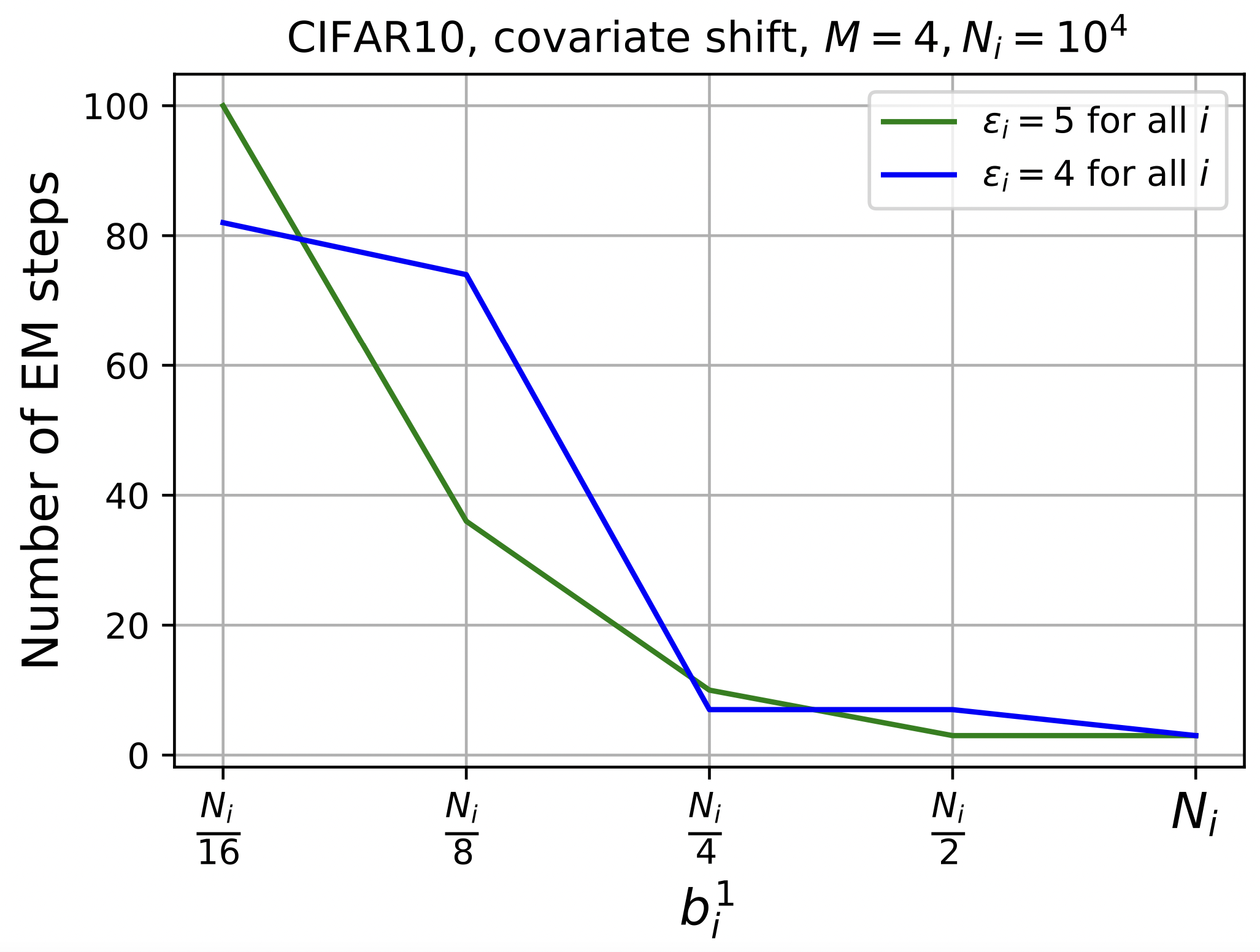}
\caption{The number of \texttt{EM} iterations needed for learning the \GMM decreases as $b_i^1$ increases. Especially, at full batch size $b_i^1=N_i (\forall i)$, very few iterations of \texttt{EM} are needed. The results are obtained on CIFAR10 with covariate shift across $M=4$ clusters. According to \Cref{lemma:localdp}, using a large enough batch size $b_i^1$ in the first round makes the underlying clusters in $\{\Delta \Tilde{\thetav}_i^1\}_{i=1}^n$ more distinguishable. Consequently, according to  \Cref{theorem:convrate}, convergence rate of \texttt{EM} algorithm for learning the \GMM increases with $b_i^1$.} 
\label{fig:EM_convrate}
% \vspace{-1em}
\end{figure}

\clearpage
\newpage
\section{Proofs}

\subsection{Proof of \cref{lemma:updatesnoise}}
\batchsizeeffectonnoise*

\begin{proof}
    The following proof has some common parts with similar results in \citep{malekmohammadi2024noiseawarealgorithmheterogeneousdifferentially}. We consider two illustrative scenarios: 
    \paragraph{\textbf{Scenario 1: the clipping threshold $c$ is effective for all samples in a batch:}} in this case we have: $\forall j \in \mathcal{B}_i^{e,t}: c< \|g_{ij}(\thetav)\|$.
    Also, we know that the two sources of randomness (i.e. stochastic and Gaussian noise) are independent, thus their variances can be summed up. Let us assume that $E[\bar{g}_{ij}(\mathbf{\thetav})] = G_i(\mathbf{\thetav})$ for all samples $j$. From \cref{eq:noisy_sg}, we can find the mean of each \emph{batch gradient} $\Tilde{g}_i^{e,t}(\mathbf{\thetav})$ (of client $i$ in round $e$ and gradient step $t$) as follows:

\begin{align}
    \mathbb E[\Tilde{g}_i^{e,t}(\mathbf{\thetav})] = \frac{1}{b_i^e}\sum_{j \in \mathcal{B}_i^{e,t}} \mathbb E[\bar{g}_{ij}(\mathbf{\thetav})] = \frac{1}{b_i^e}\sum_{j \in \mathcal{B}_i^{e,t}} G_i(\mathbf{\thetav}) = G_i(\mathbf{\thetav}).
\label{expectation_gilde}
\end{align}

Also, from \cref{eq:noisy_sg}, we can find the variance of each \emph{batch gradient} $\Tilde{g}_i^{e,t}(\mathbf{\thetav})$ (of client $i$ in round $e$ and gradient step $t$) as follows:
\begin{align} 
\label{var_g}
    \sigma_{i, \Tilde{g}}^2(b_i^e) &:= \texttt{Var}[\Tilde{g}_i^{e,t}(\mathbf{\thetav})] = \texttt{Var}\bigg[\frac{1}{b_i^e}\sum_{j \in \mathcal{B}_i^{e,t}} \Bar{g}_{ij}(\mathbf{\thetav})\bigg] + \frac{p \sigma_{i, \texttt{DP}}^2}{b_i^{e^2}} \nonumber \\
    &= \frac{1}{b_i^{e^2}}\bigg(\mathbb E \bigg[\bigg\|\sum_{j \in \mathcal{B}_i^{e,t}}\Bar{g}_{ij}(\mathbf{\thetav})\bigg\|^2\bigg] - \bigg\|\mathbb E \bigg[\sum_{j \in \mathcal{B}_i^{e,t}}\Bar{g}_{ij}(\mathbf{\thetav})\bigg]\bigg\|^2\bigg) + \frac{pc^2 z_i^2(\epsilon_i, \delta_i, b_i^1, b_i^{>1}, N_i, K, E)}{b_i^{e^2}} \nonumber \\ 
    & = \frac{1}{b_i^{e^2}}\bigg(\mathbb E \bigg[\bigg\|\sum_{j \in \mathcal{B}_i^{e,t}}\Bar{g}_{ij}(\mathbf{\thetav})\bigg\|^2\bigg] - \bigg\|\sum_{j \in \mathcal{B}_i^{e,t}}G_i(\mathbf{\thetav}) \bigg\|^2\bigg) + \frac{pc^2 z_i^2(\epsilon_i, \delta_i, b_i^1, b_i^{>1}, N_i, K, E)}{b_i^{e^2}} \nonumber \\
    & = \frac{1}{b_i^{e^2}} \bigg(\underbrace{\mathbb E \bigg[\bigg\|\sum_{j \in \mathcal{B}_i^{e,t}}\Bar{g}_{ij}(\mathbf{\thetav})\bigg\|^2\bigg]}_{\mathcal{A}} - b_i^{e^2}\big\| G_i(\mathbf{\thetav}) \big\|^2\bigg) + \frac{p c^2 z_i^2(\epsilon_i, \delta_i, b_i^1, b_i^{>1}, N_i, K, E)}{b_i^{e^2}}, 
\end{align}

where:

\begin{align}
\mathcal{A} = \mathbb E \bigg[\bigg\|\sum_{j \in \mathcal{B}_i^{e,t}}\Bar{g}_{ij}(\mathbf{\thetav})\bigg\|^2\bigg] &= \sum_{j \in \mathcal{B}_i^{e,t}} \mathbb E \bigg[\big\|\Bar{g}_{ij}(\mathbf{\thetav})\big\|^2\bigg] + \sum_{m \neq n \in \mathcal{B}_i^{e,t}} 2 \mathbb E \bigg[[\Bar{g}_{im}(\mathbf{\thetav})]^\top [\Bar{g}_{in}(\mathbf{\thetav})]\bigg] \nonumber \\
&= \sum_{j \in \mathcal{B}_i^{e,t}} \mathbb E \bigg[\big\|\Bar{g}_{ij}(\mathbf{\thetav})\big\|^2\bigg] + \sum_{m \neq n \in \mathcal{B}_i^{e,t}} 2 \mathbb E \bigg[\Bar{g}_{im}(\mathbf{\thetav})\bigg ]^\top \mathbb E \bigg[\Bar{g}_{in}(\mathbf{\thetav})\bigg] \nonumber \\
&= b_i^e c^2 + 2 \binom{b_i^e}{2} \big\| G_i(\mathbf{\thetav})\big\|^2.
\end{align}

The last equation has used \cref{expectation_gilde} and that we clip the norm of sample gradients $\Bar{g}_{ij}(\mathbf{\thetav})$ with an ``effective" clipping threshold $c$. By replacing $\mathcal{A}$ into eq. \ref{var_g}, we can rewrite it as:

\begin{align}\label{eq:effective}
    \sigma_{i, \Tilde{g}}^2(b_i^e) &:= \texttt{Var}[\Tilde{g}_i^{e,t}(\mathbf{\thetav})] = \frac{1}{b_i^{e^2}}\bigg(\mathbb E \bigg[\bigg\|\sum_{j \in \mathcal{B}_i^{e,t}}\Bar{g}_{ij}(\mathbf{\thetav})\bigg\|^2\bigg] - b_i^{e^2}\big\| G_i(\mathbf{\thetav}) \big\|^2\bigg) + \frac{p c^2 z_i^2(\epsilon_i, \delta_i, b_i^1, b_i^{>1}, N_i, K, E)}{b_i^{e^2}} \nonumber \\
    &=  \frac{1}{b_i^{e^2}}\bigg( b_i^e c^2 + \bigg(2 \binom{b_i^e}{2} -b_i^{e^2}\bigg) \big\| G_i(\mathbf{\thetav})\big\|^2 \bigg) + \frac{p c^2 z_i^2(\epsilon_i, \delta_i, b_i^1, b_i^{>1}, N_i, K, E)}{b_i^{e^2}} \nonumber \\
    &=  \frac{c^2 - \big\| G_i(\mathbf{\thetav})\big\|^2}{b_i^e} + \frac{p c^2 z_i^2(\epsilon_i, \delta_i, b_i^1, b_i^{>1}, N_i, K, E)}{b_i^{e^2}} \approx \frac{p c^2 z_i^2(\epsilon_i, \delta_i, b_i^1, b_i^{>1}, N_i, K, E)}{b_i^{e^2}}
\end{align} 

    The last approximation is valid because $p\gg 1$ (it is the number of model parameters). 

    \paragraph{\textbf{Scenario 2: the clipping threshold $c$ is ineffective for all samples in a batch:}}
    when the clipping is ineffective for all samples, i.e., $\forall j \in \mathcal{B}_i^{e,t}: c >\|g_{ij}(\thetav)\|$, we have a noisy version of the batch gradient $g_i^{e,t}(\mathbf{\thetav}) = \frac{1}{b_i^e} \sum_{j \in \mathcal{B}_i^{e,t}} g_{ij}(\mathbf{\thetav})$, which is unbiased with variance bounded by $\sigma_{i, g}^2(b_i^e)$ (see \Cref{assump:_boundedvariance}). We note that $\sigma_{i, g}^2(b_i^e)$ is a constant that depends on the used batch size $b_i^e$. The larger the batch size $b_i^e$ used during round $e$, the smaller the constant. Hence, in this case:

\begin{align}
    & \mathbb E[\Tilde{g}_i^{e,t}(\mathbf{\thetav})] = \mathbb E[g_i^{e,t}(\mathbf{\thetav})] = \nabla f_i(\mathbf{\thetav}),
\label{mean_g_ineffective}
\end{align}

and 
\begin{align}
    \sigma_{i, \Tilde{g}}^2(b_i^e) = \texttt{Var}[\Tilde{g}_i^{e,t}(\mathbf{\thetav})] = \texttt{Var}[g_i^{e,t}(\mathbf{\thetav})] + \frac{p \sigma_{i, \texttt{DP}}^2}{b_i^{e^2}}
    &\leq \sigma_{i, g}^2(b_i^e) + \frac{p \sigma_{i, \texttt{DP}}^2}{b_i^{e^2}}  \nonumber \\
    &= \sigma_{i, g}^2(b_i^e) + \frac{p c^2 z_i^2(\epsilon, \delta, b_i^1, b_i^{>1}, N_i, K, E)}{b_i^{e^2}} \nonumber \\
    & \approx \frac{p c^2 z_i^2(\epsilon, \delta, b_i^1, b_i^{>1}, N_i, K, E)}{b_i^{e^2}}.
\label{eq:var_g_effective}
\end{align}

\begin{figure*}[t]
\centering
    \includegraphics[width=0.5\columnwidth,height=6cm]{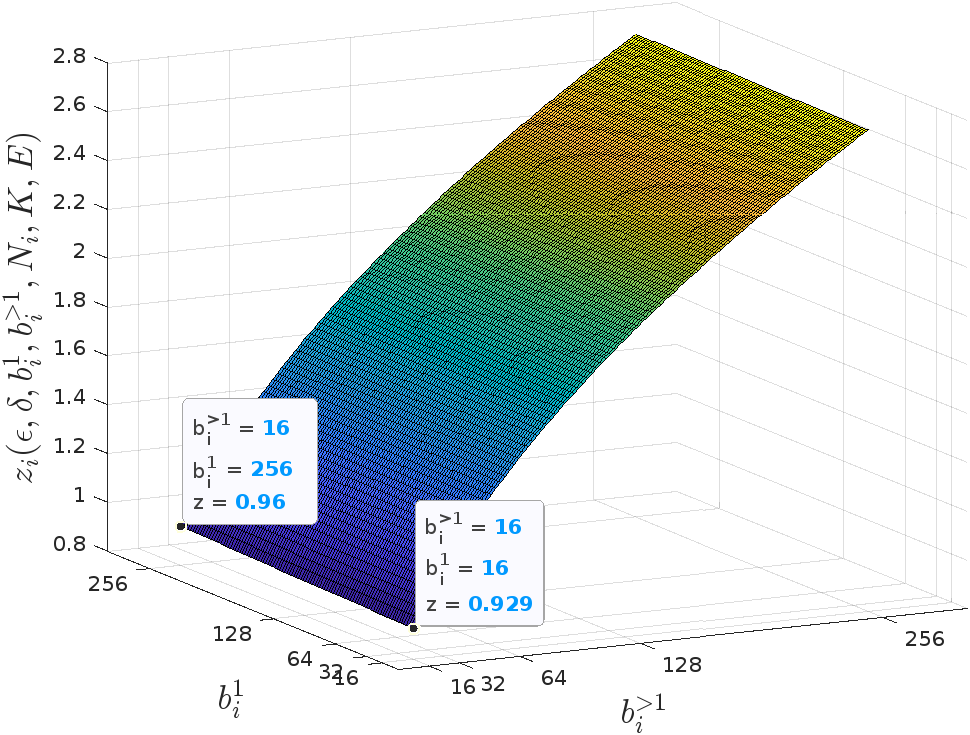}
    \caption{Plot of $z_i(\epsilon, \delta, b_i^1, b_i^{>1}, N_i, K, E)$ v.s. $b_i^1$ and $b_i^{>1}$  obtained from Renyi-\DP Accountant \citep{mironov2019renyidifferentialprivacysampled} in a setting with $N_i=6600, \epsilon=5, \delta=10^{-4}, K=1, E=200$. It is clearly observed that the effect of $b_i^{>1}$ is much more than the effect of $b_i^1$. The reason is that $b_i^{>1}$ is used in $E-1$ rounds, while $b_i^1$ is used only in the first round. So it is the value of $b_i^{>1}$ that affects $z_i$ the most.}
    \label{fig:zvsq}
\end{figure*}

The approximation is valid because $p \gg 1$ (number of model parameters). Also,  note that $\sigma_{i, g}^2(b_i^e)$ decreases with $b_i^e$. Therefore, we got to the same result as in \cref{eq:effective}. 

As observed in see \Cref{fig:zvsq}, $z_i$ grows with $b_i^1$ and $b_i^{>1}$  \emph{sub-linearly} (especially with $b_i^1$). Therefore, the variance of the client $i$'s \DP batch gradients $\Tilde{g}_i^{e,t}(\mathbf{\thetav})$ during communication round $e$, decreases with $b_i^e$ fast. The larger the batch size $b_i^e$, the less the noise existing in its batch gradients during the same round.

With the findings above, we now investigate the effect of batch size $b_i^e$ on \textbf{the noise level in clients' model updates at the end of round $e$}. During the global communication round $e$, a participating client $i$ performs $E^e_i = K \cdot \lceil \frac{N_i}{b_i^e} \rceil$ batch gradient updates locally with step size $\eta_l$:
\begin{align}
    \mathbf{\thetav}_{i}^{e,k} = \mathbf{\thetav}_{i}^{e,k-1} -\eta_l \Tilde{g}_i(\mathbf{\thetav}_{i}^{e,k-1}), ~ k=1, \ldots, E^e_i.
\end{align}
Hence,

\begin{align}
    \Delta \Tilde{\mathbf{\thetav}}_i^e =  \mathbf{\thetav}_{i}^{e,E^e_i} - \mathbf{\thetav}_{i}^{e,0}
\end{align}
In each update, it adds a Gaussian noise from $\mathcal{N}(0, \frac{c^2 z_i^2(\epsilon, \delta, b^1, b^{>1}, N_i, K, E)}{b^{e^2}}\mathbb{I}_p)$ to its batch gradients independently (see \Cref{eq:noisy_sg}). Hence:
\begin{align}
    \texttt{Var}[\Delta \Tilde{\mathbf{\thetav}}_i^e|\mathbf{\thetav}_{i}^{e,0}]
    & = E^e_i \cdot \eta_l^2 \cdot \sigma_{i, \Tilde{g}}^2(b_i^e),
\end{align}
where $\sigma_{i, \Tilde{g}}^2(b_i^e)$ was computed in \Cref{eq:effective} and \Cref{eq:var_g_effective}, and was a decreasing function of $b_i^e$. Therefore:

\begin{align}
    \texttt{Var}[\Delta \Tilde{\mathbf{\thetav}}_i^e|\mathbf{\thetav}_{i}^{e,0}]
    & \approx K \cdot N_i \cdot \eta_l^2 \cdot \frac{p c^2 z_i^2(\epsilon, \delta, b_i^1, b_i^{>1}, N_i, K, E)}{b_i^{e^3}}.
\end{align}

\end{proof}

\subsection{Proof of \cref{lemma:localdp}}
\batchsizeeffect* 

\begin{proof}
We first find the overlap between two arbitrary Gaussian distributions. Without loss of generality, lets assume we are in 1-dimensional space and that we have two Gaussian distributions both with variance $\sigma^2$ and with means $\mu_1=0$ and $\mu_2=\mu$ ($\|\mu_1 - \mu_2\|=\mu$), respectively. Based on symmetry of the distributions, the two components start to overlap at $x=\frac{\mu}{2}$. Hence, we can find the overlap between the two gaussians as follows:

\begin{align}
    O := 2\int_{\frac{\mu}{2}}^{\infty} \frac{1}{\sqrt{2\pi} \sigma} e^{-\frac{x^2}{2\sigma^2}} dx = 2\int_{\frac{\mu}{2\sigma }}^{\infty} \frac{1}{\sqrt{2\pi}} e^{-\frac{x^2}{2}} dx = 2Q(\frac{\mu}{2\sigma}),  
\end{align}

where $Q(\cdot)$ is the tail distribution function of the standard normal distribution. Now, lets consider the 2-dimensional space, and consider two similar symmetric distributions centered at $\mu_1=(0,0)$ and $\mu_2=(\mu,0)$ ($\|\mu_1 - \mu_2\|=\mu$) and with $\Sigma_1 = \Sigma_2 = \begin{bmatrix}
    \sigma^2 & 0\\
    0&  \sigma^2
  \end{bmatrix}$. The overlap between the two gaussians can be found as:

\begin{align} \label{eq:overlap}
    O = 2\int_{-\infty}^{\infty} \int_{\frac{\mu}{2}}^{\infty} \frac{1}{2\pi \sigma^2} e^{-\frac{x^2+y^2}{2\sigma^2}} dx dy &= 2 \int_{\frac{\mu}{2}}^{\infty} \frac{1}{\sqrt{2\pi} \sigma} e^{-\frac{x^2}{2\sigma^2}} dx \cdot \int_{-\infty}^{\infty} \frac{1}{\sqrt{2\pi} \sigma} e^{-\frac{y^2}{2\sigma^2}} dy = 2Q(\frac{\mu}{2\sigma}).
\end{align}

If we compute the overlap for two similar symmetric $p$-dimensional distributions with $\|\mathbf{\mu}_1-\mathbf{\mu}_2\|=\mu$ and variance $\sigma^2$ in every direction, we will get to the same result 
$2Q(\frac{\mu}{2\sigma})$. 

In the lemma, when using batch size $b^1$, we have two Gaussian distributions $\mathcal{N}\big(\mu_m^*(b^1), \Sigma_m^*(b^1)\big)$ and $\mathcal{N}\big(\mu_{m'}^*(b^1), \Sigma_{m'}^*(b^1)\big)$, where 

\begin{align}\label{eq:diagonal_cov}
    \Sigma_m^*(b^1) &= \Sigma_{m'}^*(b^1) =
  \begin{bmatrix}
    \frac{\sigma^{1^2}(b^1)}{p} & & \\
    & \ddots & \\
    & & \frac{\sigma^{1^2}(b^1)}{p}
  \end{bmatrix}.
\end{align}

Therefore, from \Cref{eq:overlap}, we can immediately conclude that the overlap between the two Gaussians, which we denote with $O_{m,m'}(b^1)$, is:

\begin{align}\label{eq:overlapb1}
    O_{m,m'}(b^1) = 2Q(\frac{\sqrt{p}\Delta_{m,m'}(b^1)}{2\sigma^1(b^1)}),
\end{align}
which proves the first part of the lemma.

Now, lets see the effect of increasing batch size. First, note that we had:

\begin{align}
    &\Delta \Tilde{\mathbf{\thetav}}_i^1 =  \mathbf{\thetav}_{i}^{1,E^1_i} - \mathbf{\thetav}_{i}^{1,0},\nonumber \\
    &\mathbf{\thetav}_{i}^{1,k} = \mathbf{\thetav}_{i}^{1,k-1} -\eta_l \Tilde{g}_i(\mathbf{\thetav}_{i}^{1,k-1}), ~ k=1, \ldots, E^1_i,
\end{align}

where $E^1_i = K \cdot \lceil \frac{N}{b^1} \rceil$ is the total number of gradients steps taken by client $i$ during communication round $e=1$. Therefore, considering that \DP batch gradients are clipped with a bound $c$, we have:

\begin{align} \label{eq:upperbound}
    \|\mathbb{E}[\Delta \Tilde{\mathbf{\thetav}}_i^1(b^1)]\| \leq E^1_i \cdot \eta_l \cdot c.
\end{align}
When we increase batch size $b_i^1$ for all clients from $b^1$ to $kb^1$, the upperbound in \Cref{eq:upperbound} gets $k$ times smaller. In fact by doing so, the number of local gradient updates that client $i$ performs during round $e=1$, which is equal to $E^1_i$, decreases $k$ times. 
As such, we can write:

\begin{align}
    \Delta \Tilde{\mathbf{\thetav}}_i^1(b^1) = k\cdot\Delta \Tilde{\mathbf{\thetav}}_i^1(kb^1) + \upsilon_i,
\end{align}

where $\upsilon_i \in \mathbb{R}^p$ is a vector capturing the discrepancies between $\Delta \Tilde{\mathbf{\thetav}}_i^1(b^1)$ and $k \cdot \Delta \Tilde{\mathbf{\thetav}}_i^1(kb^1)$. Therefore, we have:

\begin{align}
    \mu_m^*(b^1) &= \mathbb{E}[\Delta \Tilde{\mathbf{\thetav}}_i^1(b^1)|s(i)=m] = \mathbb{E}[k\cdot \Delta \Tilde{\mathbf{\thetav}}_i^1(kb^1) + \upsilon_i|s(i)=m] \nonumber\\
    & = k\cdot \mathbb{E}[\Delta \Tilde{\mathbf{\thetav}}_i^1(kb^1)] + \mathbb{E}[\upsilon_i | s(i)=m] = k\cdot \mu_m^*(kb^1) + \mathbb{E}[\upsilon_i | s(i)=m].
\end{align}

Therefore, we have:

\begin{align}
\|\mu_m^*(b^1)-\mu_{m'}^*(b^1)\| = \bigg\|k\mu_m^*(kb^1)-k\mu_{m'}^*(kb^1) + \bigg(\mathbb{E}[\upsilon_i | s(i)=m] - \mathbb{E}[\upsilon_i| s(i)=m']\bigg)\bigg\|.
\end{align}

Based on our experiments, the last term above, in parenthesis, is small and we can have the following approximation for the equation above:

\begin{align}
\|\mu_m^*(b^1)-\mu_{m'}^*(b^1)\| \approx \|k\mu_m^*(kb^1)-k\mu_{m'}^*(kb^1) \|,
\end{align}

or equivalently:

\begin{align}\label{eq:meank}
\|\mu_m^*(kb^1)-\mu_{m'}^*(kb^1)\| \approx \frac{\|\mu_m^*(b^1)-\mu_{m'}^*(b^1)\|}{k}.
\end{align}

\Cref{fig:ztoz} (left) shows the validity of the approximation above with some experimental results. On the other hand, from \Cref{eq:sigma_i^2} and also noting that a client, with dataset size $N$ and batch size $b^1$, takes $\frac{N}{b^1}$ gradient steps during each epoch of the first round, we have:

\begin{align}
    \forall m\in [M]: \sigma_m^2(b^1) = \sigma^2(b^1) \approx K \cdot N \cdot \eta_l^2 \cdot \frac{p c^2 z^2(\epsilon, \delta, b^1, b^{>1}, N, K, E)}{b^{1^3}}.
\end{align}

\begin{figure*}[t]
\centering
    
    \includegraphics[width=0.35\columnwidth,height=5.2cm]{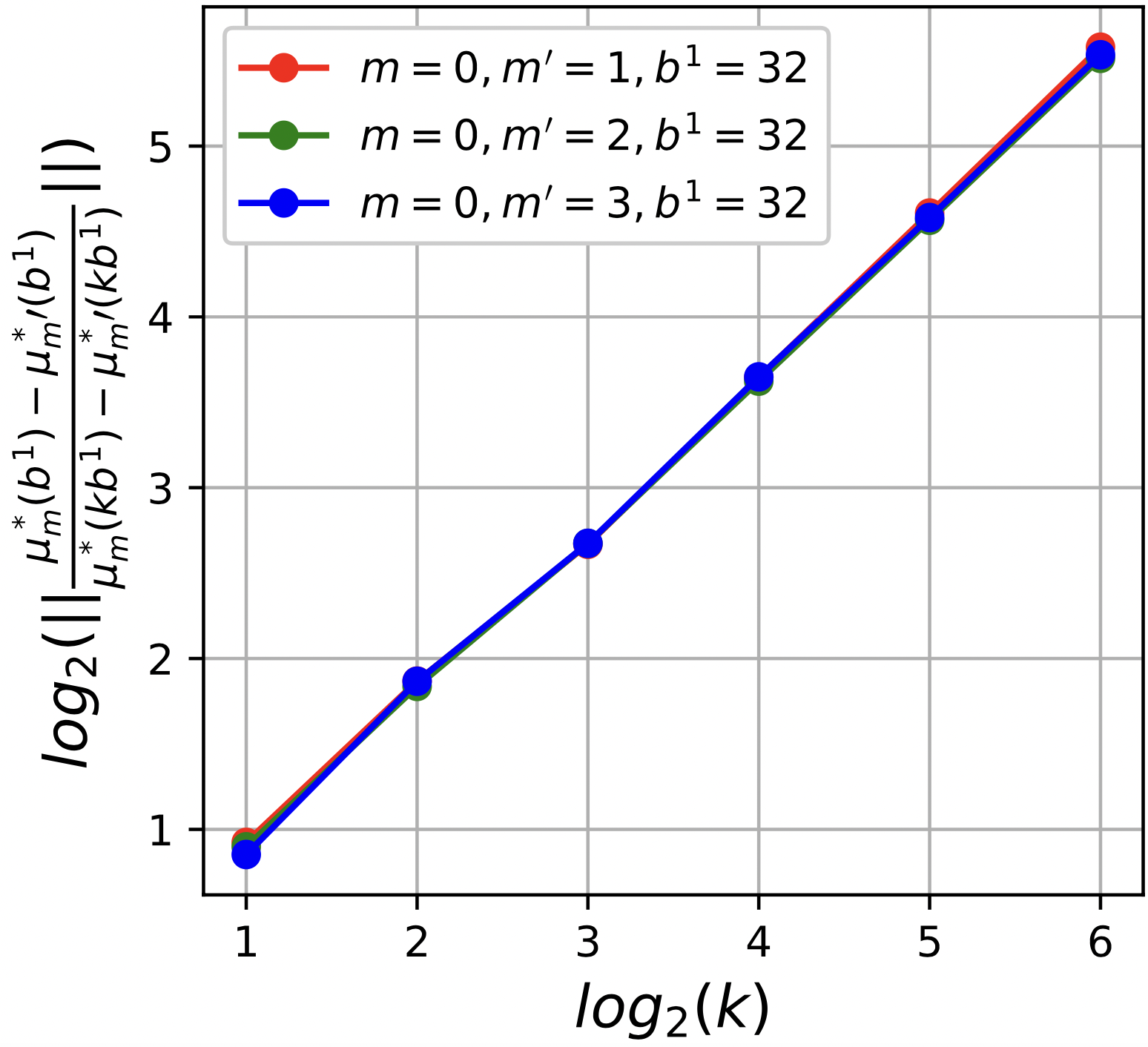}
    \includegraphics[width=0.42\columnwidth,height=5.4cm]{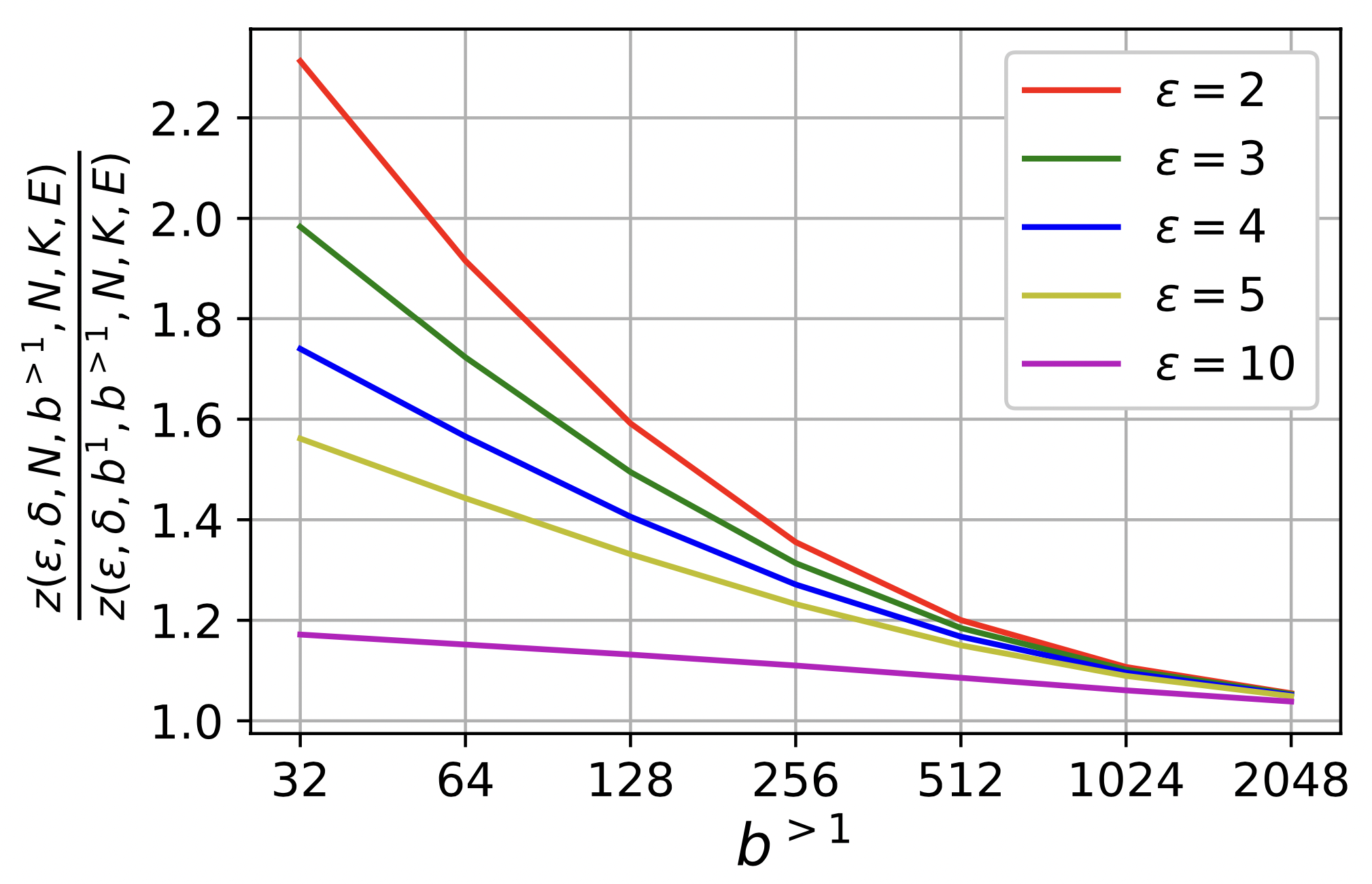}
    
    \caption{\textbf{Left:} Distance between the centers of different clusters, i.e., the distance between $\mu^*_m(b^1)$ and $\mu^*_{m'}(b^1)$, decreases $k$ times as $b^1$ increases $k$ times. The three curves in the plot are obtained on CIFAR10 with 4 clusters $m \in \{0,1,2,3\}$ obtained from covariate shift (rotation). The curves are overlapping all with slope 0.95, which is very close to 1. This shows the validity of the approximation in \cref{eq:meank}. \textbf{Right:} Effect of changing batch size $b^1$ to full batch size in the first round on the noise scale $z$. In the denominator, $b^1$ is equal to $b^{>1}$. Results are obtained from  Renyi-\DP accountant \citep{mironov2019renyidifferentialprivacysampled} with $N=50000$, $K=1$ and $E=200$. For each value of $\epsilon$, we have shown the results for seven values of $b^{>1}$.}
    \label{fig:ztoz}
\end{figure*}

When we change the batch size used during the first communication round $e=1$ from $b^1$ to $kb^1$ and we fix the batch size of rounds $e>1$, then the noise scale $z$ changes from $z(\epsilon, \delta, b^1, b^{>1}, N_i, K, E)$ to $z(\epsilon, \delta, kb^1, b^{>1}, N_i, K, E)$. Confirmed by our experimental analysis (see \cref{fig:ztoz}, right), the amount of change in $z$ due to this is small, as we have changed the batch size only in the first round $e=1$ from $b^1$ to $kb^1$, while the batch sizes in the other $E-1$ rounds are unchanged and $E\gg1$. Therefore, supported by the results in \cref{fig:ztoz}, we can always establish an upper bound on the amount of change in $z$ as $b^1$ increases: $z(\epsilon, \delta, kb^1, b^{>1}, N, K, E) \leq \rho z(\epsilon, \delta, b^1, b^{>1}, N, K, E)$, where $\rho$ is a small constant (e.g. $\rho=2.5$ in \cref{fig:ztoz}). So we have:

\begin{align} \label{eq:vark}
    \forall m \in [M]: \sigma_m^2(kb^1) = \sigma^2(kb^1) &\approx  K \cdot N \cdot \eta_l^2 \cdot \frac{p c^2 z^2(\epsilon, \delta, kb^1, b^{>1}, N, K, E)}{(kb^1)^3}\nonumber\\
    &\leq  K \cdot N \cdot \eta_l^2 \cdot \frac{p c^2 \rho^2 z^2(\epsilon, \delta, b^1, b^{>1}, N, K, E)}{(kb^1)^3}\nonumber\\
    & = \frac{\rho^2 \sigma^2(b^1)}{k^3}.
\end{align}

From \Cref{eq:meank} and \Cref{eq:vark}, we have:

\begin{align}\label{eq:overlapkb1}
    O_{m,m'}(kb^1) = 2Q\bigg(\frac{\sqrt{p}\Delta_{m,m'}(kb^1)}{2\sigma(kb^1)}\bigg)\leq  2Q\bigg(\frac{\sqrt{p}\frac{\Delta_{m,m'}(b^1)}{k}}{2\frac{\rho \sigma(b^1)}{k^{\frac{3}{2}}}}\bigg) = 2Q(\frac{\sqrt{kp}\Delta_{m,m'}(b^1)}{2\rho \sigma(b^1)}),
\end{align}
which completes the proof.
\end{proof}

\subsection{Proof of \cref{theorem:convrate}}
\convrate* 

\begin{proof}
    The proof directly follows from the proof of Theorem 1 in \cite{Ma2000AsymptoticCR} by considering $\{\Delta \Tilde{\mathbf{\thetav}}_i^1(b^1)\}_{i=1}^n$ as the samples of Gaussian mixture $\{\mathcal{N}\big(\mu_m^*(b^1), \Sigma_m^*(b^1)\big), \alpha_m^*\}_{m=1}^M$. 
\end{proof}

\subsection{Formal privacy guarantees of \algname{R-DPCFL}}\label{app:dp_guaranteres}

The privacy guarantee of \algname{R-DPCFL} for each client $i$ in the system comes from the fact that the client runs \DPSGD with a fixed \DP noise variance $\sigma_{i,\texttt{DP}}^2 = c^2\cdot z_i^2(\epsilon, \delta, b_i^1, b_i^{>1}, N_i, K, E)$ in each of its batch gradient computations. In the following \Cref{thm:localdp}, we provide a formal privacy guarantee for the algorithm to show the sample-level \DP privacy guarantees provided to each client $i$ with respect to its local dataset $\mathcal{D}_i$ and against the untrusted server (and any other external third party).

\begin{restatable}{theorem}{localdp}
The set of model updates $\{\Delta \Tilde{\thetav}_i^e\}_{e=1}^E$, which are uploaded to the server by each client $i \in \{1,\cdots,n\}$ during the training time, as well as their local model cluster selections satisfy $(\epsilon, \delta)$-\DP with respect to the client's local dataset $\mathcal{D}_i$, where the parameters $\epsilon$ and $\delta$ depend on the amount of \DP noise $\sigma_{i,\texttt{DP}}^2$ used by the client. 
\label{thm:localdp}
\end{restatable}

\begin{proof}
The sensitivity of the batch gradient in \cref{eq:noisy_sg} to every data sample is $c$. Therefore, based on \Cref{prop:7mironov}, each of the batch gradient computations by client $i$ (in the first round $e=1$ as well as the next rounds $e>1$) is $(\alpha, \frac{\alpha c^2}{2 \sigma_{i,\texttt{DP}}^2})$-\RDP w.r.t the local dataset $\mathcal{D}_i$. Therefore, if the client runs $E_i^{\texttt{
tot}}$ total number of gradient updates during the training time, which results in the model updates $\{\Delta \Tilde{\thetav}_i^e\}_{e=1}^E$ uploaded to the server, the set of model updates will be $(\alpha, \frac{ E_i^{\texttt{
tot}} \alpha c^2}{2 \sigma_{i,\texttt{DP}}^2})$-\RDP w.r.t $\mathcal{D}_i$, according to \Cref{prop:1mironov}. Finally, according to \Cref{lemma:rdptodp}, this guarantee is equivalent to 
$(\frac{ E_i^{\texttt{
tot}} \alpha c^2}{2 \sigma_{i,\texttt{DP}}^2}+\frac{log(1/\delta)}{\alpha - 1}, \delta)$-\DP (for any $\delta>1$). The \RDP-based guarantee is computed over a bunch of orders $\alpha$ and the best result among them is chosen as the privacy guarantee. Therefore, the proof is complete and the set $\{\Delta \Tilde{\thetav}_i^e\}_{e=1}^E$ satisfies $(\epsilon, \delta)$-\DP w.r.t $\mathcal{D}_i$, with $\epsilon = \frac{ E_i^{\texttt{
tot}} \alpha c^2}{2 \sigma_{i,\texttt{DP}}^2}+\frac{log(1/\delta)}{\alpha - 1}$ derived above, and $\delta>0$. Tight bounds for $\epsilon$ can be derived by using the numerical procedure, proposed in \citep{mironov2019renyidifferentialprivacysampled}, for accounting sampled Gaussian mechanism. On the other hand, clients' local cluster selections are also privatized by exponential mechanism and satisfy $(\epsilon, \delta)$-\DP. Therefore, the overall training process for each client is private and satisfies ($\epsilon, \delta$)-\DP.
\end{proof}

\section{The relation between \cref{lemma:updatesnoise} and the law of large numbers} \label{app:lolns}

We first state the weak law of large numbers and then explain how \cref{lemma:updatesnoise} is closely related to it.

\begin{theorem}[Weak law of large numbers \citep{patrickB}]\label{thm:wlolns}
Suppose that $\{X_i\}_{i=1}^b$, is an independent sequence (of size $b$) of i.i.d random variables with expected value $\mu$ and positive variance $\sigma^2$. Define $\bar X_b = \frac{\sum_{i=1}^b X_i}{b}$ as their sample mean. Then, for any positive number $\Delta>0$:

\begin{align}
    \lim_{b\rightarrow \infty}\texttt{Pr}[|\bar X_b - \mu |>\Delta] = 0 .
\end{align} 

\end{theorem}

In fact, the weak law of large numbers states that the sample mean of some i.i.d random variables converges in probability to their expected value ($\mu$). Furthermore, we can see that $\texttt{Var}[\bar X_b]=\frac{\sigma^2}{b}$, which means that \emph{the variance of the sample mean decreases as the sample size $b$ increases}.

Now, remember from \cref{eq:noisy_sg} that when computing the \DP stochastic batch gradients in round $e$ (with batch size $b_i^e$), we add \DP noise with variance $\sigma_{i, \texttt{DP}}^2/b_i^e$ to each of the $b_i^e$ clipped sample gradients in the batch \emph{and average the resulting $b_i^e$ noisy clipped sample gradients}. The sampled noise terms added to the clipped sample gradients in a batch are i.i.d with mean zero. Therefore, based on the above theorem, the variance of their average over each batch should  approach zero as the batch size $b_i^e$ grows. The same discussion applies to all the $K\cdot N_i/b_i^e$ gradient updates performed by client $i$ during a communication round $e$ (whose noises will be summed up), which results in \cref{lemma:updatesnoise}.

\section{Gradient accumulation}\label{app:grad_acc}
When training large models with \DPSGD, increasing the batch size results in memory exploding during training or finetuning. This might happen even when we are not using \DP training. On the other hand, using a small batch size results in larger stochastic noise in batch gradients. Also, in the case of \DP training, using a small batch size results in fast increment of \DP noise (as explained in \cref{lemma:updatesnoise} in details). Therefore, if the memory budget of devices allow, we prefer to avoid using small batch sizes. But what if there is a limited memory budget? A solution for virtually increasing batch size is ``gradient accumulation", which is very useful when the available physical memory of GPU is insufficient to accommodate the desired batch size. In gradient accumulation, gradients are computed for smaller batch sizes and summed over multiple batches, instead of updating model parameters after computing each batch gradient. When the accumulated gradients reach the target logical batch size, the model weights are updated with the accumulated batch gradients. The page in \url{https://opacus.ai/api/batch_memory_manager.html} shows more details.

\section{Further Related Works}\label{app:related_work}

\par{\textbf{Performance parity in \FL}}: Performance parity of the final trained model across clients is an important goal in \FL. Addressing this goal, \cite{mohri2019agnostic} proposed Agnostic \FL (\algname{AFL}) by using a min-max optimization approach. \algname{TERM} \citep{li2020tilted} used tilted losses to up-weight clients with large losses. Finally, \cite{li2019fair} and \cite{zhang2023proportional} proposed \algname{$q$-FFL} and \algname{PropFair}, inspired by $\alpha$-fairness \citep{lan_axiomatic_2010} and proportional fairness \citep{bertsimas2011price}, respectively. Generating one common model for all clients, these techniques do not perform well when the data distribution across clients is highly heterogeneous or a structured data heterogeneity exists across clusters of clients. While model personalization techniques (e.g., \algname{MR-MTL} \citep{liu2022csfl}) are proposed for the former case, stronger personalization techniques, e.g., client clustering, are used for the latter.

\par{\textbf{Differential privacy, group fairness and performance parity:}} Gradient clipping and random noise addition used in \DPSGD disproportionately affect underrepresented groups. Some works tried to address the tension between group fairness and \DP in centralized settings \citep{Tran2020DifferentiallyPA} (by using Lagrangian duality) and \FL settings \citep{pentyala2022privfairfl} (by using Secure Multiparty Computation (\texttt{MPC})). Another work tried to remove the disparate impact of \DP on model performance of minority groups in centralized settings \citep{Esipova2022DisparateII}, by preventing gradient misalignment across different groups of data. Unlike the previous works on group fairness, our work adopts cross-model fairness, where the utility drop after adding \DP must be close for different groups \citep{chu2023focus}, including minority and majority clients. As we consider a structured data heterogeneity across clients, the mentioned approaches are not appropriate, due to generating one single model for all. 

% In contrast, we propose a robust \DP ``clustered" \FL algorithm, which identifies different groups of clients and learns a separate model for each. This way, the clusters with minority clients will be separated from the rest of clients and will experience a lower utility drop, compared to when a single model is learned for all clients.

\end{document}